\documentclass{article}

\PassOptionsToPackage{numbers, sort&compress}{natbib}

\usepackage[preprint]{neurips_2026}

\usepackage[utf8]{inputenc} %
\usepackage[T1]{fontenc}    %
\usepackage{hyperref}       %
\usepackage{url}            %
\usepackage{booktabs}       %
\usepackage{multirow}       %
\usepackage{amsfonts}       %
\usepackage{nicefrac}       %
\usepackage{microtype}      %
\usepackage[table]{xcolor}  %
\definecolor{citeaccent}{RGB}{172, 49, 93}
\definecolor{refaccent}{RGB}{25, 92, 170}
\hypersetup{
  colorlinks=true,
  citecolor=citeaccent,
  linkcolor=refaccent,
  urlcolor=refaccent
}
\usepackage{graphicx}
\usepackage{array}

\usepackage{subcaption}
\newcommand{\datasetTopRule}{\specialrule{0.9pt}{0pt}{0pt}}
\newcommand{\datasetRowRule}{\specialrule{0.25pt}{0pt}{0pt}}
\newcommand{\datasetBottomRule}{\specialrule{0.9pt}{0pt}{0pt}}
\newcommand{\datasetHeaderCell}[1]{\cellcolor{black!8}\textbf{#1}}
\newcommand{\datasetHeaderRow}{\datasetHeaderCell{Template} & \datasetHeaderCell{Clean prompt template example} & \datasetHeaderCell{Corrupted prompt template example} & \datasetHeaderCell{Correct} & \datasetHeaderCell{Incorrect} \\}
\newcommand{\codeblock}[1]{\begin{tabular}[t]{@{}l@{}}#1\end{tabular}}
\newcommand{\codeindent}{\hspace*{1.25em}}
\newcommand{\asciisq}{\textquotesingle}
\newcommand{\asciidq}{\textquotedbl}
\usepackage{amsmath}
\usepackage{amssymb}
\usepackage{mathtools}
\usepackage{amsthm}
\usepackage{float}
\usepackage{placeins}
\usepackage{wrapfig}
\usepackage{enumitem}
\usepackage[capitalize,noabbrev]{cleveref}

\usepackage{aliascnt} %

\theoremstyle{plain}
\newtheorem{theorem}{Theorem}[section]

\newaliascnt{proposition}{theorem}

\aliascntresetthe{proposition}

\newaliascnt{lemma}{theorem}

\aliascntresetthe{lemma}

\newaliascnt{corollary}{theorem}

\aliascntresetthe{corollary}

\theoremstyle{definition}
\newaliascnt{definition}{theorem}
\newtheorem{definition}[definition]{Definition}
\aliascntresetthe{definition}

\newaliascnt{assumption}{theorem}

\aliascntresetthe{assumption}

\theoremstyle{remark}
\newaliascnt{remark}{theorem}
\newtheorem{remark}[remark]{Remark}
\aliascntresetthe{remark}

\crefname{theorem}{Theorem}{Theorems}
\crefname{proposition}{Proposition}{Propositions}
\crefname{lemma}{Lemma}{Lemmas}
\crefname{corollary}{Corollary}{Corollaries}
\crefname{definition}{Definition}{Definitions}
\crefname{assumption}{Assumption}{Assumptions}
\crefname{remark}{Remark}{Remarks}

\title{Demystifying Variance in Circuit Discovery of LLMs}

\author{%
  \normalfont
  \textbf{Frank Zhengqing Wu}
  \qquad \textbf{Francesco Tonin}
  \qquad \textbf{Volkan Cevher} \\[3pt]
  {\small Laboratory for Information and Inference Systems (LIONS)} \\
  {\small École Polytechnique Fédérale de Lausanne (EPFL), Lausanne, Switzerland} \\
  {\footnotesize\texttt{\{zhengqing.wu,francesco.tonin,volkan.cevher\}@epfl.ch}} \\
}

\begin{document}

\maketitle
\captionsetup[figure]{font=small}

\begin{abstract}
Circuit discovery is a key technique in mechanistic interpretability to pinpoint the model components that are crucial for performing a given task.  
Although the current state-of-the-art method (EAP-IG) performs well on the metric of (un)faithfulness, it suffers from substantial variability.
This includes \emph{resampling variance}, where the circuit changes when we probe with a new batch of data from the same distribution; \emph{rephrasing variance}, where the discovered circuit shifts when the prompts are rephrased; and \emph{sample-wise variance}, where a circuit with low population unfaithfulness exhibits large fluctuations in unfaithfulness across individual samples. 

This paper studies the roots of these variances.
We demonstrate that CEAP, our new circuit discovery method that improves upon EAP-IG with a theoretical guarantee, can substantially lessen resampling variance.
We further show that rephrasing variance arises because prompts with different templates tend to activate different circuits in the model.  
This leads us to argue that it may be challenging to find a comprehensive circuit that explains and controls the model’s behavior on a task, which can be expressed in countless templates, suggesting that LLMs may be inherently hard to steer.
We show that sparsity, which has been claimed to form more compact and interpretable task circuits, fails to solve this problem.  
Regarding sample-wise variance, we argue that it is largely benign: extremely poor unfaithfulness scores often stem from how unfaithfulness is defined, rather than from defects in the measured circuits. 
We show that the magnitude of unfaithfulness is affected by \emph{selective contribution scaling}, a neural mechanism that accounts for the extremely poor scores sometimes observed.
\end{abstract}

\section{Introduction}
Circuit discovery aims to find the key pathways a model uses to perform a task. 
It usually does this by first evaluating the importance of all components, then selecting the most important ones to form the circuit that is considered responsible for carrying out the task.
In recent years, this area has advanced along two main axes: greater efficiency and broader generality.
On the efficiency side, methods are becoming more automated and easier to parallelize.
To eliminate the need for extensive manual work, \citep{conmy2023towards} introduced an automated framework that evaluates the importance of model components by applying causal interventions to each in turn.
As such a method scales poorly, \citep{marks2024sparse,syed2024attribution,nanda2022attribution} proposed gradient-based circuit discovery approaches that leverage the parallelism provided by modern deep learning hardware.
In particular, edge attribution patching (EAP) \citep{syed2024attribution,nanda2022attribution} can evaluate all the components of a model in two forward and one backward passes.
However, this level of efficiency comes at the cost of precision.  
To balance these two aspects, \citep{faithfulness} interpolated the model behavior into multiple steps and applied EAP to each step, yielding EAP-IG (integrated gradients).

On the generality side, earlier work \citep{nanda2022attribution,nostalgebraist2020logitlens} often focused on single-sample probing.
This methodology was unable to provide much insight into how models handle a task as a whole \citep{nanda2022attribution}.
Subsequently, researchers transitioned to methods that assess the significance of model components over a task distribution with many samples, typically through computing the average importance across those samples \citep{faithfulness,wang2023interpretability,hanna2023how,marks2024sparse}.
By steering the discovered task circuits, \citep{finetunecircuit,cheng2026drives} enhanced the model’s performance on those tasks, highlighting the practical value of circuit discovery.
Nevertheless, several works revealed great variances in the circuit discovery algorithms, calling their reliability into question.

In this work, we discuss three forms of such variance.
The first is \emph{resampling variance} \citep{meloux2025mechanistic}, which means that the identified circuit may differ when the probing prompts are resampled from the same underlying distribution.
We show that such variance can be reduced by our new circuit discovery method, conductance-based EAP (CEAP), which achieves the same level of efficiency as EAP-IG.

The second type of variance is \emph{rephrasing variance} \citep{meloux2025mechanistic}, where alternative phrasings of the same input can produce substantially different circuits.
We show that this is because different templates activate distinct circuits in the model.
A recent work \citep{franco2026findinghighlyinterpretablepromptspecific} also observed template dependence for a different circuit-discovery algorithm.
Together, our results suggest the ubiquity of this phenomenon.
While \citep{franco2026findinghighlyinterpretablepromptspecific} proposed to shift the granularity of circuit discovery from tasks to templates, we take the stance that the dependence of circuits on templates reveals a profound shortcoming of current models and notions of tasks.
Since it is impractical to exhaust all templates, the generality goal of circuit discovery might be unattainable: it is infeasible to find comprehensive task circuits, and one cannot reliably predict the outcome of task circuit steering when the model is faced with unseen templates.
We investigate sparse training \citep{gao2025weight} as a potential remedy, as it tends to produce more compact task circuits, which might compel the model to rely on a single circuit for all templates.  
Nonetheless, we do not find strong evidence that sparsity adequately mitigates this issue.

The third form of variance is \emph{sample-wise variance} \citep{miller2024robust}, in which the circuit quality metric, unfaithfulness, varies sharply across samples.
We show that the poor unfaithfulness scores arise from a negative correlation between the metric and the magnitude of the model's behavior, rather than from deficiencies in the discovered circuit itself. We account for this using a neural mechanism that we term selective contribution scaling.

Taken together, our contributions are the following:
\begin{itemize}[leftmargin=*,itemsep=0.2em,parsep=0pt,topsep=0.2em]
    \item We introduce CEAP, whose underlying component selection strategy is more principled than that of EAP-IG. We support this claim by proving that conductance satisfies additive order preservation—an intuitive and desirable property that IG fails to meet. We further demonstrate that CEAP lowers resampling variance.
    \item We show that template-induced rephrasing variance is widespread and study its implications for circuit steering. We find that sparsity does not solve this problem, despite its prima facie appeal.
    \item We explain extreme sample-wise unfaithfulness through selective contribution scaling, showing that large unfaithfulness values reflect the metric geometry rather than catastrophic circuit failure.
\end{itemize}

\section{Background: Evaluating Subgraphs With (Un)faithfulness}
\label{sec: background}
\begin{figure}[H]
    \centering
    \includegraphics[width=0.6\textwidth]{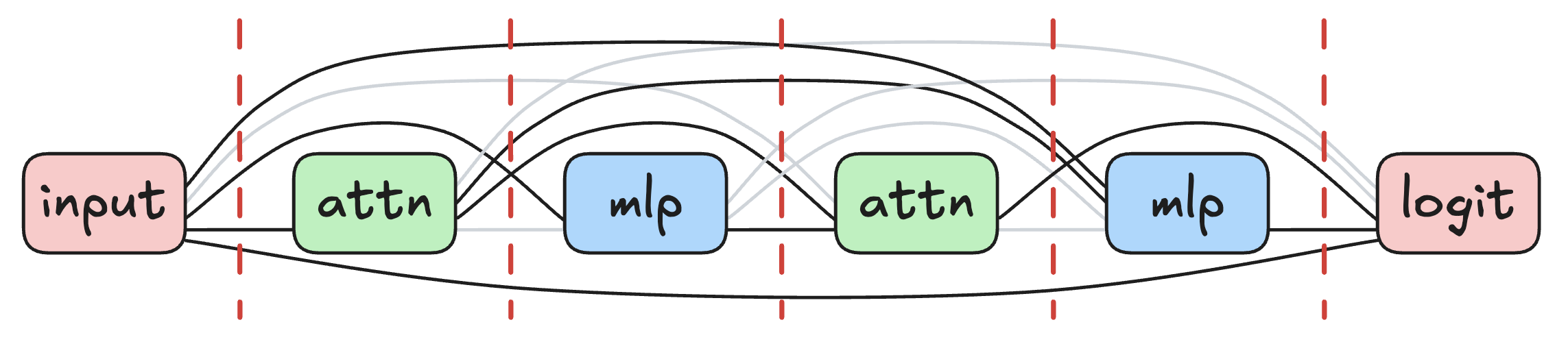}
    \caption{The computational graph of a mini Transformer. The residual addition is split into separate edges. The black edges are the ones included in the subgraph/circuit, while the grey ones are excluded. During a patched run, modules will take the results from previous ones through black edges but the grey edges will be fixed at activations corresponding to the corrupted input. Each red dotted line intercepts a set of edges that is necessary and sufficient for producing the output.}
    \label{fig:circuit-diagram}
\end{figure}

Circuit discovery aims to identify a subgraph (circuit) from the model's full computational graph that best reproduces the full model's behavior (illustrated in \cref{fig:circuit-diagram}).
A widely used measure of the circuit quality is {faithfulness} \citep{faithfulness}.
In the current setup, the \emph{behavior} of a task is usually characterized by how a task-specific output metric changes when the input changes from a base point (referred to as corrupted input in previous literature) to the clean input.
For instance, when examining subject-verb agreement (SVA) \citep{SVA}, we might quantify how much output probability shifts from plural verbs to singular verbs as the input changes from one that cues a plural verb (corrupted input) to one that cues a singular verb (clean input).
To quantify this change in general, for a given input $x$, we use the metric \texttt{prob diff} $$M(x;x_0,x_1) \coloneqq \mathbb{P}(C(x_1)|x) - \mathbb{P}(W(x_0)|x),$$
where $x_1$ is a clean input, like \emph{``The book on the table"}, and $x_0$ is a corrupted input like \emph{``The books on the table"}.
$C(x_1)$ is the set of tokens that $x_1$ encourages, and $W(x_0)$ is the set of competing tokens encouraged by $x_0$ instead.
Continuing our example of SVA, $C(x_1)$ is singular verbs, and $W(x_0)$ is plural verbs.
In the following, when the clean-corrupted pair is clear from context, we shorten $M(x;x_0,x_1)$ to $M(x)$.

Besides \texttt{prob diff},
another commonly used metric is \texttt{logit diff} \citep{nanda2022attribution,faithfulness}, which, in this case, is the summation of all the logits corresponding to singular verbs minus those of the plural ones.
\citep{faithfulness,conmy2023towards} also recommended using KL divergence as a metric. Nevertheless, we observed numerical issues with this approach, detailed in \cref{app: KL}.

We use subscripts $G$ and $G'$ to denote evaluation on the full computational graph and selected subgraph, respectively.
The \emph{behavior} of the full model and the selected circuit can be written as a subtraction of metrics, 
\begin{equation}
Q_\Gamma(x_0, x_1)\coloneqq M_\Gamma(x_1) - M_\Gamma(x_0),
\label{eq:metric_subtraction}
\end{equation}
where $\Gamma \in \{G,G'\}$.
The metrics for the full model in \cref{eq:metric_subtraction} can each be obtained via two standard forward passes.
By contrast, $M_{G'}(\cdot)$ is obtained by patching: edges outside $G'$ are fixed to their activations under $x_0$, while edges inside $G'$ remain responsive to the current input through the selected subgraph \citep{faithfulness}, as visualized in \cref{fig:circuit-diagram}.
By this design, $M_{G'}(x_0) = M_{G}(x_0)$.

Finally, the faithfulness \citep{faithfulness} of $G'$ with respect to $G$ is defined as $$\Phi(G',G; x_0, x_1) = \frac{Q_{G'}(x_0, x_1)}{Q_G(x_0, x_1)}.\footnote{This value was previously called normalized faithfulness in previous literature \citep{faithfulness,zhang2025eapgp}. As the ``unnormalized" faithfulness is not the emphasis of this paper, we will refer to the normalized faithfulness as faithfulness for brevity.}$$ 
The best value for $\Phi(G',G; x_0, x_1)$ is $1$, as it is the faithfulness yielded by choosing the entire graph.
For the ease of discussion, in what follows, we use \emph{unfaithfulness}, defined as $$U(G',G; x_0, x_1)\coloneqq \vert 1 - \Phi(G',G; x_0, x_1)\vert,$$
to evaluate the subgraph.
The smaller this metric, the better, which clarifies the presentation.

\section{Conductance-based EAP (CEAP)}

To find a subgraph that best explains the model's behavior on a certain task, we first quantify the importance of the edges in the model with a scoring method, after which we select the most important edges, which constitute the subgraph.
Naturally, the quality of the subgraph hinges upon the soundness of the scoring method.
In this section, we discuss how the conductance-based scoring method, which CEAP builds on, improves upon the previous IG-based scoring method.

Suppose we are studying a metric function $M(\cdot)\in C(\mathbb{R}^{d_{\text{in}}},\mathbb{R})$, where $d_\text{in}$ is the input dimension.
The value of the metric given any input can be written as a function of the activations on the edges, $M(x) = M(\{a_e(x)\}_{e\in E})$.
Here, $a_e(x)\in\mathbb{R}^{d_e}$ are the activations on edge $e$, and $E$ denotes the set of all edges.
We study the process where we change the input from $x_0$ to $x_1$.
The importance we attribute to each component $i$ of the activation ${a_e}$ is its conductance \citep{dhamdhere2018how}:
\begin{definition} The \emph{conductance} of a scalar activation ${a_{ei}}$ for a scalar-valued function $M(\cdot)$, evaluated with respect to the corrupted-input pair $(x_0,x_1)$, is given by:
\begin{equation}
\begin{aligned}
I_{\text{Cond}}({a_{ei}})
= \int_{\gamma}\frac{\partial M}{\partial {a_{ei}}}\mathrm{d}{a_{ei}}
= \sum_{j\in [d_\mathrm{in}]} ({x_1}_j - {x_0}_j)
   \int_{0}^{1}
   \frac{\partial M(x_0 + \alpha(x_1 - x_0))}{\partial {a_{ei}}}
   \frac{\partial {a_{ei}}}{\partial x_j}
   \,\mathrm{d}\alpha,
\end{aligned}
\label{eq:conductance_formula}
\end{equation}
where $\gamma$ is the path that ${a_{ei}}$ traverses when the input moves from $x_0$ to $x_1$ in a straight line.
Then, we define the importance/conductance of the edge $e$ to be:
$I_{\text{Cond}}(e) = \sum_{i\in [d_e]} I_{\text{Cond}}({a_{ei}}).$
\end{definition}

\paragraph{Comparison with previous methods.} The current state-of-the-art patching method (EAP-IG \citep{faithfulness}) uses integrated gradients (IG) \citep{IG-sundararajan17a}, which scores each activation component with:
\begin{equation}
I_{\text{IG}}({a_{ei}}) = [{a_{ei}}(x_1) - {a_{ei}}(x_0)]\int_0^1 \frac{\partial M(x_0+\alpha(x_1 - x_0))}{\partial {a_{ei}}}\mathrm{d}\alpha.
\label{eq:IG_formula}
\end{equation}
Then, the IG method scores an edge with $I_{\text{IG}}(e) =  \sum_{i\in [d_e]} I_{\text{IG}}({a_{ei}}).$

When evaluating edges across multiple samples, the edge scores are computed as the average over all samples.
After all edges have been scored, \emph{circuit discovery} proceeds by retaining only those edges whose scores have the largest absolute values.
In this paper, we use the greedy method \citep{faithfulness}.
This method first finds a circuit with a target edge number and prunes the childless or parentless edges.\footnote{The target edge number and the final edge number are usually close. 
We do not distinguish between them in our discussion.}
The excluded edges are patched as described in the previous section.
Different scoring methods, IG or conductance, result in different patching: EAP-IG and CEAP.
Overall, to obtain the most informative graph, we need to give higher absolute-value scores to edges whose patching causes greater changes at the network output, which motivates us to study the following property.

\begin{definition}[Additive order preservation]
Consider a function that can be decomposed additively into several branches:
$
\mathbb{R} \ni F(x; B)
= \sum_{b \in B} f_b(y_b)
= \sum_{b\in B} f_b(g_b(x)),
$
where \(B\) is the set of branches, and \(x\) and \(y_b\) may be vectors.
We also denote $
\mathbb{R} \ni F(x; {B'})
= \sum_{b \in B'} f_b(y_b)
$
for any $B'\subseteq B$.
Suppose we study the process where \(x\) moves from \(x_0\) to \(x_1\).
For convenience, we denote
$
\Delta F(B') =  F(x_1; B') - F(x_0; B')
$
and $\Delta f_b = f_b(x_1) - f_b(x_0)$.
For a scoring function that assigns a scalar to an edge, \(I(y_b) \in \mathbb{R}\),
we say it satisfies \emph{additive order preservation} if
  \[
  \resizebox{\textwidth}{!}{\ensuremath{\displaystyle \vert \Delta F(B)-\Delta F(B \setminus \{b_1\})\vert = \vert \Delta f_{b_1} \vert >  \vert \Delta f_{b_2} \vert =\vert\Delta F(B)-\Delta F(B \setminus \{b_2\})\vert \iff \vert I(y_{b_1}) \vert  > \vert I(y_{b_2}) \vert.}}
  \]
\label{def:additive-order-preservation}
\end{definition}

\begin{remark}
Note that $\vert \Delta F(B)-\Delta F(B \setminus \{b\})\vert$ exactly equals to the importance of "activation" $y_{b}$ measured with activation patching \citep{nanda2022attribution}, which freezes $y_{b}$ to be $y_{b}(x_0)$ while letting other activations freely move with upstream inputs.
    Additive order preservation conveys that, if the patching of one branch causes the whole function to change its behavior more than the patching of another, then the former branch should be assigned a score higher in its absolute value than that of the latter, so that the former can be prioritized to be admitted into the subgraph.
    This property is particularly desirable when we study unfaithfulness, which evaluates a subgraph by how much the behavior is affected when the complement of the subgraph is patched.
\end{remark}

\begin{remark}
The sum-of-branches type of functions discussed in \cref{def:additive-order-preservation} can be of significant relevance if we wish to interpret the residual stream of currently prevalent architectures, as the residual stream is always a summation of previous branches.
More broadly, additive order preservation is a minimal desideratum for scoring when applied to general functions.
General functions may have other ways to mix the outputs of branches.
 The appropriate desiderata for more complicated mixings may not be immediately intuitive.
However, the first-order Taylor expansion of any mixing always appears in the form of sum-of-branches functions, so studying the latter still provides insight.
\end{remark}

To show that conductance is a more principled scoring method, we prove:
\begin{theorem}
   Conductance satisfies additive order preservation, while IG does not.
   \label{thm:Cond_VS_IG}
\end{theorem}
The proof is in \cref{app:proof_of_Cond_VS_IG}.
In the following, we derive the discretized version of conductance and IG, which shows how the two quantities are computed numerically.
The discretized computation also sheds light on why IG appears less principled.

Let us break down the input movement $x_0 \to x_1$ into $K$ evenly spaced segments and denote the endpoints of each segment by $x^k\coloneqq x_0 + k/K (x_1-x_0)$, where $k\in \{0,1,\cdots, K\}$.
We also denote the activations corresponding to these input endpoints by $a_{ei}^k \coloneqq {a_{ei}}(x^k)$.
The discretized IG can be written as
\begin{equation}
I_{\text{IG}}^{\text{dis}}({a_{ei}}) = \frac{{a_{ei}^K} - {a_{ei}^0}}{K}\sum_{k=0}^{K-1} \frac{\partial M(x^k)}{\partial {a_{ei}}}.
\label{eq:IG_discrete}
\end{equation}
One can interpret this formula as a refined version of \texttt{gradient$\times$input} \citep{nanda2022attribution}.
The latter scores the edge with $\left({a_{ei}^K} - {a_{ei}^0}\right)\frac{M(x^0)}{a_{ei}}$.
\cref{eq:IG_discrete}, in comparison, computes the average gradient instead of using one gradient at one endpoint.
Another more fine-grained interpretation of \cref{eq:IG_discrete} is that it is a weighted sum of the gradients $\frac{\partial M(x^k)}{\partial {a_{ei}}}$, with all weights equal to $\frac{{a_{ei}^K} - {a_{ei}^0}}{K}$.
Nevertheless, it may be undesirable for all the weights to be the same.
After all, the rate of change of $M$ with respect to $e_{ei}$ is roughly $\frac{\partial M(x^k)}{\partial {a_{ei}}}$ only when $a_{ei}$ lies within the segment of $\left(a_{ei}^k ,a_{ei}^{k+1}\right)$, whose length is not necessarily $\frac{{a_{ei}^K} - {a_{ei}^0}}{K}$ due to the nonlinearity of $a_{ei}(x)$.
This motivates a new formula that rectifies the weightings of $\frac{\partial M(x^k)}{\partial {a_{ei}}}$:
\begin{equation}
I_{\text{Cond}}^{\text{dis}}({a_{ei}}) = \sum_{k=0}^{K-1} \frac{\partial M(x^k)}{\partial {a_{ei}}}\left( a_{ei}^{k+1}-a_{ei}^{k}\right),
\label{eq:Cond_discrete}
\end{equation}
which is exactly the discrete form of conductance \citep{shrikumar2018computationallyefficientmeasuresinternal}.
The ``weighted-sum of gradients" perspective of \cref{eq:IG_discrete,eq:Cond_discrete} also suggests that IG can be taken as an approximation to conductance, which explains its efficacy shown by \citep{faithfulness}.

\subsection{CEAP Reduces Circuit Variance Under Data Resampling}

\citep{meloux2025mechanistic} noted that EAP-IG finds significantly different circuits when we resample the dataset from the same distribution.
We will show that such an instability is ameliorated by CEAP.
Following the setup in \citep{faithfulness}, we conducted experiments on GPT-2 small, GPT-2 XL, Pythia-160M, and Pythia-2.8B over three datasets: SVA \citep{SVA}, IOI \citep{wang2023interpretability}, and greater-than \citep{hanna2023how}.\footnote{\citep{faithfulness} also used three other datasets.
They have relatively few samples and are unsuitable for our purpose.}
The templates for generating these datasets are shown in \cref{app: dataset templates}.
Experimental details are in \cref{app:experiments}.

In our experiments, for each task, we used a full dataset of roughly $10000$ samples and subsampled $4$ smaller datasets (without replacement), each with $1000$ samples.
We performed circuit discovery for the $4$ datasets and measured the stability of the found circuits using pairwise Jaccard index (PJI).
More concretely, suppose we have two graphs $ G_1$ and $ G_2$, with edge sets $E_{G_1}$ and $E_{G_2}$, respectively.
The PJI for these two graphs is defined as
$\left|E_{G_1} \cap E_{G_2}\right|/\left|E_{G_1} \cup E_{G_2}\right|$.
With $4$ sub-datasets, for the same target edge number, we obtained $4$ graphs and thus $6$ PJIs.
We report the mean and standard deviation of these PJIs in \cref{fig: ceap_vs_eap-ig_jaccard} for GPT-2 XL on SVA.
The x-axis shows the number of edges we admitted during circuit discovery.
We separate experiments for different templates, as we will show in \cref{sec: template-induced circuit variance} that different templates activate distinct circuits in the model and thus their respective circuit discovery should be done separately.
Due to space constraints, we only show the first $6$ templates here.
The results for other models, datasets, and templates are shown in \cref{app: complete ceap vs eap-ig jaccard and unfaithfulness}.
\emph{Overall, we find that CEAP achieves PJI values that are higher than, or at least on par with, those of EAP-IG.}
We believe this is because CEAP can find important edges in a more principled manner as shown in \cref{thm:Cond_VS_IG}, while EAP-IG can only be viewed as a noisy version of CEAP.

\begin{figure}
    \centering
    \includegraphics[width=0.8\linewidth]{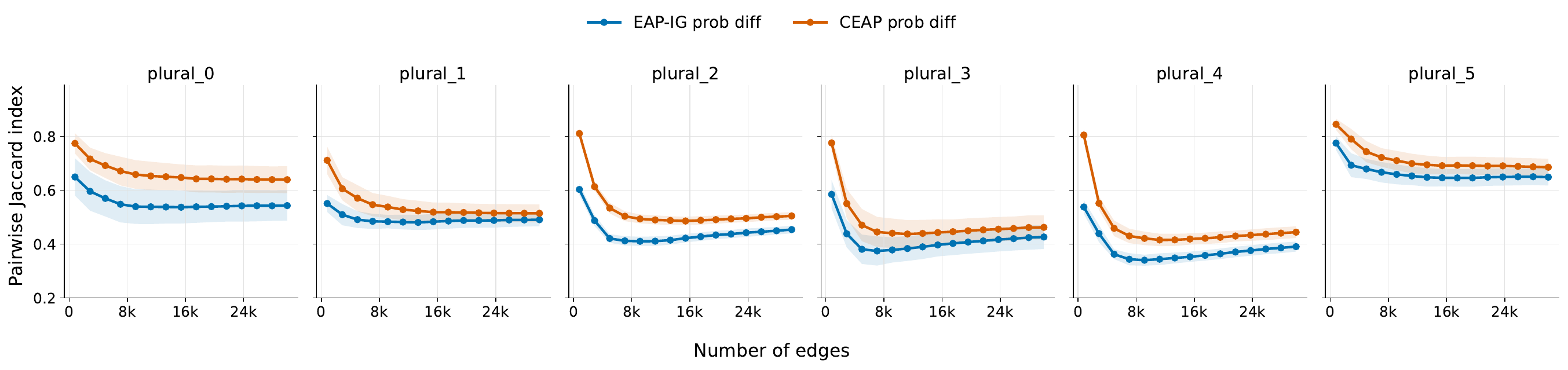}
    \caption{Comparison of PJI yielded by CEAP and EAP-IG (using GPT-2 XL for a subset of SVA)}
    \label{fig: ceap_vs_eap-ig_jaccard}
\end{figure}

Note that CEAP attains unfaithfulness comparable to EAP-IG, so its increased stability in circuit selection does not arise from choosing task-irrelevant circuits. 
Further details are given in \cref{app: complete ceap vs eap-ig jaccard and unfaithfulness}.

\section{Template-induced Circuit Variance}

\label{sec: template-induced circuit variance}
\citep{meloux2025mechanistic} observed that performing circuit discovery with paraphrased clean-corrupted prompt pairs can yield dramatically different circuits, even when a human judges the paraphrased prompts as representing the same underlying task.
We demonstrate that this occurs because the templates influence the circuits the models use to perform the same task.
To show this, we take $1000$ samples from the same task but in different templates, apply CEAP to evaluate edge importance, and identify the critical circuit for each sample.
Then, we illustrate the similarity of the scores across different samples, as well as the similarity of the graphs identified for these samples.

For the similarity of the scores, we investigate the ranking of edges based on the absolute values of their scores.
More precisely, consider all edges in the model arranged in a single vector $\left(e_1, e_2, \cdots, e_{\vert E \vert} \right)$.
Their associated scores for a given sample are given by $\left(s_1, s_2, \cdots, s_{\vert E \vert} \right)$, where $s_i = I_{\mathrm{Cond}}(e_i)$.
We then take the absolute values of these scores,\footnote{This is to follow the convention of \citep{faithfulness} that performed circuit discovery using the absolute score values.} $\left(\vert s_1 \vert, \vert s_2 \vert, \cdots, \vert s_{\vert E \vert} \vert \right)$, and sort them in descending order to obtain the absolute-score rank vector $\left(r_1, r_2, \cdots, r_{\vert E \vert} \right)$, where $r_i$ denotes the rank of $\vert s_i \vert$.
We then applied UMAP to the absolute-score rank vectors for all $1000$ samples to project them onto $2$ dimensions. 
The resulting visualization for GPT-2 small on SVA is shown in \cref{fig:umap}.
Additionally, we directly visualize the similarity between the graphs obtained for different samples.
We select the graph size for which the mean unfaithfulness across all $1000$ samples falls below $0.2$, and then compute the PJI among all graphs found for each sample at that size; these results are shown in \cref{fig:sample-wise pji}.
We produced the same visualization for GPT-2 small and Pythia-160M on SVA, IOI, and greater-than.
Across all figures, it is evident that the model relies on different circuits to process different templates.
The complete set of visualizations is provided in \cref{app: umap and pji matrix}.

\begin{figure}
    \centering
    \begin{subfigure}[t]{0.36\textwidth}
        \centering
        \includegraphics[width=\linewidth,height=0.22\textheight,keepaspectratio]{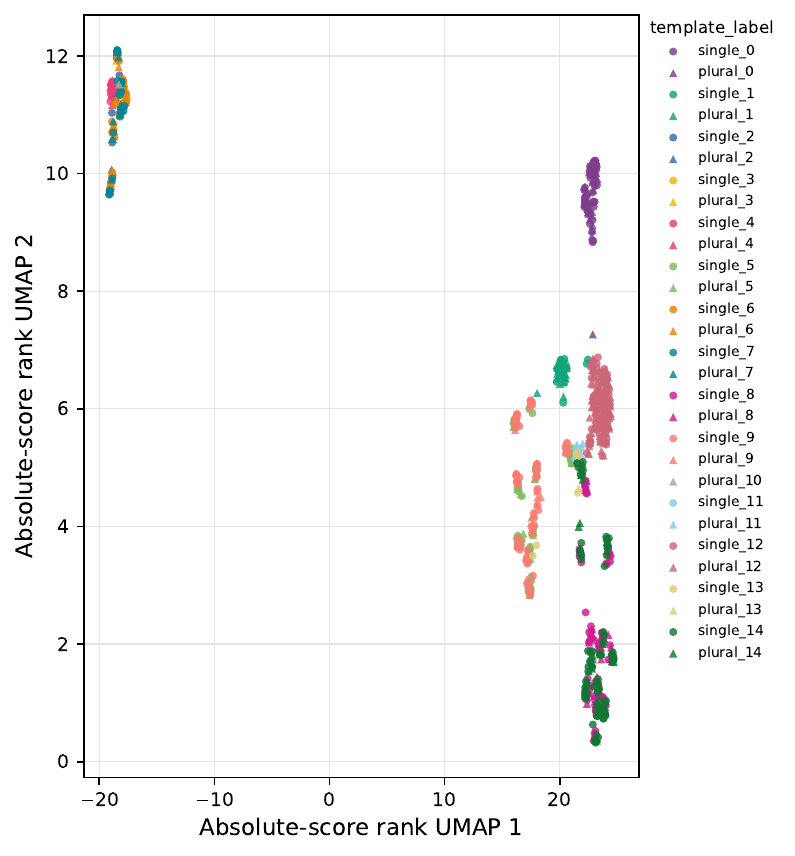}
        \caption{Absolute-score rank UMAP.}
        \label{fig:umap}
    \end{subfigure}
    \hfill
    \begin{subfigure}[t]{0.43\textwidth}
        \centering
        \includegraphics[width=\linewidth,height=0.22\textheight,keepaspectratio]{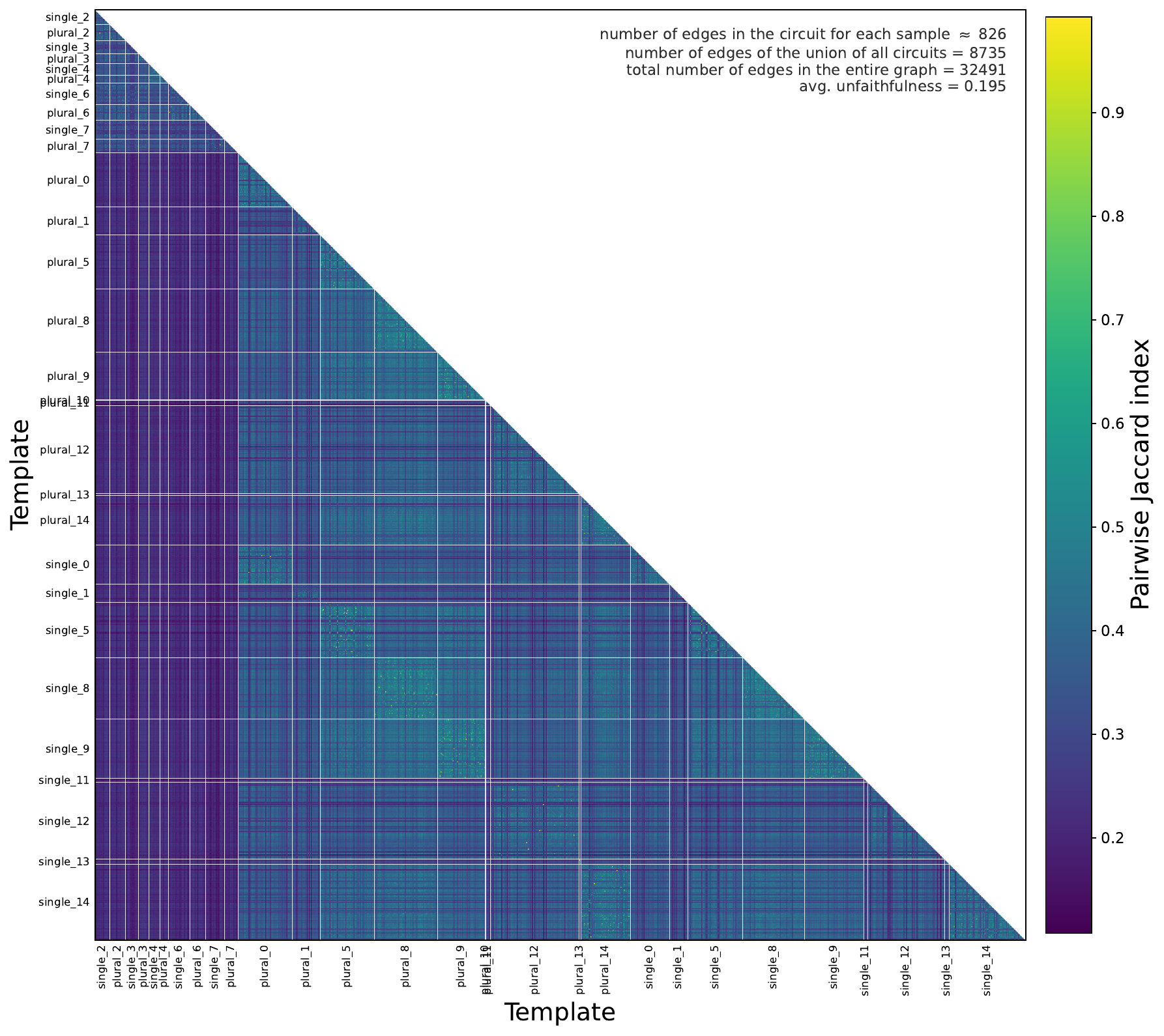}
        \caption{Pairwise Jaccard index.}
        \label{fig:sample-wise pji}
    \end{subfigure}
    \caption{Template-dependent sample circuits for GPT-2 small on SVA. The UMAP embeds samples by their absolute-score rank vectors, and the Pairwise Jaccard index matrix captures the overlap between the corresponding circuits.
    In the right panel, two clearly separated template groups emerge.
    Circuits within each group exhibit substantial overlap, while overlap between the two groups is almost zero.
    These two groups correspond to the two clusters in the left panel: one in the upper-left corner and the other containing the rest.
    The statistics in the upper-right of the right panel offer quantitative insight into how little overlap there is among the circuits for different templates.
    }
    \label{fig:gpt2-sva-template-umap-pji-main}
\end{figure}

Numerous prior studies \citep{faithfulness,meloux2025mechanistic,wang2025discovering,miller2024robust,nanda2022attribution,gao2025weight} have operated under the assumption that a single circuit within a given model is responsible for performing a specific task, and that, by evaluating edge importance via averaging over a sufficient number of samples, usually generated with multiple templates \citep{faithfulness,meloux2025mechanistic}, we can comprehensively pinpoint the circuit responsible for performing the task.
However, we have shown that such an assumption is flawed.
Different templates, in fact, trigger different circuits within the models.
Since the space of possible templates is effectively unbounded, it follows that \emph{identifying a comprehensive circuit for a particular task may be an unattainable goal}.

This implication challenges one of the central aspirations of mechanistic interpretability: that, by adjusting the identified task circuits, we can reliably steer the model’s behavior on that task in our preferred direction  \citep{zhang2026locatesteerimprovepractical}, including boosting math capability \citep{finetunecircuit}, moderating refusal \citep{cheng2026drives}, editing knowledge \citep{meng2022locating,meng2023massediting,knowledge_cirtuits_in_pretrained_transformers}, and many more \citep{gao2025weight,lindsey2025biology,wang2023interpretability}.
Although previous works showed that interventions on circuits identified from available samples can be effective \citep{gao2025weight,wang2023interpretability} and may improve benchmark performance \citep{finetunecircuit,cheng2026drives}, our observations here suggest that such a methodology does not guarantee reliable behavior in the wild, where the model may encounter unseen templates.\footnote{In \cref{app: fail to steer}, we present a case that, after encouraging plural verbs for a sample by patching a circuit obtained for this purpose from another sample, the model actually encourages singular verbs. }
In practice, the degree of concentration among template circuits can serve as a proxy for our confidence in the effect of a model intervention.

To deepen the discussion, we ask whether it is achievable or meaningful to pursue concentrated task circuits independent of templates, since "tasks" may seem an artificial concept, while templates are the concrete data the models process.
It may be quite demanding to require models to develop task-focused circuits rather than template-focused ones.
Nevertheless, a task can be thought of informally as a collection of prompts that a human is able to address using a single, unified algorithm,\footnote{For example, for IOI the algorithm is: find all names, remove duplicates, and output the rest \citep{wang2023interpretability}.} which indicates that it should be feasible to implement it using just one circuit.
The fact that current models resort to different circuits for different templates implies that the models leverage redundant algorithms.\footnote{A redundant algorithm may look like: \texttt{if handling template\_A: use algorithm\_A, elif handling template\_B: use algorithm\_B...} A good example of this for IOI is shown in Figure 3 of \citep{franco2026findinghighlyinterpretablepromptspecific}.}
A well-known principle in AI holds that intelligence can be understood as a form of compression \citep{wolff2013computingcompressionsptheory,schmidhuber2008driven}.
In particular, from the perspective of Kolmogorov complexity \citep{hutter2000theoryuniversalartificialintelligence}, intelligence can be viewed as the ability to find short descriptions of data that enable accurate prediction.
In this sense, models that rely on redundant algorithms to perform tasks have suboptimal intelligence, and the aim of training models that rely on a template-independent circuit for a given task aligns with the broader objective of discovering more compressive, thus intelligent, models.

A recent effort \citep{gao2025weight} to train more interpretable transformers showed that imposing sparsity during training shrinks the task circuits.
Prima facie, this seems like a potential solution to mitigate the unsteerability of models that we discussed above.
\emph{Could it be that sparsity forces the model to merge the circuits for different templates, thereby causing the task circuits to shrink?} 
Our experiments, based on the models from \citep{gao2025weight}, suggest that the benefit from sparsity is limited.

In our experiments, we adapted \texttt{transformer-lens} \citep{nanda2022transformerlens}, the interpretability infrastructure that we used for CEAP, to the specialized architectures of the sparse models built by \citep{gao2025weight}, which were customized for coding.  
We also created two datasets suitable for these models, single-double-quote and else-elif, summarized in \cref{app: dataset templates}.
Across these experiments, we do not find that sparsity causes the model to merge template circuits: sparse models still deploy different circuits for different templates.
Full visual comparisons and summary statistics are presented in \cref{app:sparse umap pji plot,app: template overlap summary tables}.

\begin{wrapfigure}[15]{r}{0.27\textwidth}
    \centering
    \vspace{-0.4em}
    \includegraphics[width=\linewidth,height=0.18\textheight,keepaspectratio]{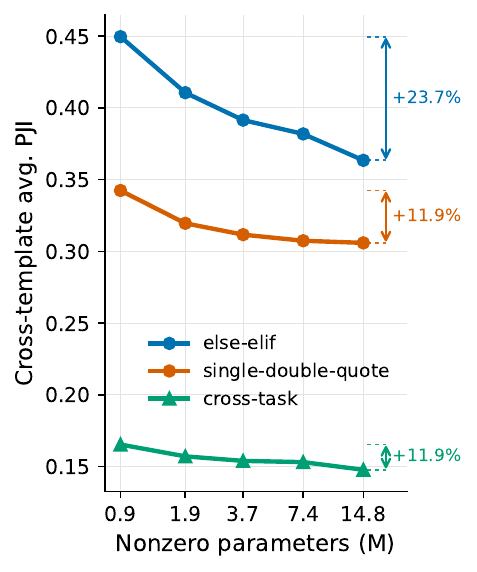}
    \caption{\footnotesize Cross-template average PJI increases with more sparsity, but so does cross-task average PJI.}
    \label{fig:cross-template-pji-vs-nonzero}
    \vspace{0.4em}
\end{wrapfigure}

In addition, we analyze how the mean PJI between circuit pairs obtained from different templates varies with increasing sparsity, while keeping all other settings fixed.
As shown in \cref{fig:cross-template-pji-vs-nonzero}, the overlap of different template circuits increases substantially when we allow fewer nonzero parameters.
Nevertheless, the average PJI for circuit pairs across different tasks (cross-task) also increases to an extent comparable to the mean cross-template PJI for single-double-quote.
This suggests that the level of polysemanticity increases with sparsity, which leads to the possibility that the increasing overlap between cross-template circuits for one task might be a result of the model using a larger portion of polysemantic neurons to implement redundant algorithms, rather than truly compressing for a single, unified algorithm for the task. This complicates task steering.
Determining whether sparsity primarily enhances polysemanticity or compression requires detailed model-level investigation, which lies beyond the scope of this variance-focused paper. We leave such an analysis for future work. Our current observations do not indicate that sparsity resolves the issue discussed in this section.

\FloatBarrier
\section{Variance of Unfaithfulness Across Samples}
\label{sec: sample-wise variance}
\citep{miller2024robust} noted that averaging edge scores over samples to identify circuits can yield large variance in per-sample unfaithfulness, even if the resulting circuit is faithful at the population level.\footnote{The notion of population-level faithfulness has varied in previous works \citep{miller2024robust,faithfulness,wang2023interpretability}, which we discuss in \cref{app: population-level unfaithfulness}.} 
One might think this is because the averaged scores better match some samples while fitting others less well.
However, it turns out that this is not the primary source of variance.
In our experiments, we observed that even if we perform circuit discovery with CEAP using a fixed circuit size \emph{for each sample separately}, thereby eliminating influence from other samples, the resulting unfaithfulness still fluctuates greatly across samples.
\cref{fig: u distribution} shows such an example, which we use throughout this section.
We see that the unfaithfulness $U$ of a sample (sample 78) may go as high as $20$, while most of the samples from the same template (\texttt{plural\_1}) already achieve low unfaithfulness.

In this section, we show that under a reasonable circuit discovery scheme, poor unfaithfulness scores typically stem from how unfaithfulness interacts with the distribution of edge scores, rather than from missing important edges.
Comparing unfaithfulness values between samples does not reveal which sample’s circuit is more sufficient for recovering the full model’s behavior.

\begin{wrapfigure}[14]{l}{0.24\textwidth}
    \centering
    \vspace{-0.2em}
    \includegraphics[width=0.96\linewidth]{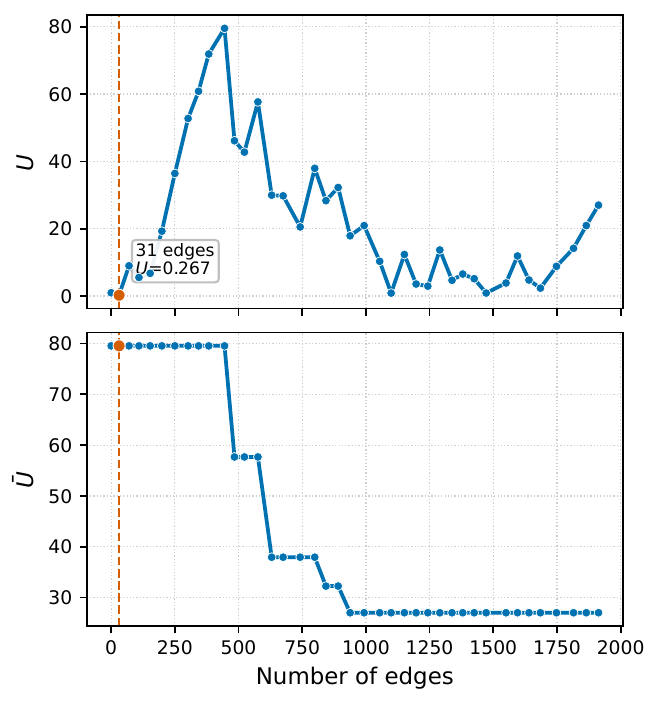}
    \caption{\footnotesize Normal and pessimistic unfaithfulness for sample 78 of SVA.}
    \label{fig:sva-sample78-u-vs-ubar}
    \vspace{0.2em}
\end{wrapfigure}
It is important to recognize that unfaithfulness measured on an individual sample is often hard to interpret, because it can vary in a highly non-monotonic way as the circuit size increases.
\cref{fig:sva-sample78-u-vs-ubar} shows how $U$ changes with the circuit size.
It reaches $0.267$ with just $31$ edges but shoots up to around $80$ when we increase the circuit size.
Ideally, however, unfaithfulness should consistently decrease when we admit more edges into the circuit.
Single-sample unfaithfulness fails to meet this because it evaluates the behavior of the model at a single point in the input space.
We detail this in \cref{app: single sample faithfulness}.
Highly fluctuating unfaithfulness can introduce noise to our analysis and lower statistical significance.
To reduce this noise, we define pessimistic unfaithfulness $\bar{U}$:
\begin{definition}
Let $U(m)$ denote the unfaithfulness of the circuit found by the algorithm with target edge number $m$. Pessimistic unfaithfulness is defined as $\bar{U}(m) = \max_{m\leq\hat{m}\leq\vert E\vert} U(\hat{m})$.
\end{definition}
In experiments, since we only evaluate a subset of all possible target edge counts, we estimate $\bar{U}(m)$ by taking the maximum unfaithfulness observed over the {evaluated} counts.

Next, we show that extremely poor $U$ should not be interpreted as a dramatic discrepancy between the circuit and the model's behavior.
We define the unnormalized unfaithfulness $U'\coloneqq \vert Q_{G'}(x_0,x_1) -  Q_{G}(x_0,x_1)\vert=\vert M_{G'}(x_1) -  M_{G}(x_1)\vert$ to directly measure the behavior gap between the circuit and the full model.
We also define the pessimistic unnormalized unfaithfulness $\bar U'$ by enforcing monotonicity.
In our experiments, we found that extremely high $\bar U$ is usually associated with near-zero $\vert Q_G\vert$,\footnote{In very rare cases, we find negative $Q_G$, meaning the full model is more likely to output the wrong answer.} as shown in \cref{fig:U_barVSQ}, where we obtained a negative $\rho$ for Spearman rank correlation (SRC).
Since $Q_G$ is the normalization for computing (un)faithfulness, it is natural to suspect whether the extremely small normalization values are what cause $\bar U$ to blow up for certain samples.
This suspicion is corroborated by a positive SRC coefficient $\rho$ between the unnormalized $\bar U'$ and $\vert Q_G\vert$,  illustrated in \cref{fig:U_bar_primeVSQ}.
\emph{These observations suggest that the extremely large $\bar U$ does not mean the logits produced by the chosen circuit differ wildly from those of the full model; rather, it is merely a consequence of normalization.}

\begin{figure}
    \centering
    \begin{subfigure}[t]{0.17\textwidth}
        \centering
        \raisebox{0.8em}{\includegraphics[width=\linewidth,height=0.155\textheight,keepaspectratio]{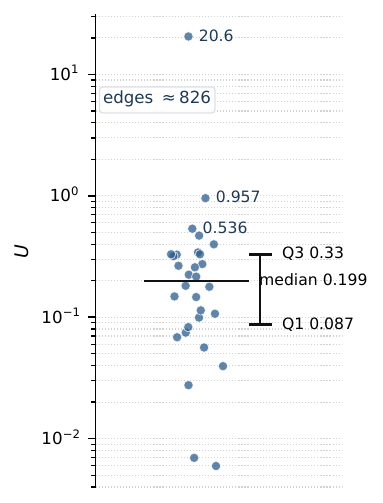}}
        \caption{$U$ distribution.}
        \label{fig: u distribution}
    \end{subfigure}%
    \hspace{0.005\textwidth}%
    \begin{subfigure}[t]{0.34\textwidth}
        \centering
        \includegraphics[width=\linewidth,height=0.155\textheight,keepaspectratio]{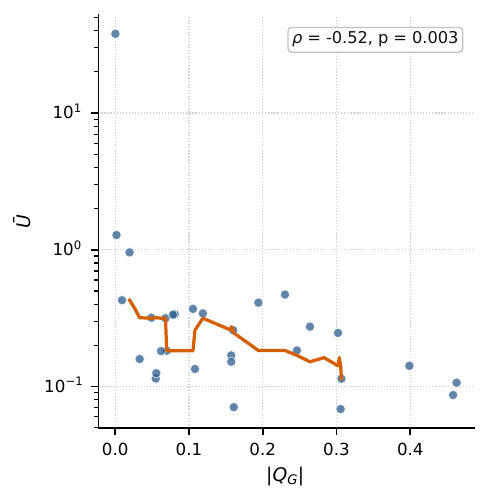}
        \caption{$\bar{U}$ vs.\ $|Q_G|$.}
                     \label{fig:U_barVSQ}
    \end{subfigure}%
    \hspace{0.005\textwidth}%
    \begin{subfigure}[t]{0.34\textwidth}
        \centering
        \includegraphics[width=\linewidth,height=0.155\textheight,keepaspectratio]{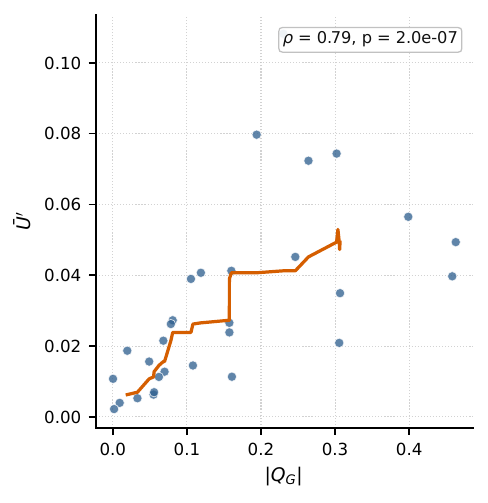}
        \caption{$\bar{U}'$ vs.\ $|Q_G|$.}
             \label{fig:U_bar_primeVSQ}

    \end{subfigure}
    \caption{For GPT-2 small on SVA template \texttt{plural\_1} with target edge number $826$. Here, we perform circuit discovery for each sample separately. The corresponding $U$ distribution for scores obtained via averaging within the template is in \cref{fig:u-distribution-all-averaging}. Samples with smaller $|Q_G|$ tend to show larger pessimistic unfaithfulness $\bar U$, while the pessimistic unnormalized unfaithfulness $\bar U'$ tends to increase with $|Q_G|$. Orange curves show rolling-median trends as visual guides only. Boxes report Spearman rank correlation $\rho$ and its $p$-value.}
    \label{fig: unfaithfulness exaggeration}
\end{figure} 

\subsection[The Correlation Between Ubar and |QG|]{The Correlation Between $\bar U$ and $|Q_G|$}

What remains unclear is why, given the same number of allowed edges, the circuits identified for samples with small $|Q_G|$ achieve much worse $\bar U$ than the circuits identified for samples with large $|Q_G|$.
Intuitively, this may be because as $|Q_G|$ decreases, the relative significance of the edges left out of the circuit, compared with those that are included, becomes greater.
The following analysis supports this explanation.

Given a model with each edge scored on sample $n$, we arrange the absolute-valued scores of all its edges into a nonincreasing vector $s^n$.
Let $S^n\coloneqq \sum_{\ell=1}^{\vert E \vert} s^n_\ell$, and $u^n = s^n/S^n$ so that all $u^n$ components sum up to $1$.
We define the normalized momentum to be $\mu^n  \coloneqq \sum_{\ell=1}^{\vert E \vert} u^n_\ell\times \ell$, which quantifies how heavy-tailed $s^n$ is.
From here on, for clarity, we introduce a superscript of sample index $n$ for $Q_G$.
We found that the samples with lower $\vert Q_G^n\vert$ have larger $\mu^n$ (\cref{fig:muVSQG}).
Suppose we consider two samples such that $\vert Q_G^{n_1} \vert>\vert Q_G^{n_2} \vert$. Then $u^{n_1}$ and $u^{n_2}$ would appear as illustrated in the lower panel of \cref{fig:long-tail-mental-image}.
Given an edge number $m$, our circuit discovery algorithm approximately selects the edges with the highest absolute scores to form the circuit.
Define $C^n(m)\coloneqq \sum_{\ell\in L}s^n_\ell$ where $L$ is the set of indices of $s^n$ that correspond to the edges admitted into the circuit.
Then, $C^n(m)/S^n\coloneqq R^n(m)$ can be interpreted as the ratio of score mass included in the circuit.
Naturally, $\mu^n$ and $R^n(m)$ should have a negative correlation (\cref{fig:RVSmu}).
Moreover, as $R^n(m)$ carries the meaning of the aggregate importance of edges included in the circuit \emph{relative} to total importance of all edges, it should have a negative correlation with unfaithfulness $\bar U$, which reflects the change of logit behavior caused by including only the edges in the chosen circuit \emph{relative} to the original logit behavior caused by the full model (\cref{fig:UbarVSR}).
Based on these correlations, we derive the following relationships that account for the trend observed in \cref{fig:U_barVSQ}.
\begin{figure}
    \centering
    \begin{subfigure}{0.245\textwidth}
        \centering
        \includegraphics[width=\linewidth]{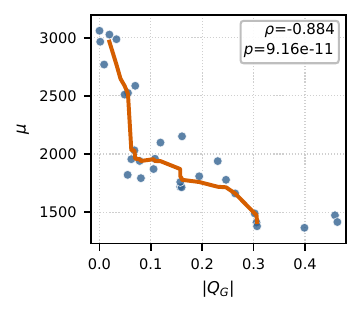}
        \caption{$\mu$ vs.\ $|Q_G|$.}
        \label{fig:muVSQG}
    \end{subfigure}\hfill
    \begin{subfigure}{0.245\textwidth}
        \centering
        \includegraphics[width=\linewidth]{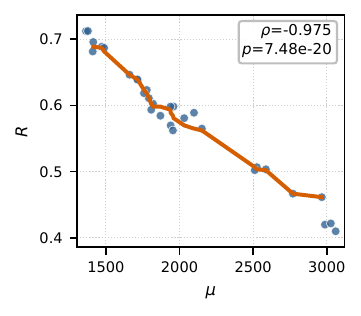}
        \caption{$R$ vs.\ $\mu$.}
        \label{fig:RVSmu}
    \end{subfigure}\hfill
    \begin{subfigure}{0.245\textwidth}
        \centering
        \includegraphics[width=\linewidth]{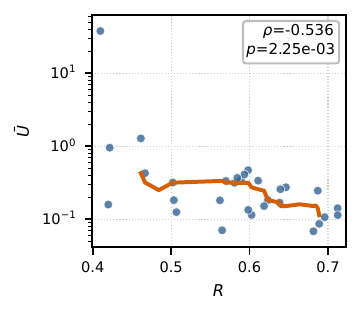}
        \caption{$\bar{U}$ vs.\ $R$.}
        \label{fig:UbarVSR}
    \end{subfigure}\hfill
    \begin{subfigure}{0.245\textwidth}
        \centering
        \includegraphics[width=\linewidth]{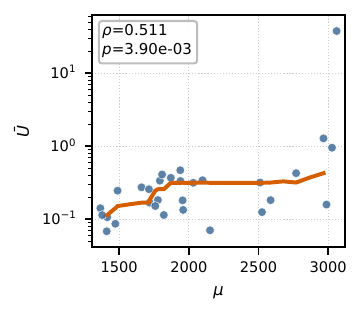}
        \caption{$\bar{U}$ vs.\ $\mu$.}
    \end{subfigure}
    \caption{Spearman rank correlations that explain \cref{fig:U_barVSQ}. The setup is exactly the same as in \cref{fig: unfaithfulness exaggeration}. The last panel, though not discussed, 
    adds supporting evidence: it suggests that the heavy-tailedness of $s^n$ is indicative of the unfaithfulness $\bar U$, combining the last two correlations in \cref{eq: arrow trend}.}
    \label{fig: long-tail mechanism plot}
\end{figure}

\begin{equation}
     \vert Q_G^n\vert \searrow ~~~ \boldsymbol{\longrightarrow} ~~~\mu^n \nearrow~~~ \boldsymbol{\longrightarrow}~~~{R^n(m)} \searrow ~~~\boldsymbol{\longrightarrow} ~~~\bar U \nearrow
     \label{eq: arrow trend}
\end{equation}

While the last two arrows in the above appear intuitive, the mechanism underlying the first one is unclear.
We show that this is due to a neural mechanism that we term \emph{selective contribution scaling}, which we cover in \cref{sec: selective contribution scaling}.

We verify all the correlation trends in this section, including the two in \cref{fig: unfaithfulness exaggeration} and the four in \cref{fig: long-tail mechanism plot}, for all templates of SVA/IOI/greater-than, for both GPT-2 small and Pythia-160M, and for all graph sizes that we swept. The trends hold up well.
Moreover, we check these trends using $U$ in place of $\bar{U}$, in case our design of $\bar{U}$ introduces undesirable artifacts.
$U$-based plots generally show the same trends, despite being noisier.
Note that the discussion of this section so far is based on scoring samples separately.
We also produce all the visualizations for scores obtained by averaging samples within their own templates, which is closer to common practice \citep{faithfulness,meloux2025mechanistic,miller2024robust}. 
Those experiments are well captured by the conclusion drawn in this section, as the circuits found within a template are highly similar in most cases. 
All the results are presented in \cref{app:extended-correlation-diagnostics}.

\subsection{A Mental Image for Selective Contribution Scaling}
\label{sec: selective contribution scaling}

Our conductance score satisfies that, if a set of edges is necessary and sufficient to produce the output (\cref{fig:circuit-diagram}), their scores must sum up to $Q_G^n$ (known as partition consistency \citep{dhamdhere2018how}).
In this sense, one can interpret the score as the contribution of the edge to the output.
Suppose we have a pair of samples with $\vert Q_G^{n_1} \vert>\vert Q_G^{n_2} \vert$. It is reasonable to expect that the entries in $s^{n_1}$ are larger than those in $s^{n_2}$, given partition consistency.
Since we know in hindsight that $u^{n_2}$ is more heavy-tailed than $u^{n_1}$ (\cref{fig:muVSQG}), we can infer that $s^{n_1}$ is considerably higher than $s^{n_2}$ at low indices, while the differences between them at high indices are mild.
Namely, if we move the input from sample $n_2$ to $n_1$ to increase the output signal amplitude from $\vert Q_G^{n_2}\vert$ to $\vert Q_G^{n_1}\vert$, the important edges located at low indices will contribute more than the less important ones at high indices.\footnote{Note that we are implicitly assuming that the edge orderings in $s^{n_1}$ and $s^{n_2}$ coincide. While this is not strictly true, they should be very similar, because, as we show in \cref{sec: template-induced circuit variance}, samples sharing the same template tend to exhibit similar scoring patterns.}
This is what we call selective contribution scaling, as illustrated in the upper panel of \cref{fig:long-tail-mental-image}.
We verify the mental image of \cref{fig:long-tail-mental-image} with a concrete sample pair in \cref{app: mental image corroboration}.

\begin{wrapfigure}[10]{r}{0.23\textwidth}
    \centering
    \vspace{-2.5em}
    \includegraphics[width=\linewidth]{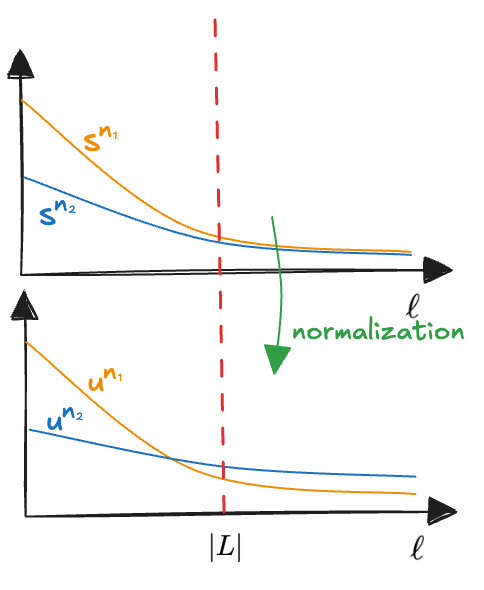}
    \caption{\footnotesize Mental image of selective contribution scaling.}
    \label{fig:long-tail-mental-image}
    \vspace{-1.4em}
\end{wrapfigure}
That there exist specialized parts of the model reacting to structurally specific stimuli (tasks expressed in templates) and other parts being indifferent to them aligns with the notion of sparse coding, a widely endorsed framework of information processing in neuroscience \citep{barlow,Olshausen1996EmergenceOS,OLSHAUSEN19973311,sparse,lonnqvist2023latent}.
Nonetheless, in our case, it links unfaithfulness to the insignificant magnitude of $Q_G$, which weakens unfaithfulness as a definitive marker of performance recovery.
By the same token, we also anticipate great cross-sample variance in faithfulness-style normalization-dependent metrics measured on systems that exhibit sparse coding.

\section{Conclusion and Discussion}
In this paper, we systematically investigate the three types of variance in circuit discovery algorithms.
While resampling variance can be mitigated with more principled scoring, and sample-wise variance of unfaithfulness is largely benign, rephrasing variance is hard to diminish and suggests the impossibility of finding comprehensive task circuits.

In discussing the above, we develop methods and guidance that can inform future practice.
First, conductance is a more principled way of scoring compared to IG, as it satisfies additive order preservation.
Second, although evaluating (un)faithfulness is relatively straightforward, and the notion itself is intuitively plausible \citep{faithfulness}, there is a caveat: comparisons of this metric across different samples are subject to an intrinsic sample-wise variance induced by selective contribution scaling.
We further discuss in \cref{app: population-level unfaithfulness} that population-level (un)faithfulness should explicitly account for this effect and be more lenient toward samples whose behavioral magnitude is small.
Nonetheless, (un)faithfulness remains informative for tracking how circuit quality improves as we increase the permitted circuit size.
In this regard, pessimism serves as a useful complement to reduce interpretational noise.

Furthermore, our discussion on template-based rephrasing variance sheds new light on the generalizability of \emph{task} circuits \citep{sharkey2025open}.
On the one hand, it suggests that statements involving task circuits should be carefully tested for robustness under changes in templates,\footnote{Some previous studies have considered robustness to paraphrasing \citep{meng2022locating,meng2023massediting}. However, it has not yet become a standard practice.} and that template-agnostic benchmarking can be inconclusive.
On the other hand, it opens up avenues for future exploration.
Let us define a task as the collection of samples that engage the same circuit.
This definition is useful because it is operational and may provide insight into circuit steering.
At present, despite certain post hoc heuristics \citep{franco2026findinghighlyinterpretablepromptspecific}, we lack a principled way to determine which templates engage the same circuit and which rely on distinct circuits.
As a result, the notion of a task remains elusive.
To address this, we must systematically characterize how templates influence the algorithmic procedures that models use to process them.
 Put differently, it is important to formalize a model-centered notion of a task, in contrast to the current notion, in which humans judge whether different instances should be considered part of the same task \citep{faithfulness,wang2023interpretability,hanna2023how,SVA}.
Furthermore, as the circuits currently appear to be template-dependent, an important question is whether models can be trained to use the same circuit for tasks defined in human terms, which may, in turn, lead to greater compression and potentially higher intelligence.

\section*{Acknowledgements}
 
We thank Dr. Paul Rolland for valuable discussions.
This work was funded  by the Swiss National Science Foundation (SNSF) under grant number 2000-1-240094.
This work was also supported under project ID \# 37 as part of the Swiss AI Initiative, through a grant from the ETH Domain and computational resources provided by the Swiss National Supercomputing Centre (CSCS) under the Alps infrastructure.

\vspace{-0.4em}
\bibliography{reference}
\bibliographystyle{unsrtnat}

\newpage
\appendix
\captionsetup[figure]{font=normalsize}
\crefname{appendix}{Appendix}{Appendices}
\Crefname{appendix}{Appendix}{Appendices}
\crefalias{section}{appendix}
\onecolumn
\raggedbottom

\section{Experimental Details}
\label{app:experiments}

Following previous conventions \citep{faithfulness,nanda2022attribution}, we used \texttt{logit\_diff} for tasks whose correct and incorrect answers are single logits (IOI, single-double-quote).
For other tasks, we used \texttt{prob\_diff}.

For GPT-2 and Pythia experiments, we used integration step $K=200$ for both EAP-IG and CEAP.
For sparse-transformer experiments \citep{gao2025weight}, we used $150$.

\section{Code and Data Availability}
\label{app:code-data-availability}

A ready-to-use implementation of CEAP, along with code to reproduce the figures presented in this paper, is available at \url{https://github.com/LIONS-EPFL/Circuit-Variance}.

\section{Numerical Issue Caused by KL Divergence}
\label{app: KL}
\citep{faithfulness,conmy2023towards} also recommended using the KL divergence as a metric, as it is in principle applicable to all tasks.
In practice, however, we found KL divergence to be restrictive: it cannot be used for models that assign zero probability to some tokens, since this leads to an infinite KL value, which we observe \emph{e.g.}, for Pythia-160M.
While $\mathrm{Softmax(\cdot)}$ should theoretically ensure that every token has nonzero probability, numerical underflow (which occurs for both \texttt{fp32} and \texttt{fp64}) frequently produces exact zeros in the probabilities.
To keep the exposition uncluttered, we therefore report only the \texttt{logit}/\texttt{prob diff} metrics in the main paper.
The experiments that we carried out for KL divergence do not show qualitatively different results from the conclusions we present here.

\section{Dataset Templates}
\label{app: dataset templates}

We use template-generated datasets so that clean and corrupted prompts differ in a controlled way.
\cref{tab:sva-templates,tab:ioi-templates,tab:greater-than-template,tab:single-double-quote-templates,tab:else-elif-templates} summarize the templates used in our experiments.
Placeholders in square brackets are filled with names, nouns, variable names, or code literals during dataset construction.
For the tasks customized for sparse-transformers \citep{gao2025weight}, each base template generates a paired contrast.
For \texttt{single\_double\_quote}, we show both paired directions explicitly using the \texttt{\_s} and \texttt{\_d} suffixes.
For \texttt{else\_elif}, we show one representative direction in the table and omit the \texttt{\_else}/\texttt{\_elif} suffixes from the displayed template names for space; the generated data and visualizations keep those suffixes.

\FloatBarrier
\begin{table}[H]
\centering
\caption{SVA dataset templates. Each row shows a prompt contrast where the correct answer is either a singular or plural verb form. In the placeholders, $\mathrm{sg}$ and $\mathrm{pl}$ denote singular and plural forms, respectively.}
\label{tab:sva-templates}
\scriptsize
\setlength{\tabcolsep}{3pt}
\begin{tabular}{@{}!{\vrule width 0.9pt}p{0.13\linewidth}p{0.32\linewidth}p{0.32\linewidth}p{0.10\linewidth}p{0.10\linewidth}!{\vrule width 0.9pt}@{}}
\datasetTopRule
\datasetHeaderRow
\datasetRowRule
\texttt{single\_0} & The [Noun$_\mathrm{sg}$] that [Verb$_\mathrm{sg}$] the [Noun] & The [Noun$_\mathrm{pl}$] that [Verb$_\mathrm{pl}$] the [Noun] & singular verb & plural verb \\
\datasetRowRule
\texttt{plural\_0} & The [Noun$_\mathrm{pl}$] that [Verb$_\mathrm{pl}$] the [Noun] & The [Noun$_\mathrm{sg}$] that [Verb$_\mathrm{sg}$] the [Noun] & plural verb & singular verb \\
\datasetRowRule
\texttt{single\_1} & The [Noun$_\mathrm{sg}$] from the [Noun$_\mathrm{pl}$] & The [Noun$_\mathrm{pl}$] from the [Noun$_\mathrm{pl}$] & singular verb & plural verb \\
\datasetRowRule
\texttt{plural\_1} & The [Noun$_\mathrm{pl}$] from the [Noun$_\mathrm{pl}$] & The [Noun$_\mathrm{sg}$] from the [Noun$_\mathrm{pl}$] & plural verb & singular verb \\
\datasetRowRule
\texttt{single\_2} & The [Noun$_\mathrm{pl}$] that the [Noun$_\mathrm{sg}$] & The [Noun$_\mathrm{pl}$] that the [Noun$_\mathrm{pl}$] & singular verb & plural verb \\
\datasetRowRule
\texttt{plural\_2} & The [Noun$_\mathrm{pl}$] that the [Noun$_\mathrm{pl}$] & The [Noun$_\mathrm{pl}$] that the [Noun$_\mathrm{sg}$] & plural verb & singular verb \\
\datasetRowRule
\texttt{single\_3} & The [Noun$_\mathrm{pl}$] the [Noun$_\mathrm{sg}$] & The [Noun$_\mathrm{pl}$] the [Noun$_\mathrm{pl}$] & singular verb & plural verb \\
\datasetRowRule
\texttt{plural\_3} & The [Noun$_\mathrm{pl}$] the [Noun$_\mathrm{pl}$] & The [Noun$_\mathrm{pl}$] the [Noun$_\mathrm{sg}$] & plural verb & singular verb \\
\datasetRowRule
\texttt{single\_4} & The [Noun$_\mathrm{pl}$] said the [Noun$_\mathrm{sg}$] & The [Noun$_\mathrm{pl}$] said the [Noun$_\mathrm{pl}$] & singular verb & plural verb \\
\datasetRowRule
\texttt{plural\_4} & The [Noun$_\mathrm{pl}$] said the [Noun$_\mathrm{pl}$] & The [Noun$_\mathrm{pl}$] said the [Noun$_\mathrm{sg}$] & plural verb & singular verb \\
\datasetRowRule
\texttt{single\_5} & The [Noun$_\mathrm{sg}$] the [Noun$_\mathrm{pl}$] [Verb] & The [Noun$_\mathrm{pl}$] the [Noun$_\mathrm{pl}$] [Verb] & singular verb & plural verb \\
\datasetRowRule
\texttt{plural\_5} & The [Noun$_\mathrm{pl}$] the [Noun$_\mathrm{pl}$] [Verb] & The [Noun$_\mathrm{sg}$] the [Noun$_\mathrm{pl}$] [Verb] & plural verb & singular verb \\
\datasetRowRule
\texttt{single\_6} & The [Noun$_\mathrm{pl}$] that the [Noun$_\mathrm{sg}$] & The [Noun$_\mathrm{pl}$] that the [Noun$_\mathrm{pl}$] & singular verb & plural verb \\
\datasetRowRule
\texttt{plural\_6} & The [Noun$_\mathrm{pl}$] that the [Noun$_\mathrm{pl}$] & The [Noun$_\mathrm{pl}$] that the [Noun$_\mathrm{sg}$] & plural verb & singular verb \\
\datasetRowRule
\texttt{single\_7} & The [Noun$_\mathrm{pl}$] the [Noun$_\mathrm{sg}$] & The [Noun$_\mathrm{pl}$] the [Noun$_\mathrm{pl}$] & singular verb & plural verb \\
\datasetRowRule
\texttt{plural\_7} & The [Noun$_\mathrm{pl}$] the [Noun$_\mathrm{pl}$] & The [Noun$_\mathrm{pl}$] the [Noun$_\mathrm{sg}$] & plural verb & singular verb \\
\datasetRowRule
\texttt{single\_8} & The [Noun$_\mathrm{sg}$] that the [Noun$_\mathrm{pl}$] [Verb] & The [Noun$_\mathrm{pl}$] that the [Noun$_\mathrm{pl}$] [Verb] & singular verb & plural verb \\
\datasetRowRule
\texttt{plural\_8} & The [Noun$_\mathrm{pl}$] that the [Noun$_\mathrm{pl}$] [Verb] & The [Noun$_\mathrm{sg}$] that the [Noun$_\mathrm{pl}$] [Verb] & plural verb & singular verb \\
\datasetRowRule
\texttt{single\_9} & The [Noun$_\mathrm{sg}$] that the [Noun$_\mathrm{pl}$] [Verb] & The [Noun$_\mathrm{pl}$] that the [Noun$_\mathrm{pl}$] [Verb] & singular verb & plural verb \\
\datasetRowRule
\texttt{plural\_9} & The [Noun$_\mathrm{pl}$] that the [Noun$_\mathrm{pl}$] [Verb] & The [Noun$_\mathrm{sg}$] that the [Noun$_\mathrm{pl}$] [Verb] & plural verb & singular verb \\
\datasetRowRule
\texttt{single\_10} & The [Noun$_\mathrm{sg}$] & The [Noun$_\mathrm{pl}$] & singular verb & plural verb \\
\datasetRowRule
\texttt{plural\_10} & The [Noun$_\mathrm{pl}$] & The [Noun$_\mathrm{sg}$] & plural verb & singular verb \\
\datasetRowRule
\texttt{single\_11} & The [Noun$_\mathrm{sg}$] [Verb$_\mathrm{sg}$] [VP] and & The [Noun$_\mathrm{pl}$] [Verb$_\mathrm{pl}$] [VP] and & singular verb & plural verb \\
\datasetRowRule
\texttt{plural\_11} & The [Noun$_\mathrm{pl}$] [Verb$_\mathrm{pl}$] [VP] and & The [Noun$_\mathrm{sg}$] [Verb$_\mathrm{sg}$] [VP] and & plural verb & singular verb \\
\datasetRowRule
\texttt{single\_12} & The [Noun$_\mathrm{sg}$] next to the [Noun$_\mathrm{pl}$] & The [Noun$_\mathrm{pl}$] next to the [Noun$_\mathrm{pl}$] & singular verb & plural verb \\
\datasetRowRule
\texttt{plural\_12} & The [Noun$_\mathrm{pl}$] next to the [Noun$_\mathrm{pl}$] & The [Noun$_\mathrm{sg}$] next to the [Noun$_\mathrm{pl}$] & plural verb & singular verb \\
\datasetRowRule
\texttt{single\_13} & The [Noun$_\mathrm{sg}$] [Verb$_\mathrm{sg}$] and & The [Noun$_\mathrm{pl}$] [Verb$_\mathrm{pl}$] and & singular verb & plural verb \\
\datasetRowRule
\texttt{plural\_13} & The [Noun$_\mathrm{pl}$] [Verb$_\mathrm{pl}$] and & The [Noun$_\mathrm{sg}$] [Verb$_\mathrm{sg}$] and & plural verb & singular verb \\
\datasetRowRule
\texttt{single\_14} & The [Noun$_\mathrm{sg}$] the [Noun$_\mathrm{pl}$] [Verb] & The [Noun$_\mathrm{pl}$] the [Noun$_\mathrm{pl}$] [Verb] & singular verb & plural verb \\
\datasetRowRule
\texttt{plural\_14} & The [Noun$_\mathrm{pl}$] the [Noun$_\mathrm{pl}$] [Verb] & The [Noun$_\mathrm{sg}$] the [Noun$_\mathrm{pl}$] [Verb] & plural verb & singular verb \\
\datasetBottomRule
\end{tabular}
\end{table}
\FloatBarrier

\begin{table}[H]
\centering
\caption{IOI dataset templates. Each clean prompt is truncated immediately before the indirect-object answer; the corrupted prompt replaces the cueing name with a distractor [C].}
\label{tab:ioi-templates}
\scriptsize
\setlength{\tabcolsep}{3pt}
\begin{tabular}{@{}!{\vrule width 0.9pt}p{0.12\linewidth}p{0.31\linewidth}p{0.31\linewidth}p{0.10\linewidth}p{0.10\linewidth}!{\vrule width 0.9pt}@{}}
\datasetTopRule
\datasetHeaderRow
\datasetRowRule
\texttt{ABBA\_00} & Then, [A] and [B] went to the [PLACE]. [B] gave a [OBJECT] to & Then, [A] and [B] went to the [PLACE]. [C] gave a [OBJECT] to & [A] & [B] \\
\datasetRowRule
\texttt{ABBA\_01} & Then, [A] and [B] had a lot of fun at the [PLACE]. [B] gave a [OBJECT] to & Then, [A] and [B] had a lot of fun at the [PLACE]. [C] gave a [OBJECT] to & [A] & [B] \\
\datasetRowRule
\texttt{ABBA\_02} & Then, [A] and [B] were working at the [PLACE]. [B] decided to give a [OBJECT] to & Then, [A] and [B] were working at the [PLACE]. [C] decided to give a [OBJECT] to & [A] & [B] \\
\datasetRowRule
\texttt{ABBA\_03} & Then, [A] and [B] were thinking about going to the [PLACE]. [B] wanted to give a [OBJECT] to & Then, [A] and [B] were thinking about going to the [PLACE]. [C] wanted to give a [OBJECT] to & [A] & [B] \\
\datasetRowRule
\texttt{ABBA\_04} & Then, [A] and [B] had a long argument, and afterwards [B] said to & Then, [A] and [B] had a long argument, and afterwards [C] said to & [A] & [B] \\
\datasetRowRule
\texttt{ABBA\_05} & After [A] and [B] went to the [PLACE], [B] gave a [OBJECT] to & After [A] and [B] went to the [PLACE], [C] gave a [OBJECT] to & [A] & [B] \\
\datasetRowRule
\texttt{ABBA\_06} & When [A] and [B] got a [OBJECT] at the [PLACE], [B] decided to give it to & When [A] and [B] got a [OBJECT] at the [PLACE], [C] decided to give it to & [A] & [B] \\
\datasetRowRule
\texttt{BABA\_00} & Then, [B] and [A] went to the [PLACE]. [B] gave a [OBJECT] to & Then, [B] and [A] went to the [PLACE]. [C] gave a [OBJECT] to & [A] & [B] \\
\datasetRowRule
\texttt{BABA\_01} & Then, [B] and [A] had a lot of fun at the [PLACE]. [B] gave a [OBJECT] to & Then, [B] and [A] had a lot of fun at the [PLACE]. [C] gave a [OBJECT] to & [A] & [B] \\
\datasetRowRule
\texttt{BABA\_02} & Then, [B] and [A] were working at the [PLACE]. [B] decided to give a [OBJECT] to & Then, [B] and [A] were working at the [PLACE]. [C] decided to give a [OBJECT] to & [A] & [B] \\
\datasetRowRule
\texttt{BABA\_03} & Then, [B] and [A] were thinking about going to the [PLACE]. [B] wanted to give a [OBJECT] to & Then, [B] and [A] were thinking about going to the [PLACE]. [C] wanted to give a [OBJECT] to & [A] & [B] \\
\datasetRowRule
\texttt{BABA\_04} & Then, [B] and [A] had a long argument, and afterwards [B] said to & Then, [B] and [A] had a long argument, and afterwards [C] said to & [A] & [B] \\
\datasetRowRule
\texttt{BABA\_05} & After [B] and [A] went to the [PLACE], [B] gave a [OBJECT] to & After [B] and [A] went to the [PLACE], [C] gave a [OBJECT] to & [A] & [B] \\
\datasetRowRule
\texttt{BABA\_06} & When [B] and [A] got a [OBJECT] at the [PLACE], [B] decided to give it to & When [B] and [A] got a [OBJECT] at the [PLACE], [C] decided to give it to & [A] & [B] \\
\datasetBottomRule
\end{tabular}
\end{table}
\FloatBarrier

\begin{table}[H]
\centering
\caption{greater-than dataset template example. The task asks whether the next two-digit year suffix is greater than the suffix $xy$ implied by the clean prompt.}
\label{tab:greater-than-template}
\scriptsize
\setlength{\tabcolsep}{3pt}
\begin{tabular}{@{}!{\vrule width 0.9pt}p{0.12\linewidth}p{0.31\linewidth}p{0.31\linewidth}p{0.10\linewidth}p{0.10\linewidth}!{\vrule width 0.9pt}@{}}
\datasetTopRule
\datasetHeaderRow
\datasetRowRule
\texttt{x} & The [NOUN] lasted from the year 13xy to the year 13 & The [NOUN] lasted from the year 1301 to the year 13 & suffix $>xy$ & suffix $\leq xy$ \\
\datasetBottomRule
\end{tabular}
\end{table}
\FloatBarrier

\begin{table}[H]
\centering
\caption{single-double-quote dataset templates. Each base template is shown in both directions: \texttt{\_s} expects the closing single-quote-parenthesis token, and \texttt{\_d} expects the closing double-quote-parenthesis token.}
\label{tab:single-double-quote-templates}
\tiny
\setlength{\tabcolsep}{1.5pt}
\renewcommand{\arraystretch}{1.08}
\begin{tabular}{@{}!{\vrule width 0.9pt}p{0.145\linewidth}p{0.335\linewidth}p{0.335\linewidth}p{0.075\linewidth}p{0.075\linewidth}!{\vrule width 0.9pt}@{}}
\datasetTopRule
\datasetHeaderRow
\datasetRowRule
\texttt{append\_call\_s} & \codeblock{\texttt{if [date bounds fail]:}\\\codeindent\texttt{[RESULT\_LIST].append(\asciisq{}Invalid Input}} & \codeblock{\texttt{if [date bounds fail]:}\\\codeindent\texttt{[RESULT\_LIST].append(\asciidq{}Invalid Input}} & \texttt{\asciisq)} & \texttt{\asciidq)} \\
\datasetRowRule
\texttt{append\_call\_d} & \codeblock{\texttt{if [date bounds fail]:}\\\codeindent\texttt{[RESULT\_LIST].append(\asciidq{}Invalid Input}} & \codeblock{\texttt{if [date bounds fail]:}\\\codeindent\texttt{[RESULT\_LIST].append(\asciisq{}Invalid Input}} & \texttt{\asciidq)} & \texttt{\asciisq)} \\
\datasetRowRule
\texttt{constructor\_call\_s} & \codeblock{\texttt{[OBJECT\_VAR] = Project(}\\\codeindent\texttt{\asciisq[PROJECT\_NAME]}} & \codeblock{\texttt{[OBJECT\_VAR] = Project(}\\\codeindent\texttt{\asciidq[PROJECT\_NAME]}} & \texttt{\asciisq)} & \texttt{\asciidq)} \\
\datasetRowRule
\texttt{constructor\_call\_d} & \codeblock{\texttt{[OBJECT\_VAR] = Project(}\\\codeindent\texttt{\asciidq[PROJECT\_NAME]}} & \codeblock{\texttt{[OBJECT\_VAR] = Project(}\\\codeindent\texttt{\asciisq[PROJECT\_NAME]}} & \texttt{\asciidq)} & \texttt{\asciisq)} \\
\datasetRowRule
\texttt{for\_loop\_s} & \codeblock{\texttt{for [ITEM\_VAR] in [ITEM\_SET]:}\\\codeindent\texttt{if len([ITEM\_VAR]) == 0:}\\\codeindent\codeindent\texttt{print(\asciisq{}invalid}} & \codeblock{\texttt{for [ITEM\_VAR] in [ITEM\_SET]:}\\\codeindent\texttt{if len([ITEM\_VAR]) == 0:}\\\codeindent\codeindent\texttt{print(\asciidq{}invalid}} & \texttt{\asciisq)} & \texttt{\asciidq)} \\
\datasetRowRule
\texttt{for\_loop\_d} & \codeblock{\texttt{for [ITEM\_VAR] in [ITEM\_SET]:}\\\codeindent\texttt{if len([ITEM\_VAR]) == 0:}\\\codeindent\codeindent\texttt{print(\asciidq{}invalid}} & \codeblock{\texttt{for [ITEM\_VAR] in [ITEM\_SET]:}\\\codeindent\texttt{if len([ITEM\_VAR]) == 0:}\\\codeindent\codeindent\texttt{print(\asciisq{}invalid}} & \texttt{\asciidq)} & \texttt{\asciisq)} \\
\datasetRowRule
\texttt{function\_call\_s} & \codeblock{\texttt{[FUNC\_NAME](\asciisq[FILE\_NAME]}} & \codeblock{\texttt{[FUNC\_NAME](\asciidq[FILE\_NAME]}} & \texttt{\asciisq)} & \texttt{\asciidq)} \\
\datasetRowRule
\texttt{function\_call\_d} & \codeblock{\texttt{[FUNC\_NAME](\asciidq[FILE\_NAME]}} & \codeblock{\texttt{[FUNC\_NAME](\asciisq[FILE\_NAME]}} & \texttt{\asciidq)} & \texttt{\asciisq)} \\
\datasetRowRule
\texttt{if\_clause\_1\_s} & \codeblock{\texttt{if [COND]:}\\\codeindent\texttt{print(\asciisq{}There is no topological sort,}\\\codeindent\texttt{the graph has a cycle}} & \codeblock{\texttt{if [COND]:}\\\codeindent\texttt{print(\asciidq{}There is no topological sort,}\\\codeindent\texttt{the graph has a cycle}} & \texttt{\asciisq)} & \texttt{\asciidq)} \\
\datasetRowRule
\texttt{if\_clause\_1\_d} & \codeblock{\texttt{if [COND]:}\\\codeindent\texttt{print(\asciidq{}There is no topological sort,}\\\codeindent\texttt{the graph has a cycle}} & \codeblock{\texttt{if [COND]:}\\\codeindent\texttt{print(\asciisq{}There is no topological sort,}\\\codeindent\texttt{the graph has a cycle}} & \texttt{\asciidq)} & \texttt{\asciisq)} \\
\datasetRowRule
\texttt{if\_clause\_2\_s} & \codeblock{\texttt{if degree not in range(3, [LIMIT\_VAR]-1):}\\\codeindent\texttt{print(\asciisq{}Degree should be in the range}\\\codeindent\texttt{[3, [LIMIT\_VAR]-1]}} & \codeblock{\texttt{if degree not in range(3, [LIMIT\_VAR]-1):}\\\codeindent\texttt{print(\asciidq{}Degree should be in the range}\\\codeindent\texttt{[3, [LIMIT\_VAR]-1]}} & \texttt{\asciisq)} & \texttt{\asciidq)} \\
\datasetRowRule
\texttt{if\_clause\_2\_d} & \codeblock{\texttt{if degree not in range(3, [LIMIT\_VAR]-1):}\\\codeindent\texttt{print(\asciidq{}Degree should be in the range}\\\codeindent\texttt{[3, [LIMIT\_VAR]-1]}} & \codeblock{\texttt{if degree not in range(3, [LIMIT\_VAR]-1):}\\\codeindent\texttt{print(\asciisq{}Degree should be in the range}\\\codeindent\texttt{[3, [LIMIT\_VAR]-1]}} & \texttt{\asciidq)} & \texttt{\asciisq)} \\
\datasetRowRule
\texttt{if\_clause\_3\_s} & \codeblock{\texttt{if [ALT\_VAR] < [DIST\_VAR].get(}\\\codeindent\texttt{[NODE\_VAR], float(\asciisq{}inf}} & \codeblock{\texttt{if [ALT\_VAR] < [DIST\_VAR].get(}\\\codeindent\texttt{[NODE\_VAR], float(\asciidq{}inf}} & \texttt{\asciisq)} & \texttt{\asciidq)} \\
\datasetRowRule
\texttt{if\_clause\_3\_d} & \codeblock{\texttt{if [ALT\_VAR] < [DIST\_VAR].get(}\\\codeindent\texttt{[NODE\_VAR], float(\asciidq{}inf}} & \codeblock{\texttt{if [ALT\_VAR] < [DIST\_VAR].get(}\\\codeindent\texttt{[NODE\_VAR], float(\asciisq{}inf}} & \texttt{\asciidq)} & \texttt{\asciisq)} \\
\datasetRowRule
\texttt{method\_call\_s} & \codeblock{\texttt{[GRAPH\_VAR].read\_data(}\\\codeindent\texttt{\asciisq[GRAPH\_FILE]}} & \codeblock{\texttt{[GRAPH\_VAR].read\_data(}\\\codeindent\texttt{\asciidq[GRAPH\_FILE]}} & \texttt{\asciisq)} & \texttt{\asciidq)} \\
\datasetRowRule
\texttt{method\_call\_d} & \codeblock{\texttt{[GRAPH\_VAR].read\_data(}\\\codeindent\texttt{\asciidq[GRAPH\_FILE]}} & \codeblock{\texttt{[GRAPH\_VAR].read\_data(}\\\codeindent\texttt{\asciisq[GRAPH\_FILE]}} & \texttt{\asciidq)} & \texttt{\asciisq)} \\
\datasetRowRule
\texttt{nested\_loop\_s} & \codeblock{\texttt{for [OUTER] in [SET]:}\\\codeindent\texttt{for [INNER] in [OUTER]:}\\\codeindent\codeindent\texttt{if ([INNER] == ()):}\\\codeindent\codeindent\codeindent\texttt{print(\asciisq{}No [INNER] at}\\\codeindent\codeindent\codeindent\texttt{the current iteration}} & \codeblock{\texttt{for [OUTER] in [SET]:}\\\codeindent\texttt{for [INNER] in [OUTER]:}\\\codeindent\codeindent\texttt{if ([INNER] == ()):}\\\codeindent\codeindent\codeindent\texttt{print(\asciidq{}No [INNER] at}\\\codeindent\codeindent\codeindent\texttt{the current iteration}} & \texttt{\asciisq)} & \texttt{\asciidq)} \\
\datasetRowRule
\texttt{nested\_loop\_d} & \codeblock{\texttt{for [OUTER] in [SET]:}\\\codeindent\texttt{for [INNER] in [OUTER]:}\\\codeindent\codeindent\texttt{if ([INNER] == ()):}\\\codeindent\codeindent\codeindent\texttt{print(\asciidq{}No [INNER] at}\\\codeindent\codeindent\codeindent\texttt{the current iteration}} & \codeblock{\texttt{for [OUTER] in [SET]:}\\\codeindent\texttt{for [INNER] in [OUTER]:}\\\codeindent\codeindent\texttt{if ([INNER] == ()):}\\\codeindent\codeindent\codeindent\texttt{print(\asciisq{}No [INNER] at}\\\codeindent\codeindent\codeindent\texttt{the current iteration}} & \texttt{\asciidq)} & \texttt{\asciisq)} \\
\datasetRowRule
\texttt{print\_stmt\_1\_s} & \codeblock{\codeindent\texttt{print(\asciisq[PRINT\_TEXT]}} & \codeblock{\codeindent\texttt{print(\asciidq[PRINT\_TEXT]}} & \texttt{\asciisq)} & \texttt{\asciidq)} \\
\datasetRowRule
\texttt{print\_stmt\_1\_d} & \codeblock{\codeindent\texttt{print(\asciidq[PRINT\_TEXT]}} & \codeblock{\codeindent\texttt{print(\asciisq[PRINT\_TEXT]}} & \texttt{\asciidq)} & \texttt{\asciisq)} \\
\datasetRowRule
\texttt{print\_stmt\_2\_s} & \codeblock{\texttt{print(\asciisq[USAGE\_TEXT]}} & \codeblock{\texttt{print(\asciidq[USAGE\_TEXT]}} & \texttt{\asciisq)} & \texttt{\asciidq)} \\
\datasetRowRule
\texttt{print\_stmt\_2\_d} & \codeblock{\texttt{print(\asciidq[USAGE\_TEXT]}} & \codeblock{\texttt{print(\asciisq[USAGE\_TEXT]}} & \texttt{\asciidq)} & \texttt{\asciisq)} \\
\datasetRowRule
\texttt{print\_stmt\_3\_s} & \codeblock{\texttt{print(\asciisq[QUESTION\_TEXT]}} & \codeblock{\texttt{print(\asciidq[QUESTION\_TEXT]}} & \texttt{\asciisq)} & \texttt{\asciidq)} \\
\datasetRowRule
\texttt{print\_stmt\_3\_d} & \codeblock{\texttt{print(\asciidq[QUESTION\_TEXT]}} & \codeblock{\texttt{print(\asciisq[QUESTION\_TEXT]}} & \texttt{\asciidq)} & \texttt{\asciisq)} \\
\datasetRowRule
\texttt{print\_stmt\_4\_s} & \codeblock{\codeindent\codeindent\codeindent\texttt{print(\asciisq[STATUS\_TEXT]}} & \codeblock{\codeindent\codeindent\codeindent\texttt{print(\asciidq[STATUS\_TEXT]}} & \texttt{\asciisq)} & \texttt{\asciidq)} \\
\datasetRowRule
\texttt{print\_stmt\_4\_d} & \codeblock{\codeindent\codeindent\codeindent\texttt{print(\asciidq[STATUS\_TEXT]}} & \codeblock{\codeindent\codeindent\codeindent\texttt{print(\asciisq[STATUS\_TEXT]}} & \texttt{\asciidq)} & \texttt{\asciisq)} \\
\datasetRowRule
\texttt{raise\_error\_s} & \codeblock{\texttt{raise ValueError(\asciisq[ERROR\_TEXT]}} & \codeblock{\texttt{raise ValueError(\asciidq[ERROR\_TEXT]}} & \texttt{\asciisq)} & \texttt{\asciidq)} \\
\datasetRowRule
\texttt{raise\_error\_d} & \codeblock{\texttt{raise ValueError(\asciidq[ERROR\_TEXT]}} & \codeblock{\texttt{raise ValueError(\asciisq[ERROR\_TEXT]}} & \texttt{\asciidq)} & \texttt{\asciisq)} \\
\datasetRowRule
\texttt{string\_split\_s} & \codeblock{\texttt{def [FUNC]([ARG]):}\\\codeindent\texttt{[LEFT], [RIGHT] =}\\\codeindent\texttt{[ARG].split(\asciisq-}} & \codeblock{\texttt{def [FUNC]([ARG]):}\\\codeindent\texttt{[LEFT], [RIGHT] =}\\\codeindent\texttt{[ARG].split(\asciidq-}} & \texttt{\asciisq)} & \texttt{\asciidq)} \\
\datasetRowRule
\texttt{string\_split\_d} & \codeblock{\texttt{def [FUNC]([ARG]):}\\\codeindent\texttt{[LEFT], [RIGHT] =}\\\codeindent\texttt{[ARG].split(\asciidq-}} & \codeblock{\texttt{def [FUNC]([ARG]):}\\\codeindent\texttt{[LEFT], [RIGHT] =}\\\codeindent\texttt{[ARG].split(\asciisq-}} & \texttt{\asciidq)} & \texttt{\asciisq)} \\
\datasetRowRule
\texttt{strftime\_call\_s} & \codeblock{\texttt{[TIME\_VAR] = datetime.datetime.now()}\\\codeindent\texttt{.strftime(\asciisq\%Y\%m\%d\_\%H\%M\%S}} & \codeblock{\texttt{[TIME\_VAR] = datetime.datetime.now()}\\\codeindent\texttt{.strftime(\asciidq\%Y\%m\%d\_\%H\%M\%S}} & \texttt{\asciisq)} & \texttt{\asciidq)} \\
\datasetRowRule
\texttt{strftime\_call\_d} & \codeblock{\texttt{[TIME\_VAR] = datetime.datetime.now()}\\\codeindent\texttt{.strftime(\asciidq\%Y\%m\%d\_\%H\%M\%S}} & \codeblock{\texttt{[TIME\_VAR] = datetime.datetime.now()}\\\codeindent\texttt{.strftime(\asciisq\%Y\%m\%d\_\%H\%M\%S}} & \texttt{\asciidq)} & \texttt{\asciisq)} \\
\datasetBottomRule
\end{tabular}
\end{table}
\FloatBarrier

\begin{table}[H]
\centering
\caption{else-elif dataset templates. The task contrasts Python branches where \texttt{else} should be followed by \texttt{:\textbackslash n}, while \texttt{elif} should be followed by any other token. In visualizations, template labels additionally carry \texttt{\_else} or \texttt{\_elif} suffixes for the two directions; these suffixes are omitted here for space, and the prompt examples show one representative direction.}
\label{tab:else-elif-templates}
\tiny
\setlength{\tabcolsep}{1.5pt}
\renewcommand{\arraystretch}{1.08}
\begin{tabular}{@{}!{\vrule width 0.9pt}p{0.145\linewidth}p{0.335\linewidth}p{0.335\linewidth}p{0.075\linewidth}p{0.075\linewidth}!{\vrule width 0.9pt}@{}}
\datasetTopRule
\datasetHeaderRow
\datasetRowRule
\texttt{return\_result} & \codeblock{\texttt{if [RESULT]:}\\\codeindent\texttt{[CACHE] = [RESULT]}\\\codeindent\texttt{return [RESULT]}\\\texttt{else}} & \codeblock{\texttt{if [RESULT]:}\\\codeindent\texttt{[CACHE] = [RESULT]}\\\codeindent\texttt{return [RESULT]}\\\texttt{elif}} & \texttt{:\textbackslash n} & other tokens \\
\datasetRowRule
\texttt{answer\_check} & \codeblock{\texttt{if eval([PROBLEM]) == [ANSWER]:}\\\codeindent\texttt{print([MSG])}\\\texttt{elif}} & \codeblock{\texttt{if eval([PROBLEM]) == [ANSWER]:}\\\codeindent\texttt{print([MSG])}\\\texttt{else}} & other tokens & \texttt{:\textbackslash n} \\
\datasetRowRule
\texttt{append\_guard} & \codeblock{\texttt{if [COUNT\_MAP][[NODE].topic] >= [TARGET]:}\\\codeindent\texttt{[LIST].append([NODE])}\\\texttt{else}} & \codeblock{\texttt{if [COUNT\_MAP][[NODE].topic] >= [TARGET]:}\\\codeindent\texttt{[LIST].append([NODE])}\\\texttt{elif}} & \texttt{:\textbackslash n} & other tokens \\
\datasetRowRule
\texttt{program\_dir\_chain} & \codeblock{\texttt{if [DIR] == 0:}\\\codeindent\texttt{[DIR] = [NEG]}\\\texttt{elif [DIR] == 1:}\\\codeindent\texttt{[DIR] = [POS]}\\\texttt{elif}} & \codeblock{\texttt{if [DIR] == 0:}\\\codeindent\texttt{[DIR] = [NEG]}\\\texttt{elif [DIR] == 1:}\\\codeindent\texttt{[DIR] = [POS]}\\\texttt{else}} & other tokens & \texttt{:\textbackslash n} \\
\datasetRowRule
\texttt{schedule\_branch} & \codeblock{\texttt{if [TIME] > [LAST]:}\\\codeindent\texttt{[START] = [LAST] + 1}\\\texttt{else}} & \codeblock{\texttt{if [TIME] > [LAST]:}\\\codeindent\texttt{[START] = [LAST] + 1}\\\texttt{elif}} & \texttt{:\textbackslash n} & other tokens \\
\datasetRowRule
\texttt{strategy\_branch} & \codeblock{\texttt{if [STRATEGY] == 'bfs':}\\\codeindent\texttt{return bfs(...)}\\\texttt{elif [STRATEGY] == 'dfs':}\\\codeindent\texttt{return dfs(...)}\\\texttt{elif}} & \codeblock{\texttt{if [STRATEGY] == 'bfs':}\\\codeindent\texttt{return bfs(...)}\\\texttt{elif [STRATEGY] == 'dfs':}\\\codeindent\texttt{return dfs(...)}\\\texttt{else}} & other tokens & \texttt{:\textbackslash n} \\
\datasetRowRule
\texttt{equality\_case} & \codeblock{\texttt{if [CONSTRAINT].is\_equality():}\\\codeindent\texttt{return [Z3] == -[COEFF]['1']}\\\texttt{else}} & \codeblock{\texttt{if [CONSTRAINT].is\_equality():}\\\codeindent\texttt{return [Z3] == -[COEFF]['1']}\\\texttt{elif}} & \texttt{:\textbackslash n} & other tokens \\
\datasetRowRule
\texttt{stats\_update} & \codeblock{\texttt{if [ROOT] not in [STATS]:}\\\codeindent\texttt{[STATS][[ROOT]] =}\\\codeindent\texttt{([AREA], [COUNT])}\\\texttt{elif}} & \codeblock{\texttt{if [ROOT] not in [STATS]:}\\\codeindent\texttt{[STATS][[ROOT]] =}\\\codeindent\texttt{([AREA], [COUNT])}\\\texttt{else}} & other tokens & \texttt{:\textbackslash n} \\
\datasetRowRule
\texttt{inline\_threshold} & \codeblock{\texttt{if ([NUM] <= [LIMIT]):}\\\codeindent\texttt{return [FUNC\_CALL]}\\\texttt{else}} & \codeblock{\texttt{if ([NUM] <= [LIMIT]):}\\\codeindent\texttt{return [FUNC\_CALL]}\\\texttt{elif}} & \texttt{:\textbackslash n} & other tokens \\
\datasetRowRule
\texttt{string\_align} & \codeblock{\texttt{if [CHAR] in [STRING]:}\\\codeindent\texttt{[LEFT], [RIGHT] =}\\\codeindent\texttt{[STRING].split([CHAR], 1)}\\\texttt{elif}} & \codeblock{\texttt{if [CHAR] in [STRING]:}\\\codeindent\texttt{[LEFT], [RIGHT] =}\\\codeindent\texttt{[STRING].split([CHAR], 1)}\\\texttt{else}} & other tokens & \texttt{:\textbackslash n} \\
\datasetRowRule
\texttt{binary\_search} & \codeblock{\texttt{if [DATA][[INDEX]] < [VALUE]:}\\\codeindent\texttt{[RESULT] = [INDEX]}\\\codeindent\texttt{[LOW] = [INDEX] + 1}\\\texttt{else}} & \codeblock{\texttt{if [DATA][[INDEX]] < [VALUE]:}\\\codeindent\texttt{[RESULT] = [INDEX]}\\\codeindent\texttt{[LOW] = [INDEX] + 1}\\\texttt{elif}} & \texttt{:\textbackslash n} & other tokens \\
\datasetRowRule
\texttt{truthy\_call} & \codeblock{\texttt{if [ROOM]:}\\\codeindent\texttt{set\_parameter([BOARD], [PARAM],}\\\codeindent\texttt{[ROOM\_NAME] + ' - ' + [ROOM])}\\\texttt{elif}} & \codeblock{\texttt{if [ROOM]:}\\\codeindent\texttt{set\_parameter([BOARD], [PARAM],}\\\codeindent\texttt{[ROOM\_NAME] + ' - ' + [ROOM])}\\\texttt{else}} & other tokens & \texttt{:\textbackslash n} \\
\datasetRowRule
\texttt{compare\_objects} & \codeblock{\texttt{if isinstance([LEFT], (list, tuple)):}\\\codeindent\texttt{if isinstance([RIGHT], (list, tuple)):}\\\codeindent\codeindent\texttt{compare\_objects(...)}\\\codeindent\texttt{else}} & \codeblock{\texttt{if isinstance([LEFT], (list, tuple)):}\\\codeindent\texttt{if isinstance([RIGHT], (list, tuple)):}\\\codeindent\codeindent\texttt{compare\_objects(...)}\\\codeindent\texttt{elif}} & \texttt{:\textbackslash n} & other tokens \\
\datasetRowRule
\texttt{cell\_toggle} & \codeblock{\texttt{if [BOARD][[ROW]][[COL]] == [OFF]:}\\\codeindent\texttt{[BOARD][[ROW]][[COL]] = 1}\\\texttt{elif}} & \codeblock{\texttt{if [BOARD][[ROW]][[COL]] == [OFF]:}\\\codeindent\texttt{[BOARD][[ROW]][[COL]] = 1}\\\texttt{else}} & other tokens & \texttt{:\textbackslash n} \\
\datasetRowRule
\texttt{random\_split} & \codeblock{\texttt{if random.random() < 0.5:}\\\codeindent\texttt{[OUT].append(}\\\codeindent\texttt{f"ML,\{[PEER]\},\{[ELEMENT]\}")}\\\texttt{else}} & \codeblock{\texttt{if random.random() < 0.5:}\\\codeindent\texttt{[OUT].append(}\\\codeindent\texttt{f"ML,\{[PEER]\},\{[ELEMENT]\}")}\\\texttt{elif}} & \texttt{:\textbackslash n} & other tokens \\
\datasetRowRule
\texttt{none\_default} & \codeblock{\texttt{if [END] is None:}\\\codeindent\texttt{[GOAL] = ([ROWS]-1, [COLS]-1)}\\\texttt{elif}} & \codeblock{\texttt{if [END] is None:}\\\codeindent\texttt{[GOAL] = ([ROWS]-1, [COLS]-1)}\\\texttt{else}} & other tokens & \texttt{:\textbackslash n} \\
\datasetBottomRule
\end{tabular}
\end{table}
\FloatBarrier

\section[Proof of CEAP Scoring Theorem]{Proof of \cref{thm:Cond_VS_IG}}
\label{app:proof_of_Cond_VS_IG}
\begin{proof}
We first consider the CEAP scoring function.
Given any branch $b\in B$, suppose the edge has size $y \in \mathbb{R}^{d_b}$, we have:
$$
\begin{aligned}
I_{\text{Cond}}(y_b)
&= \sum_{i\in [d_b]}\sum_{j\in [d_{in}]} ({x_1}_j - {x_0}_j)
   \int_{0}^{1}
   \frac{\partial F(x_0 + \alpha(x_1 - x_0))}{\partial {y_b}_i}
   \frac{\partial {y_b}_i}{\partial x_j}
   \,\mathrm{d}\alpha\\
&= \sum_{i\in [d_b]}\int_{0}^{1}
   \frac{\partial F(x_0 + \alpha(x_1 - x_0))}{\partial {y_b}_i}\underbrace{\sum_{j\in [d_{in}]}\frac{\partial {y_b}_i}{\partial x_j}
   \,\mathrm{d}\alpha({x_1}_j - {x_0}_j)}_{(=\mathrm{d}{y_b}_i)}\\
   &= \sum_{i\in [d_b]}\int_{0}^{1}
   \frac{\partial f_b({y_b}_i(\alpha))}{\partial {y_b}_i}{\mathrm{d}{y_b}_i(\alpha)} = \int_{\gamma_{y_b}}\sum_{i\in [d_b]}
   \frac{\partial f_b({y_b}_i)}{\partial {y_b}_i}{\mathrm{d}{y_b}_i}\\
   &= \int_{\gamma_{y_b}}\left(\nabla_{y_b}f_b\right)\cdot \boldsymbol{\mathrm{d}r}= f_b(x_1) - f_b(x_0) = \Delta f_b.
\end{aligned}
$$
In the above, $\gamma_{y_b}$ is the trajectory that $y_b$ follows when $\alpha$ goes from $0$ to $1$, and the second last equality is due to that the integrand is a gradient, and thus the result is path-independent.
The above derivation shows that conductance satisfies additive order preservation, almost by design.

On the other hand, we can come up with the following counterexample for IG.
Suppose we are studying a two-branched function:
$F = g_1(y_1) + g_2(y_2) = g_1(f_1(x))+ g_2(f_2(x))$, where
\[
f_1(x)=
\begin{cases}
\dfrac{1}{3}x, & x<1,\\[6pt]
x-\dfrac{2}{3}, & 1\le x<2,\\[6pt]
\dfrac{1}{3}x+\dfrac{2}{3}, & x\ge 2,
\end{cases}
\qquad
g_1(y_1)=
\begin{cases}
y_1, & y_1<\dfrac{1}{3},\\[6pt]
\dfrac{1}{3}, & \dfrac{1}{3}\le y_1<\dfrac{4}{3},\\[6pt]
y_1-1, & y_1\ge \dfrac{4}{3},
\end{cases}
\qquad
\begin{aligned}
f_2(x) &= x,\\
g_2(y_2) &= \dfrac{1}{3}y_2.
\end{aligned}
\]
We study the process where the input $x$ is moved from $0$ to $3$.
In this case, there are two branches, denoted by set $B=\{1,2\}$.
We can compute that $\vert \Delta f_1\vert = \frac{2}{3}$, and  $\vert \Delta f_2\vert = 1$.
We can also compute:
\begin{align}
I_{\mathrm{IG}}(y_1) &= \frac{5}{3}\int_{0}^1 \left.\frac{\partial F}{\partial y_1}\right|_{x = 3\alpha}\mathrm{d}\alpha
 = \frac{5}{3}\int_{0}^3 \left.\frac{\partial F}{\partial y_1}\right|_{x}\mathrm{d}x \label{eq:IG_y1_counterexample}
\\
 &=\frac{5}{3}\left(\int_0^1 \mathrm{d}x + \int_1^2 0 \mathrm{d}x+\int_2^3 \mathrm{d}x\right) = \frac{10}{3}.
 \label{eq:three_integral}
\end{align}

\begin{align*}
   I_{\mathrm{IG}}(y_2) &= 3\int_{0}^1 \left.\frac{\partial F}{\partial y_2}\right|_{x = 3\alpha}\mathrm{d}\alpha= 3\int_{0}^3 \left.\frac{\partial F}{\partial y_2}\right|_{x}\mathrm{d}x = 3\int_{0}^3 \frac{1}{3}\mathrm{d}x = 3.
\end{align*}
Hence $\vert\Delta f_1\vert< \vert\Delta f_2\vert$ and yet $\vert I_{\mathrm{IG}}(y_1)\vert> \vert I_{\mathrm{IG}}(y_2)\vert$, breaking additive order preservation.
\end{proof}

\begin{remark}
    One might notice that $g_1(\cdot)$ above is not differentiable, and thus we are implicitly using the one-sided derivative in place of $\frac{\partial F}{\partial y_1}$.
    However, our statement should still hold for such non-smooth functions.
    It is easy to see that non-differentiability is not the root cause of IG breaking additive order preservation.
    First, in the counterexample, the non-differentiability occurs on a Lebesgue-null set and hence does not affect the integral.
    Second, we can smoothen the kinks in $g_1(\cdot)$ using techniques like huberized ReLU \citep{huberize_relu} or softplus \citep{softplus}, where the kinks are replaced with smooth arcs whose span can be arbitrarily small.
    Thus, $I_{\mathrm
    {IG}}(y_1)$ for the smoothened $g_1$ can be arbitrarily close to what we compute above.
    Constructing such $g_1$ might involve tedious notation and blur the main message.
    We settle for the current counterexample for its simplicity and clarity:
    The reason why the magnitude of $I_{\mathrm
    {IG}}(y_1)$ relative to $I_{\mathrm
    {IG}}(y_2)$ seems to lose its meaning is that the integrals receives the same weight $\frac{5}{3}$ in \cref{eq:three_integral}, while, in fact, the weighting should reflect how much $y_1$ moves when $x$ moves through each of the three integration intervals.
\end{remark}

\begin{remark}
The fact that conductance satisfies additive order preservation may be understood as a result of linearity and partition consistency \citep{dhamdhere2018how} of conductance.
Nevertheless, these two properties are not necessary to satisfy additive order preservation.
In fact, if we engineer a new scoring method, which simply assigns to each edge the squared conductance score, additive order preservation is still satisfied, while linearity and partition consistency will be breached.
\end{remark}

\section{Complete Results for Pairwise Jaccard Index and Unfaithfulness}
\label{app: complete ceap vs eap-ig jaccard and unfaithfulness}

In this section, we compare unfaithfulness and pairwise Jaccard index yielded by CEAP and EAP-IG.
Overall, we found that while CEAP does not have a consistent advantage over EAP-IG in terms of unfaithfulness, it clearly outperforms the latter for pairwise Jaccard index, or at the very least matches it.
Specifically, for unfaithfulness, the quicker it drops with the increase of the number of edges, the better.
This is based on the principle of minimality \citep{wang2023interpretability}.
For the pairwise Jaccard index, the higher, the better.

A circuit-finding scheme could in principle game the PJI metric by always selecting the same circuit with no regard to the task.
We check that this is not the case for CEAP: it achieves low unfaithfulness across the tasks shown below.
Nonetheless, we do not find compelling evidence that CEAP consistently attains lower unfaithfulness than EAP-IG.
For instance, on GPT-2 XL, CEAP achieves lower unfaithfulness with smaller circuits for nearly all SVA templates, but the outcomes for IOI and greater-than are inconclusive.
This is understandable: both CEAP and EAP-IG are path-integral methods that characterize trajectories in the activation space as inputs move from corrupted samples to their clean counterparts.
However, once we patch the activations of the excluded edges, the resulting trajectory no longer aligns with the paths used to assign their scores.
Compounded with the nonlinearity of the network, this introduces errors that are difficult to control or even quantify and can drown out the advantage of CEAP over EAP-IG.

\subsection{GPT-2 XL}

\subsubsection{SVA}
\begin{figure}[H]
    \centering
    \includegraphics[width=\textwidth]{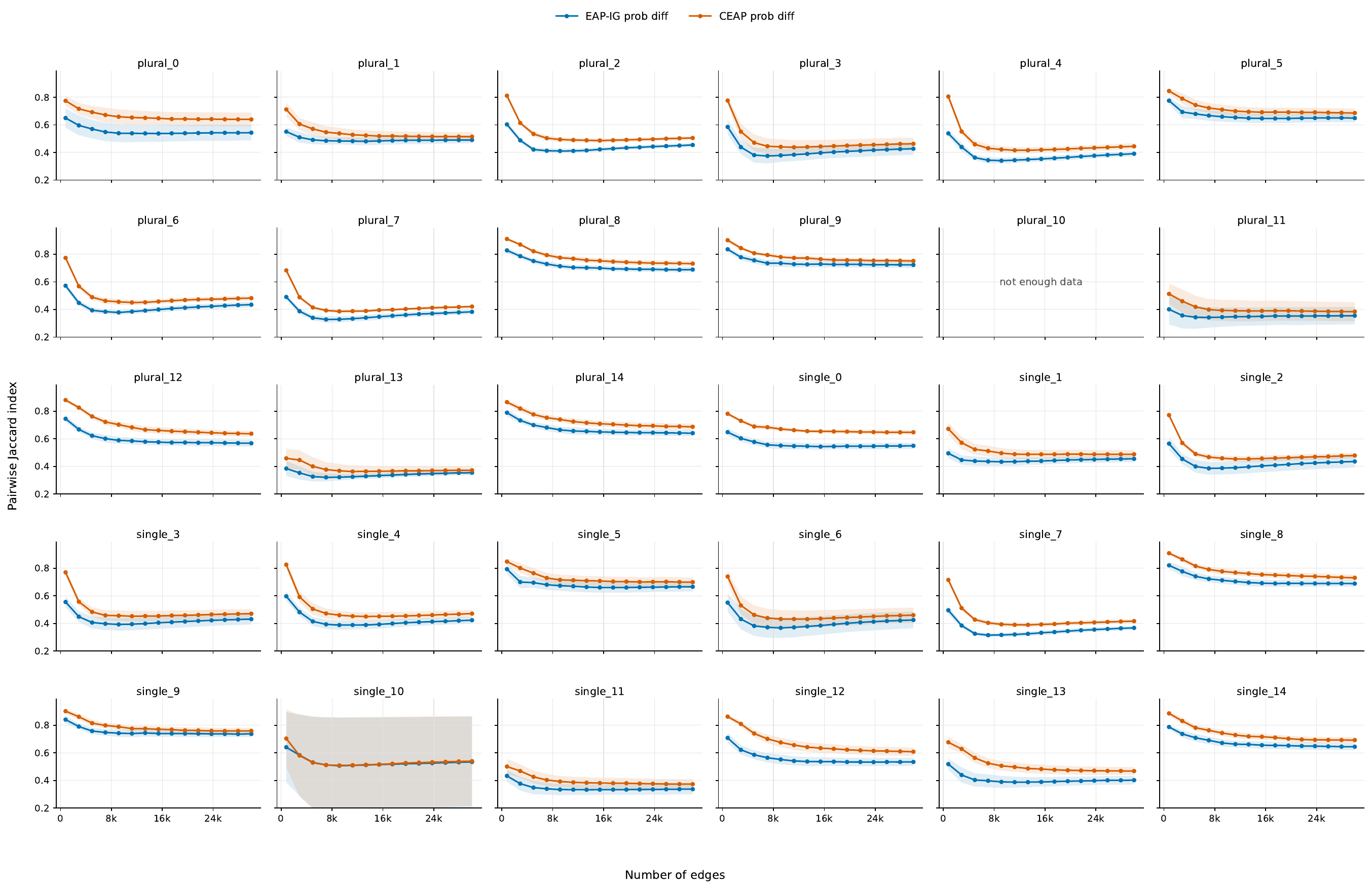}
    \caption{Pairwise Jaccard index vs.\ number of edges for GPT-2 XL on SVA.}
\end{figure}

\begin{figure}[H]
    \centering
    \includegraphics[width=\textwidth]{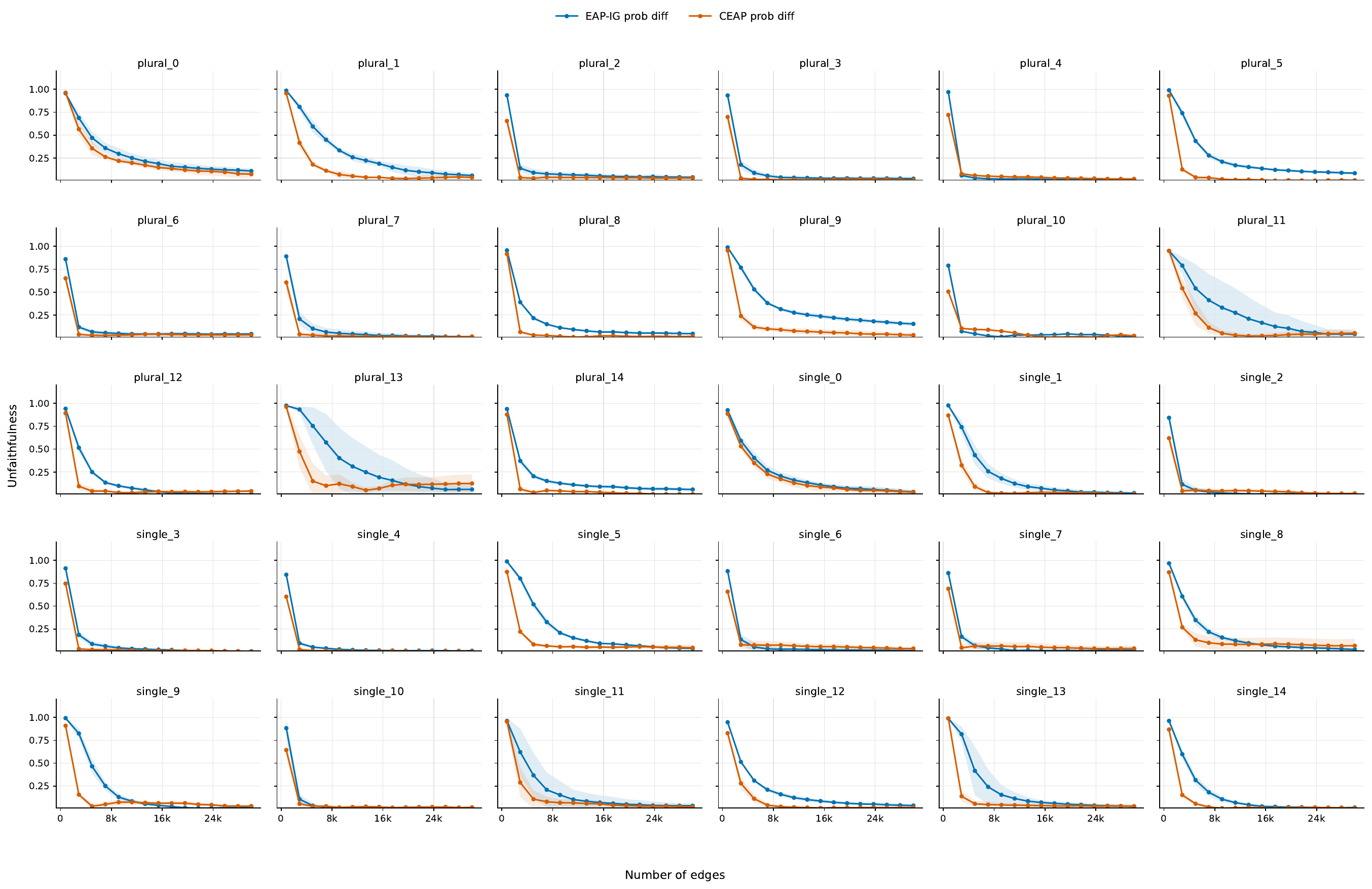}
    \caption{Unfaithfulness vs.\ number of edges for GPT-2 XL on SVA.}
\end{figure}

\subsubsection{IOI}
\begin{figure}[H]
    \centering
    \includegraphics[width=\textwidth]{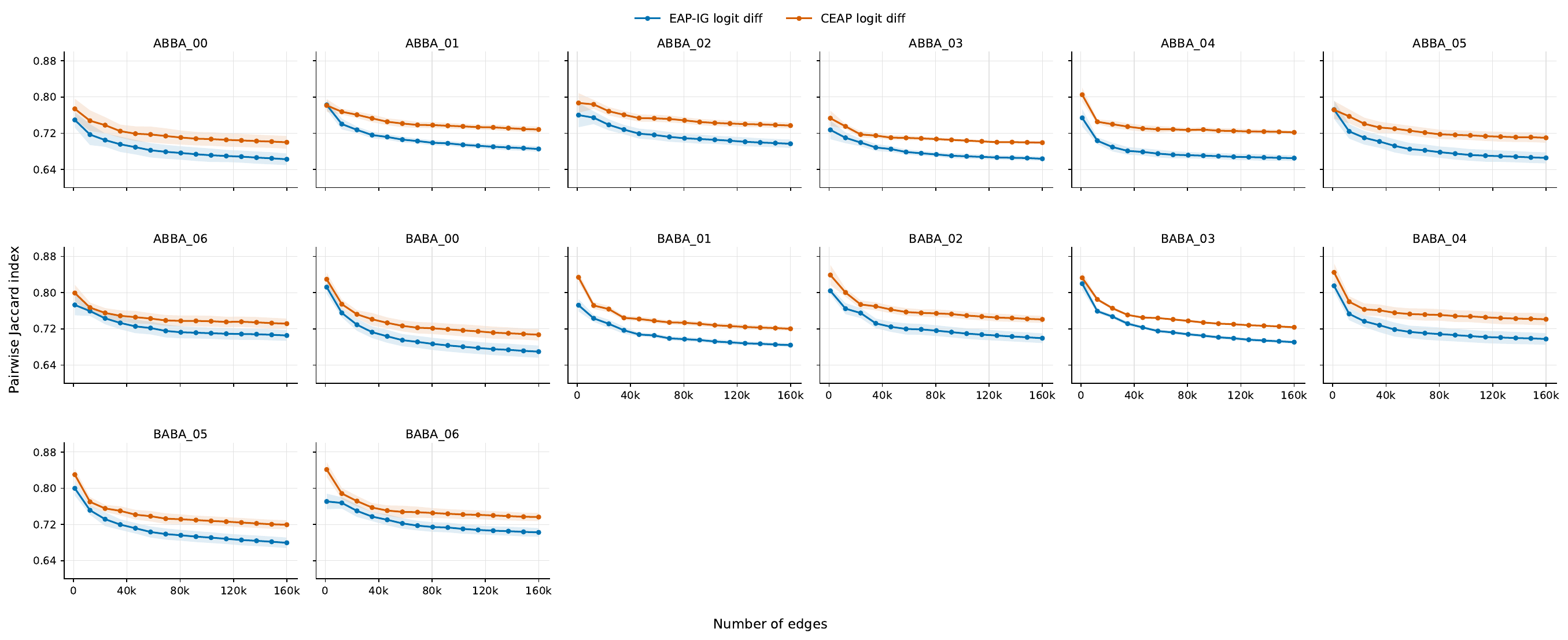}
    \caption{Pairwise Jaccard index vs.\ number of edges for GPT-2 XL on IOI.}
\end{figure}

\begin{figure}[H]
    \centering
    \includegraphics[width=\textwidth]{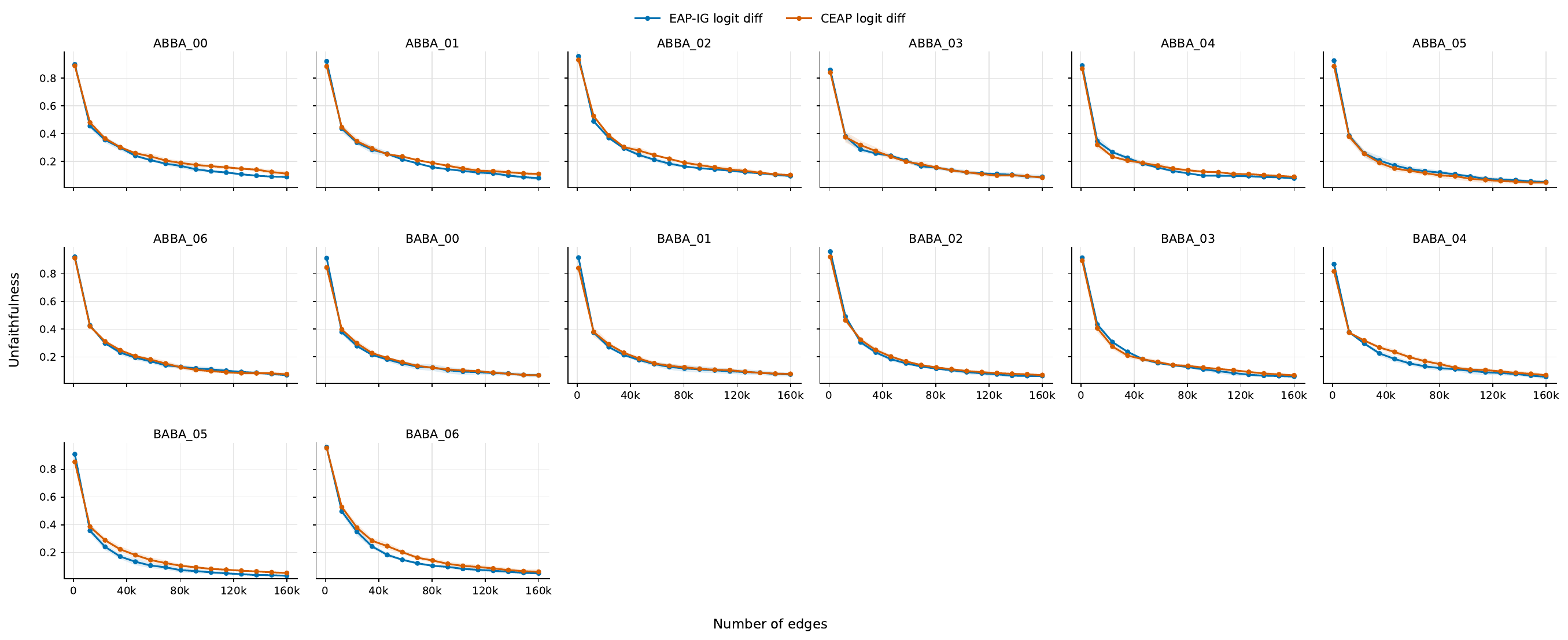}
    \caption{Unfaithfulness vs.\ number of edges for GPT-2 XL on IOI.}
\end{figure}

\subsubsection{Greater-than}
\begin{figure}[H]
    \centering
    \includegraphics[width=\textwidth]{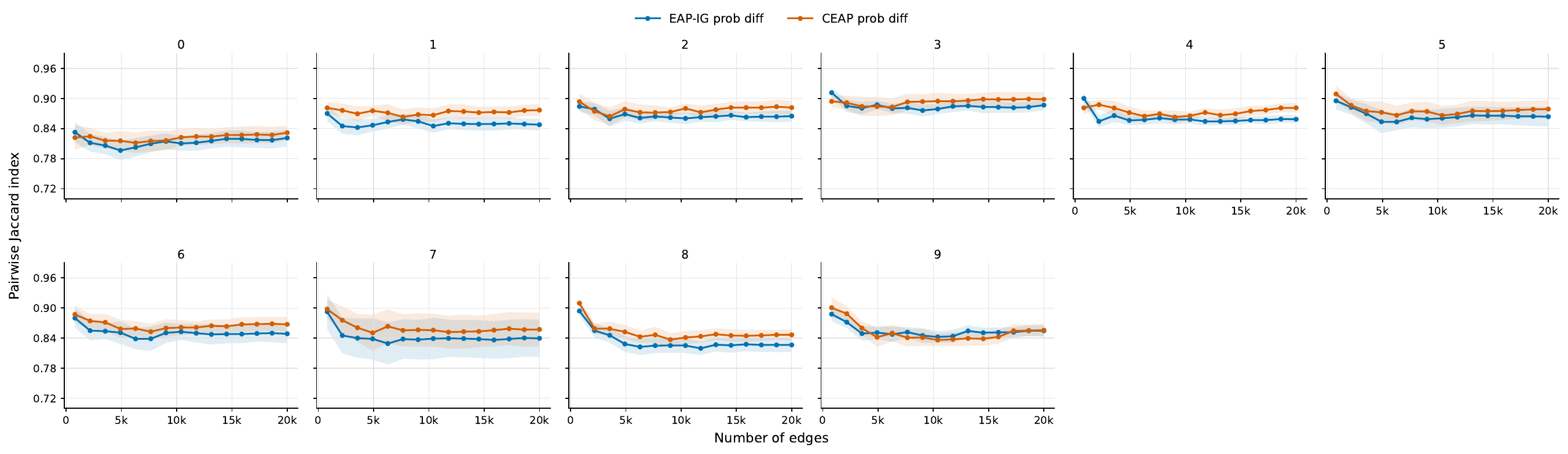}
    \caption{Pairwise Jaccard index vs.\ number of edges for GPT-2 XL on Greater-than.}
\end{figure}

\begin{figure}[H]
    \centering
    \includegraphics[width=\textwidth]{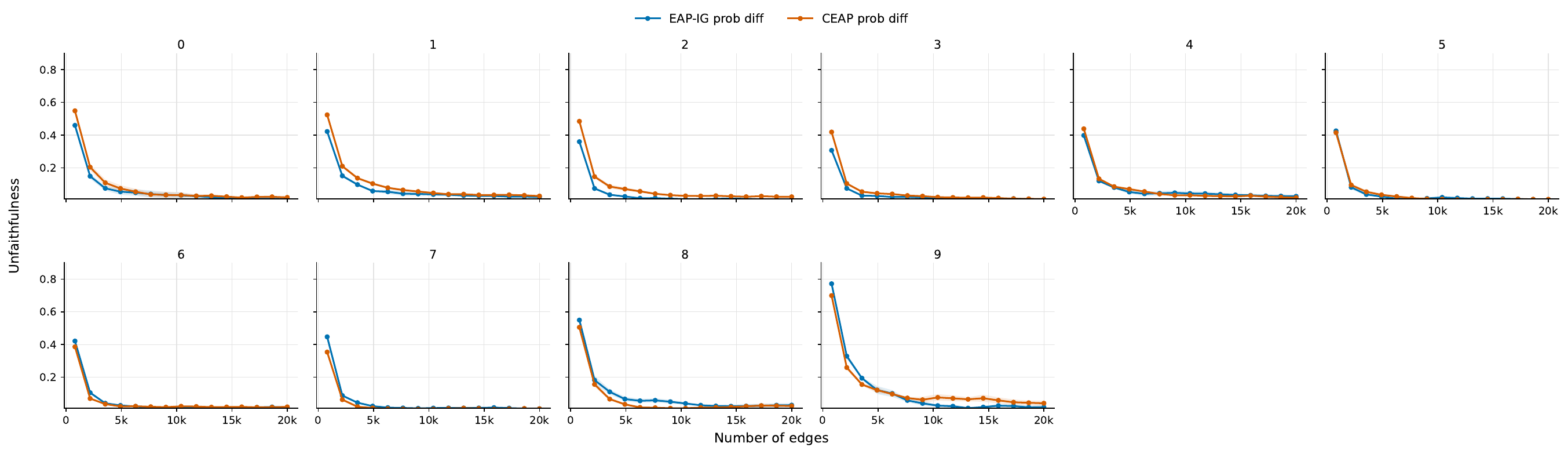}
    \caption{Unfaithfulness vs.\ number of edges for GPT-2 XL on Greater-than.}
\end{figure}

\subsection{Pythia-2.8B}

\subsubsection{SVA}
\begin{figure}[H]
    \centering
    \includegraphics[width=\textwidth]{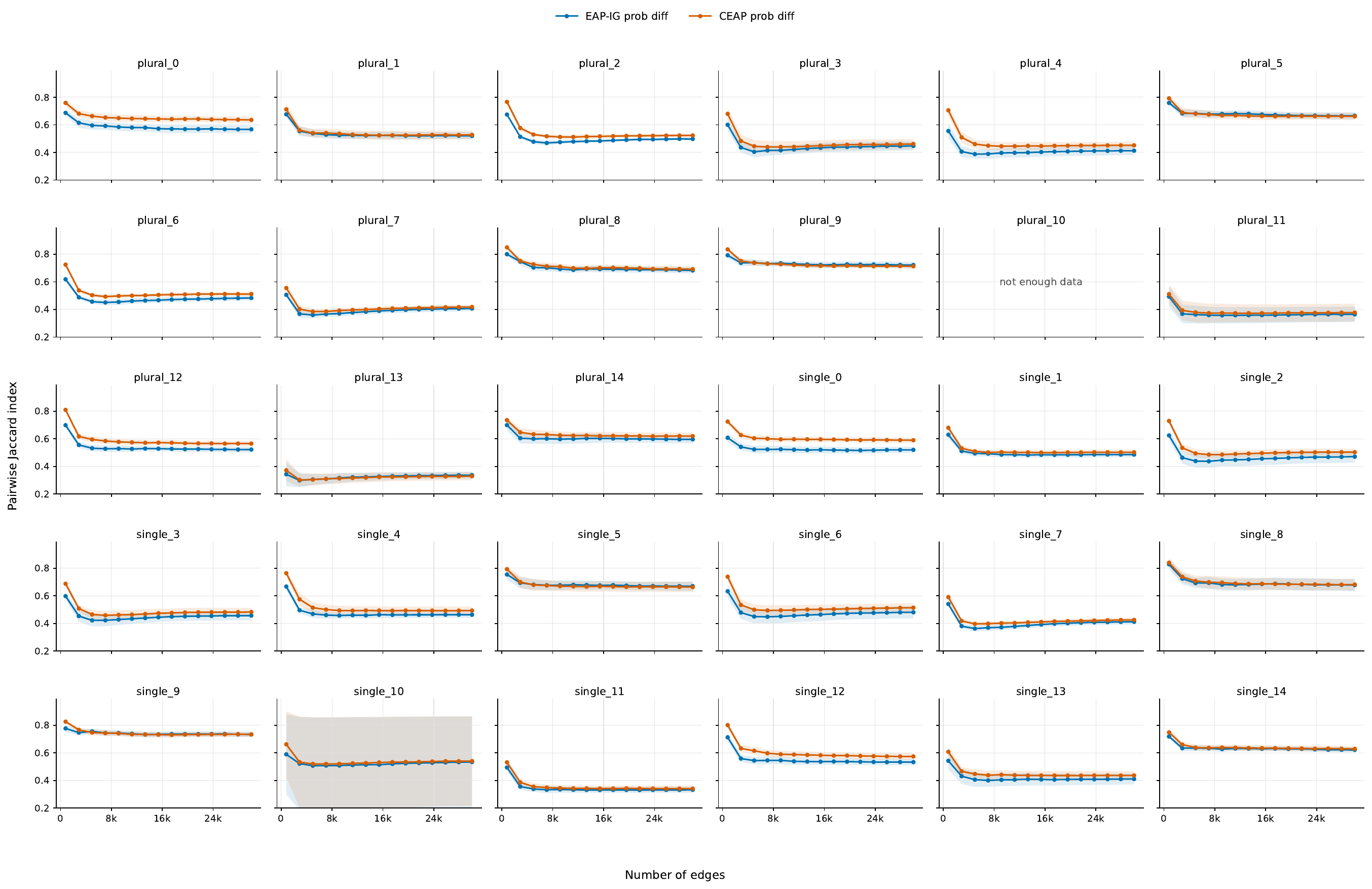}
    \caption{Pairwise Jaccard index vs.\ number of edges for Pythia-2.8B on SVA.}
\end{figure}

\begin{figure}[H]
    \centering
    \includegraphics[width=\textwidth]{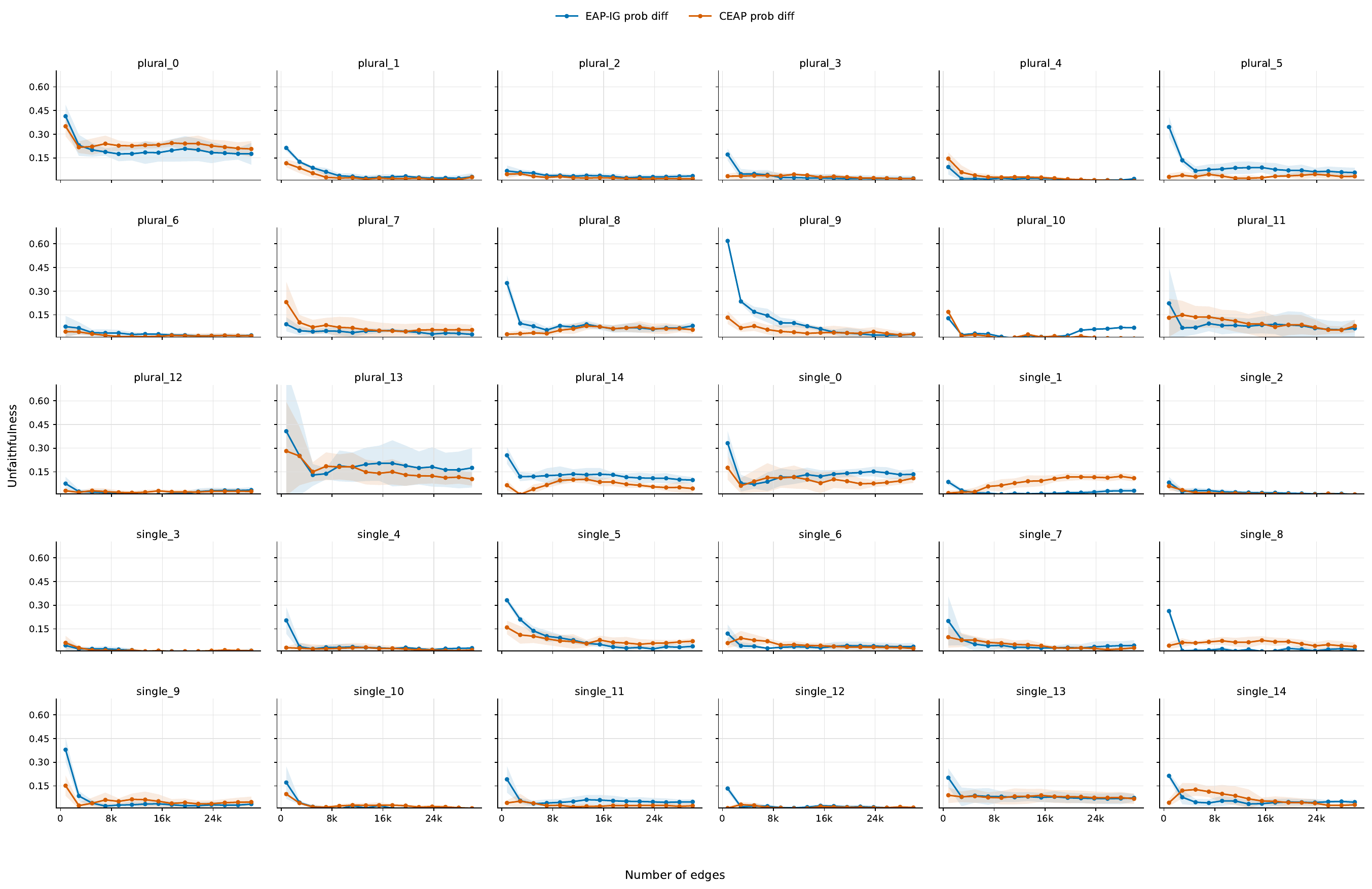}
    \caption{Unfaithfulness vs.\ number of edges for Pythia-2.8B on SVA.}
\end{figure}

\subsubsection{IOI}
\begin{figure}[H]
    \centering
    \includegraphics[width=\textwidth]{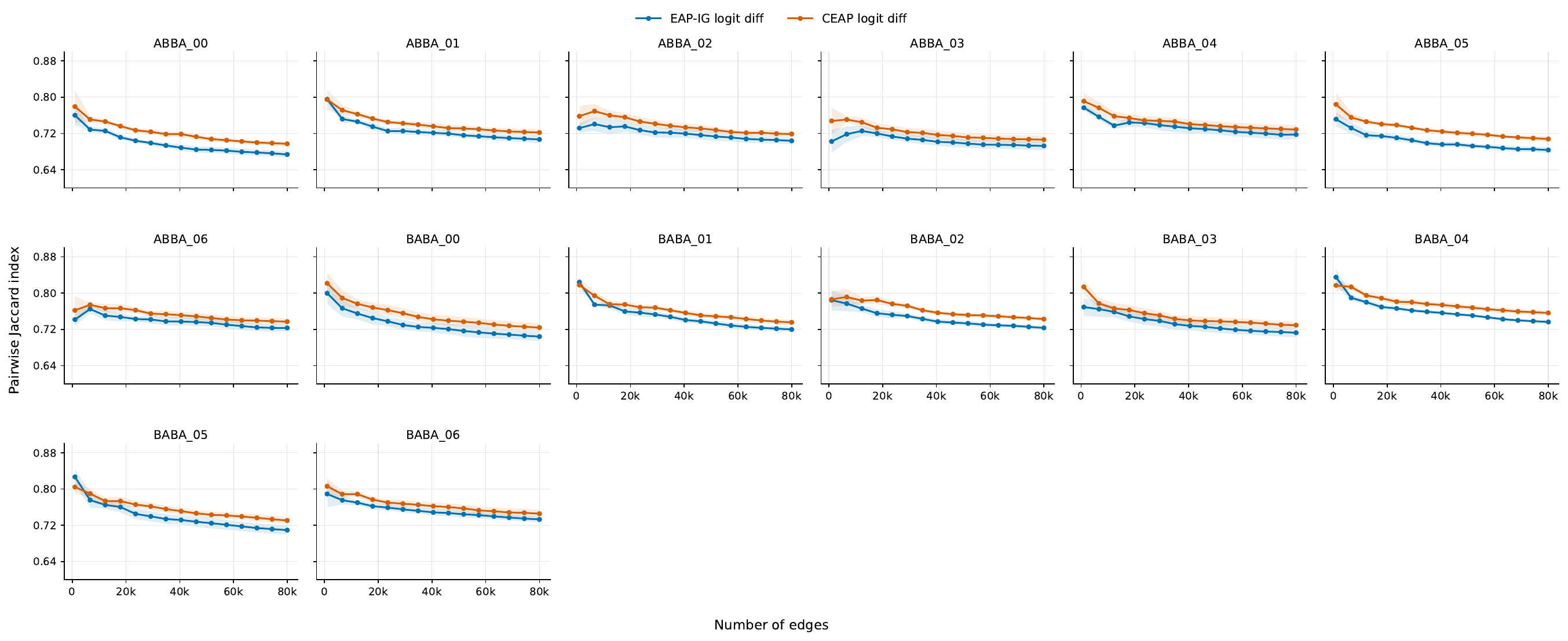}
    \caption{Pairwise Jaccard index vs.\ number of edges for Pythia-2.8B on IOI.}
\end{figure}

\begin{figure}[H]
    \centering
    \includegraphics[width=\textwidth]{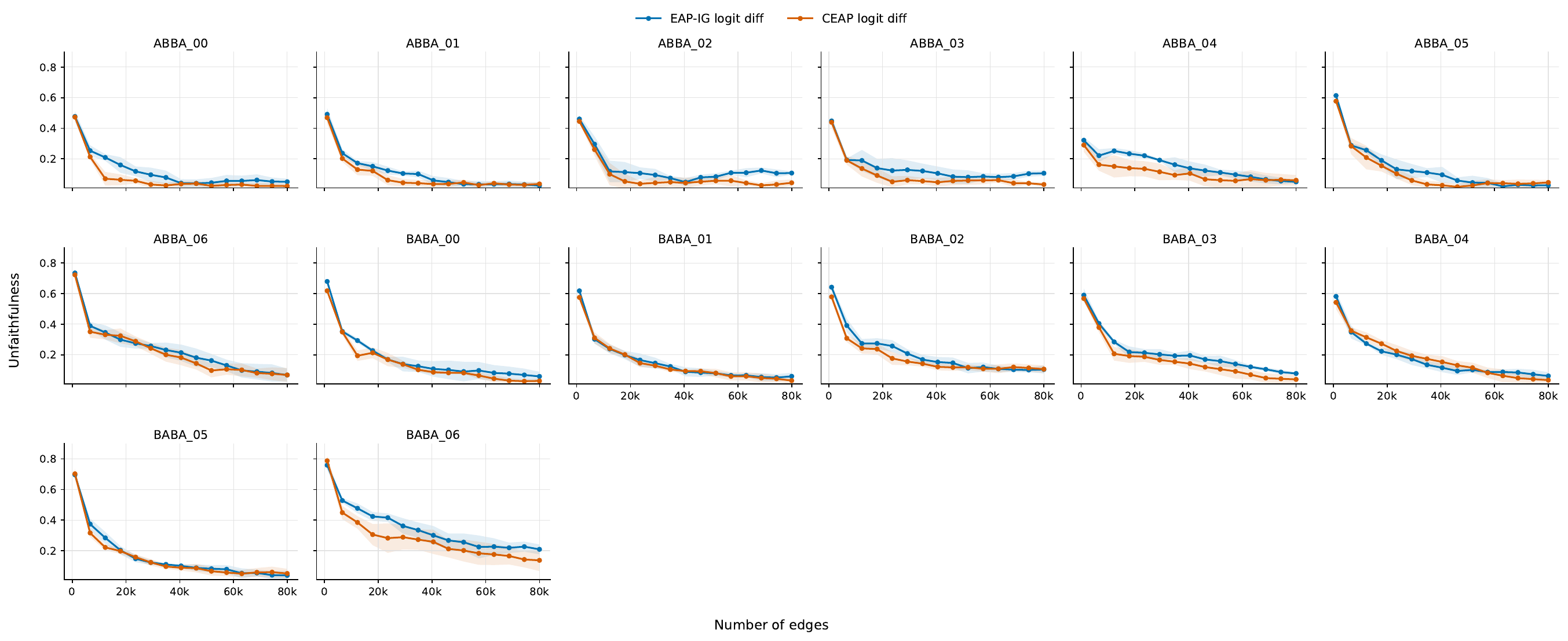}
    \caption{Unfaithfulness vs.\ number of edges for Pythia-2.8B on IOI.}
\end{figure}

\subsubsection{Greater-than}
\begin{figure}[H]
    \centering
    \includegraphics[width=\textwidth]{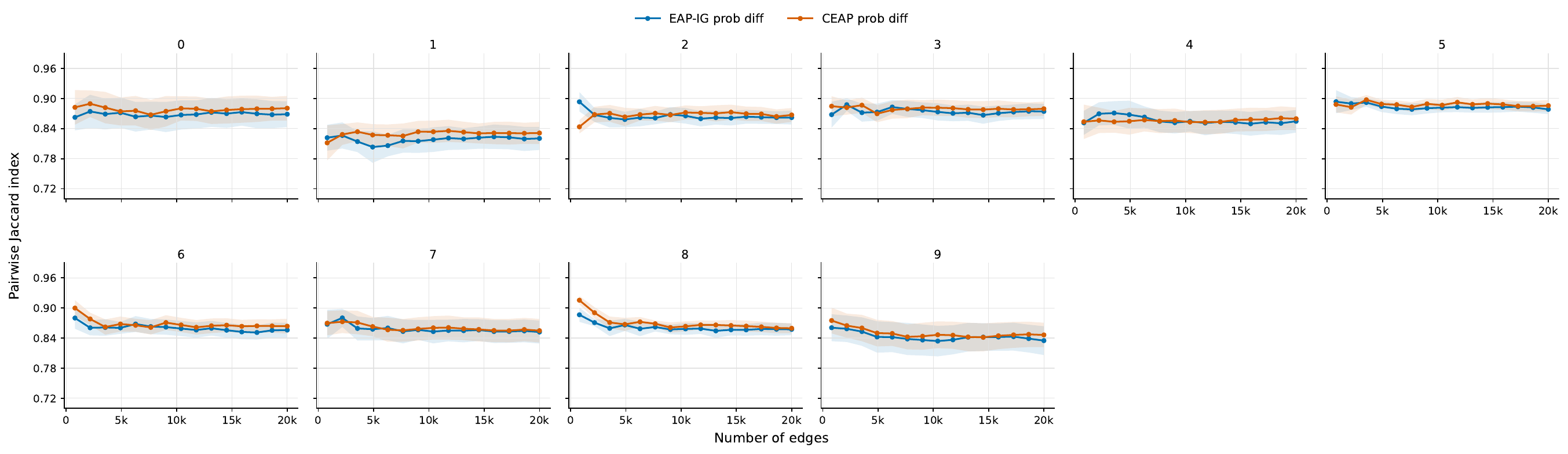}
    \caption{Pairwise Jaccard index vs.\ number of edges for Pythia-2.8B on Greater-than.}
\end{figure}

\begin{figure}[H]
    \centering
    \includegraphics[width=\textwidth]{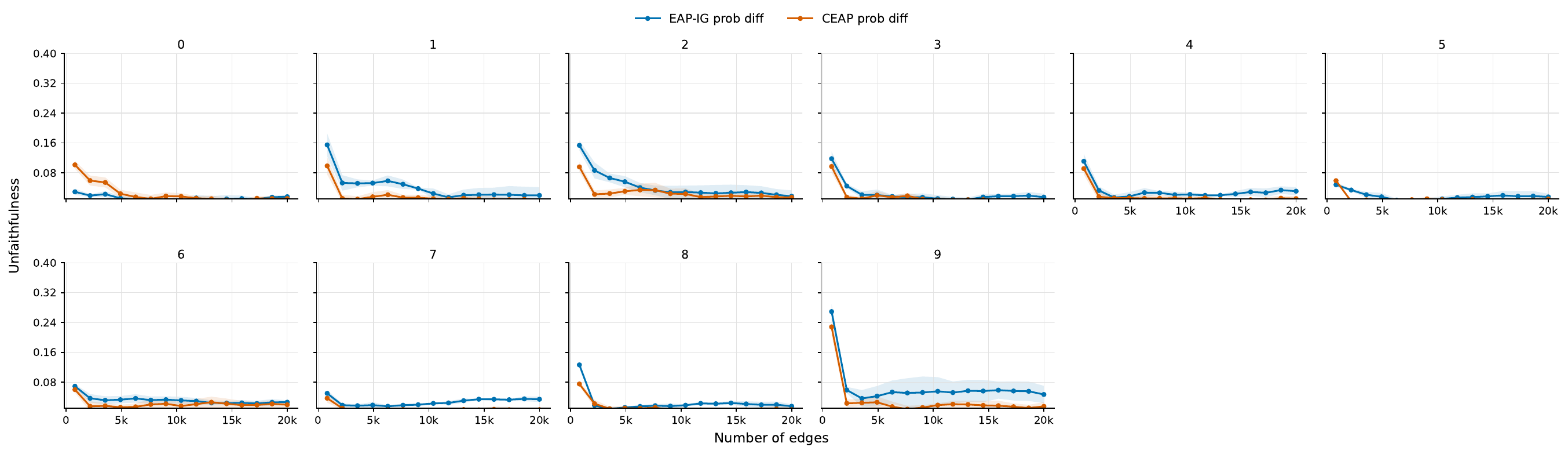}
    \caption{Unfaithfulness vs.\ number of edges for Pythia-2.8B on Greater-than.}
\end{figure}

\subsection{GPT-2 small}

\subsubsection{SVA}
\begin{figure}[H]
    \centering
    \includegraphics[width=\textwidth]{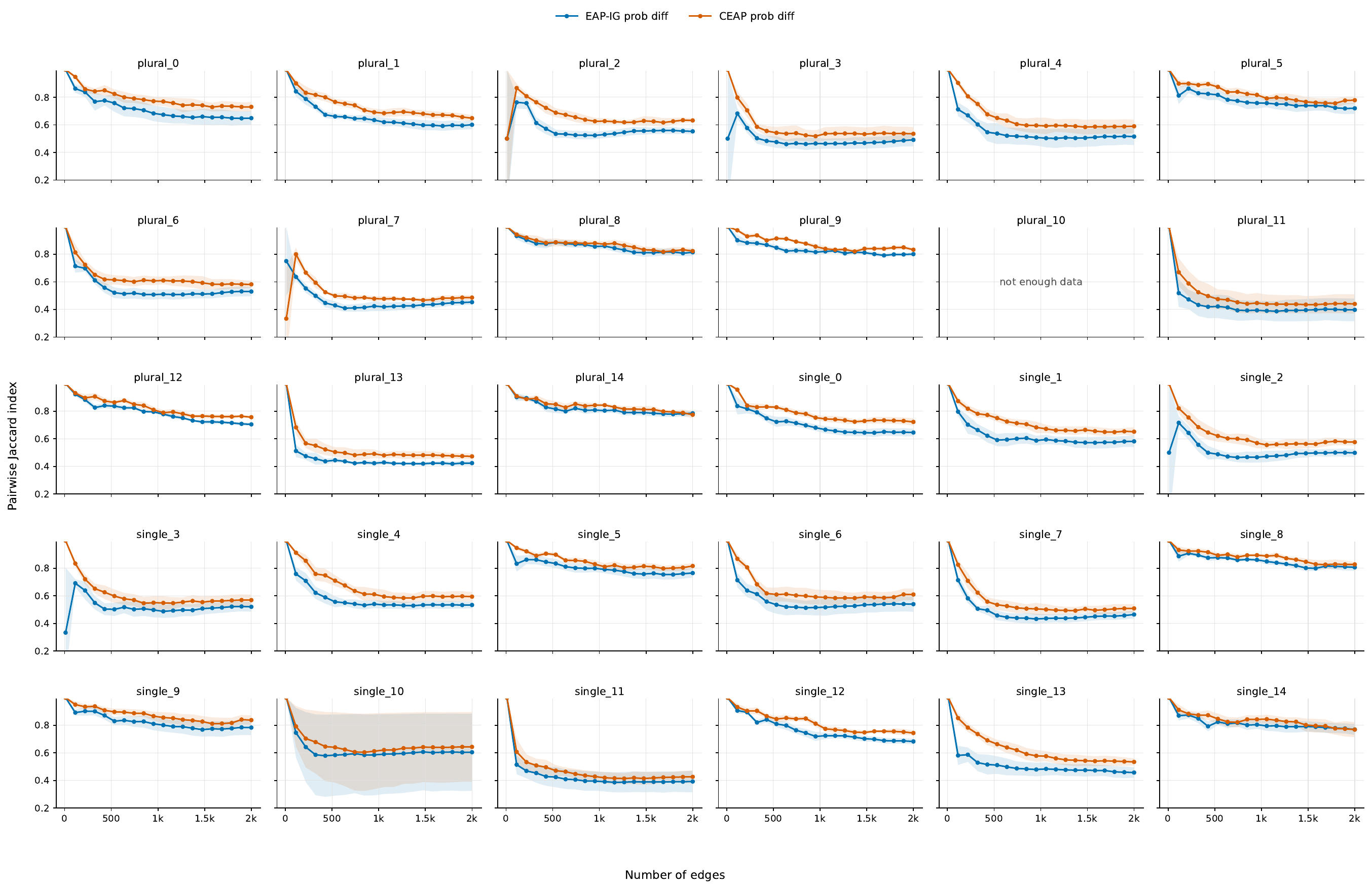}
    \caption{Pairwise Jaccard index vs.\ number of edges for GPT-2 small on SVA.}
\end{figure}

\begin{figure}[H]
    \centering
    \includegraphics[width=\textwidth]{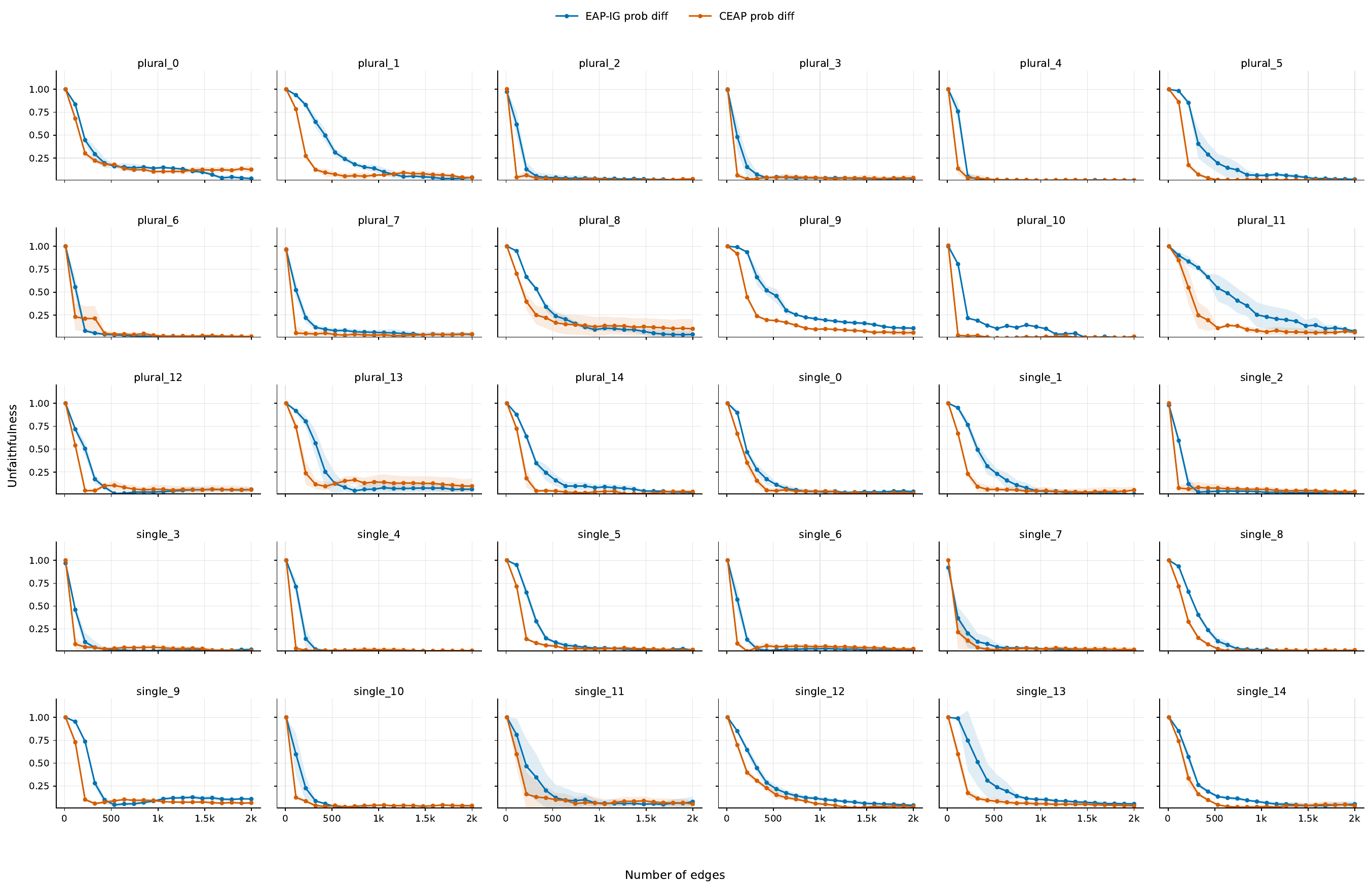}
    \caption{Unfaithfulness vs.\ number of edges for GPT-2 small on SVA.}
\end{figure}

\subsubsection{IOI}
\begin{figure}[H]
    \centering
    \includegraphics[width=\textwidth]{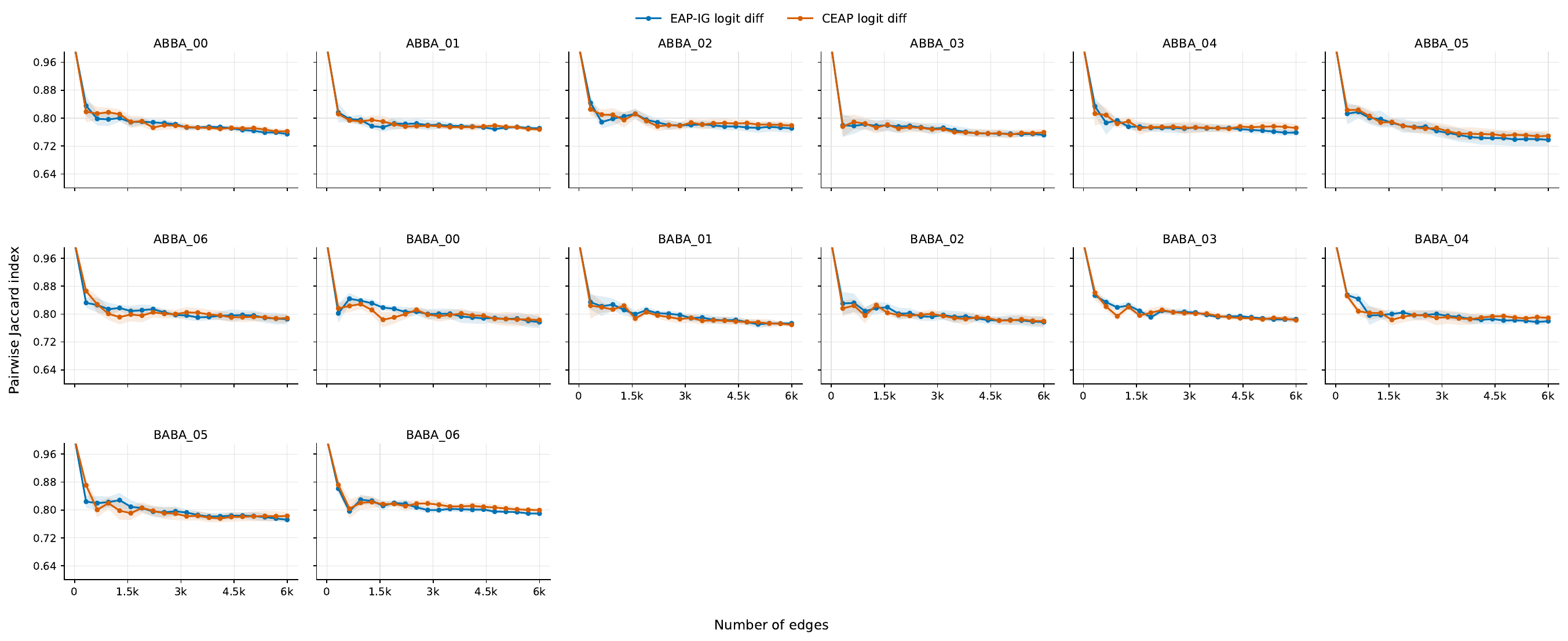}
    \caption{Pairwise Jaccard index vs.\ number of edges for GPT-2 small on IOI.}
\end{figure}

\begin{figure}[H]
    \centering
    \includegraphics[width=\textwidth]{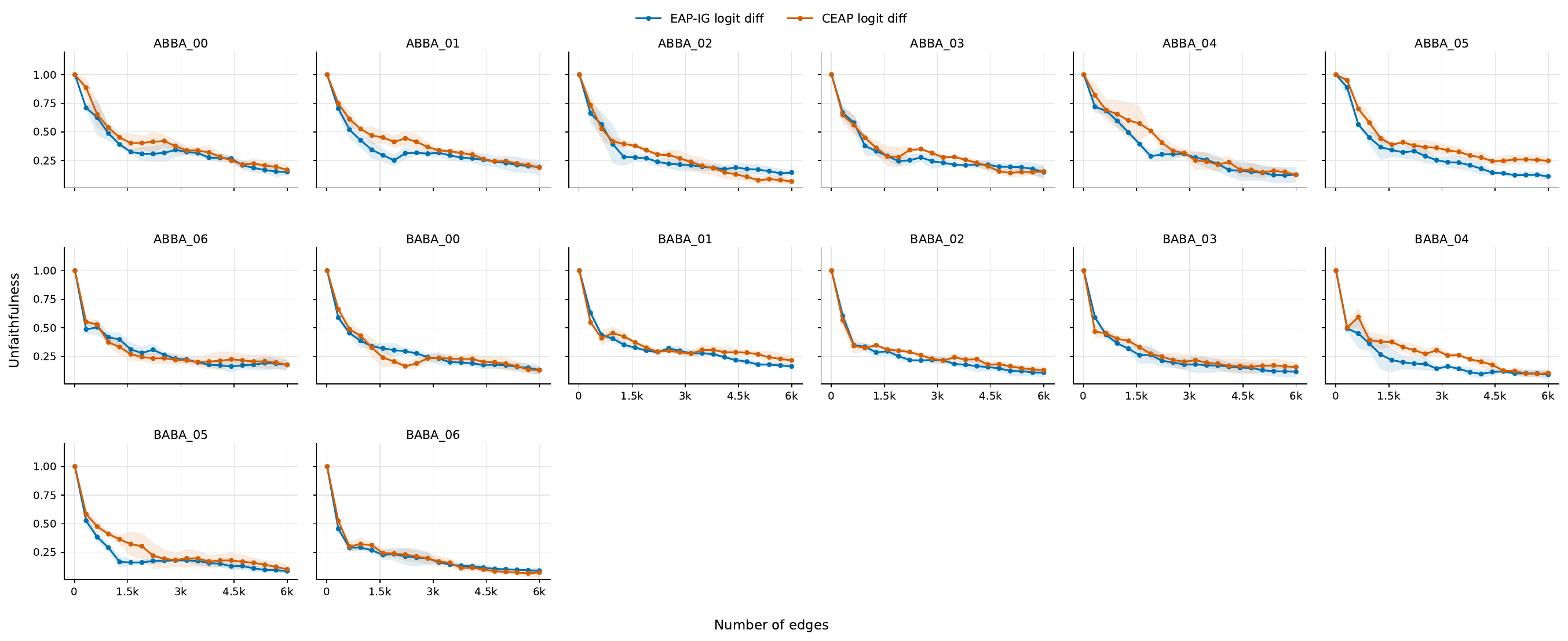}
    \caption{Unfaithfulness vs.\ number of edges for GPT-2 small on IOI.}
\end{figure}

\subsubsection{Greater-than}
\begin{figure}[H]
    \centering
    \includegraphics[width=\textwidth]{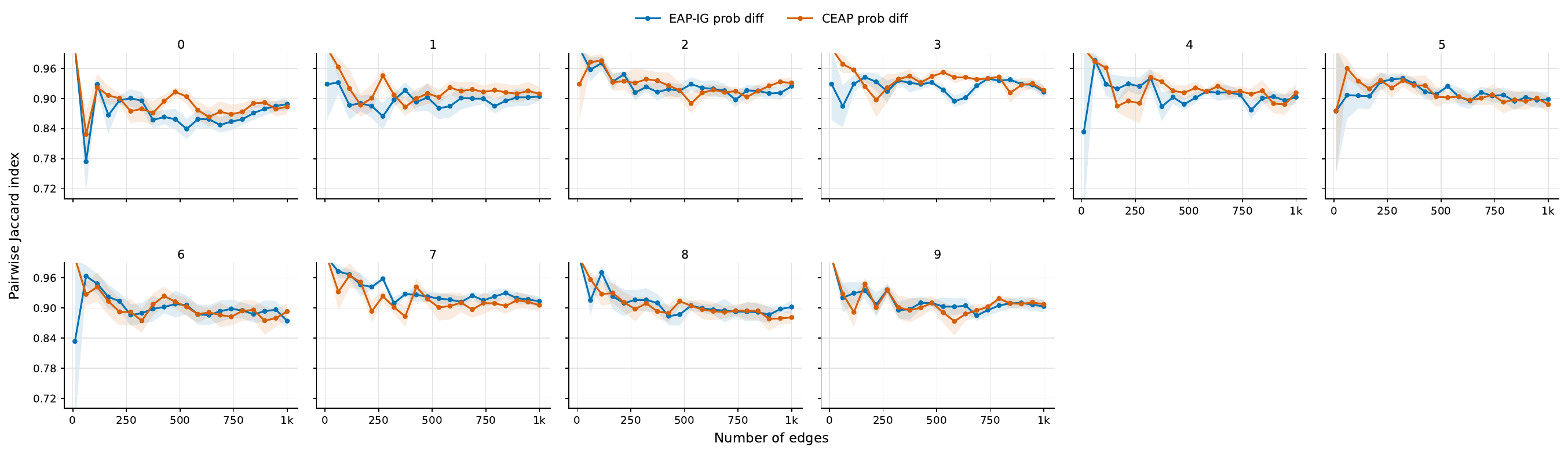}
    \caption{Pairwise Jaccard index vs.\ number of edges for GPT-2 small on Greater-than.}
\end{figure}

\begin{figure}[H]
    \centering
    \includegraphics[width=\textwidth]{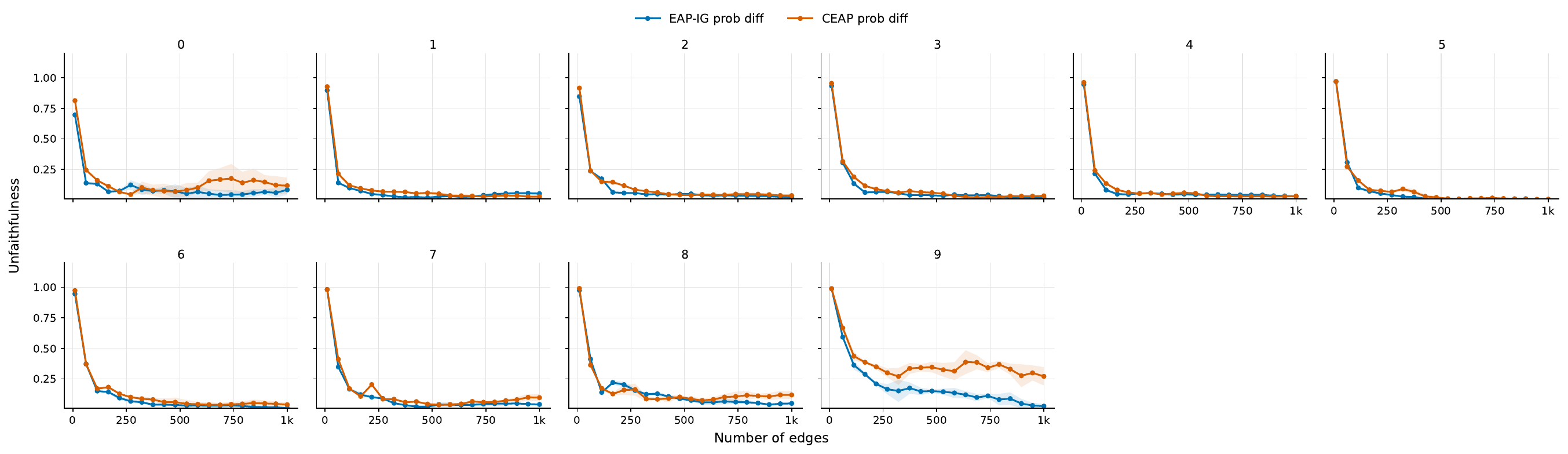}
    \caption{Unfaithfulness vs.\ number of edges for GPT-2 small on Greater-than.}
\end{figure}

\subsection{Pythia-160M}

\subsubsection{SVA}
\begin{figure}[H]
    \centering
    \includegraphics[width=\textwidth]{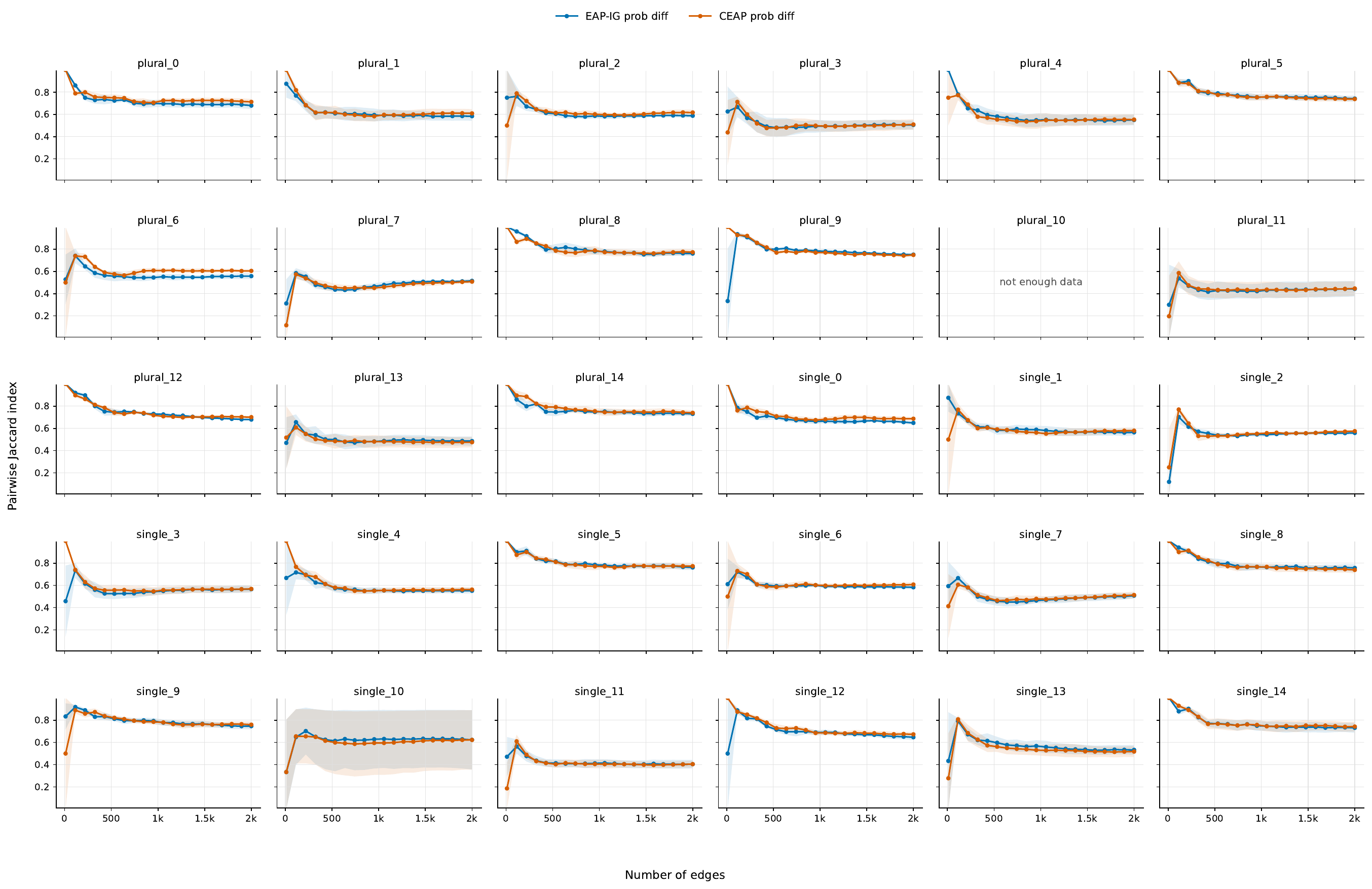}
    \caption{Pairwise Jaccard index vs.\ number of edges for Pythia-160M on SVA.}
\end{figure}

\begin{figure}[H]
    \centering
    \includegraphics[width=\textwidth]{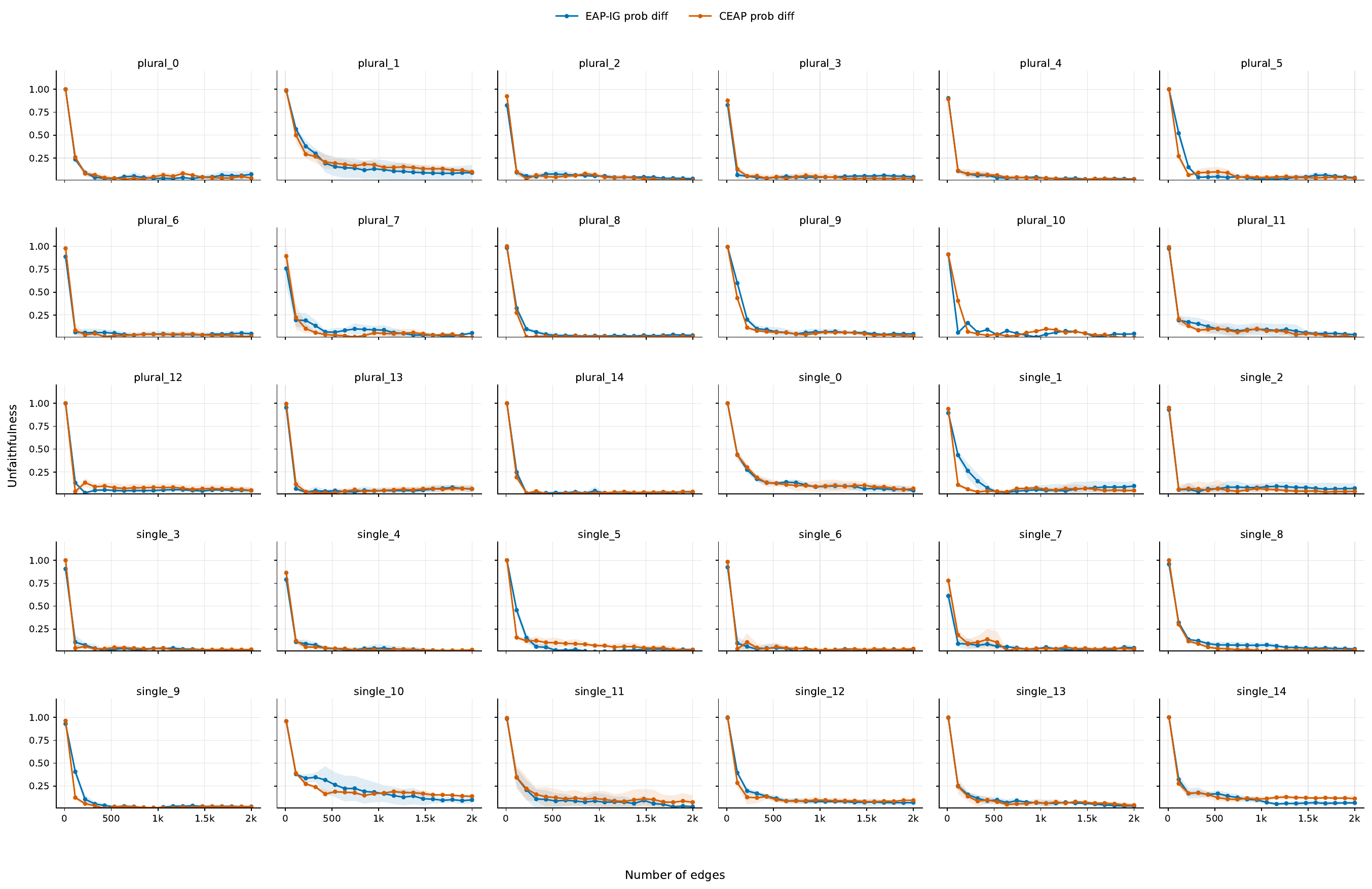}
    \caption{Unfaithfulness vs.\ number of edges for Pythia-160M on SVA.}
\end{figure}

\subsubsection{IOI}
\begin{figure}[H]
    \centering
    \includegraphics[width=\textwidth]{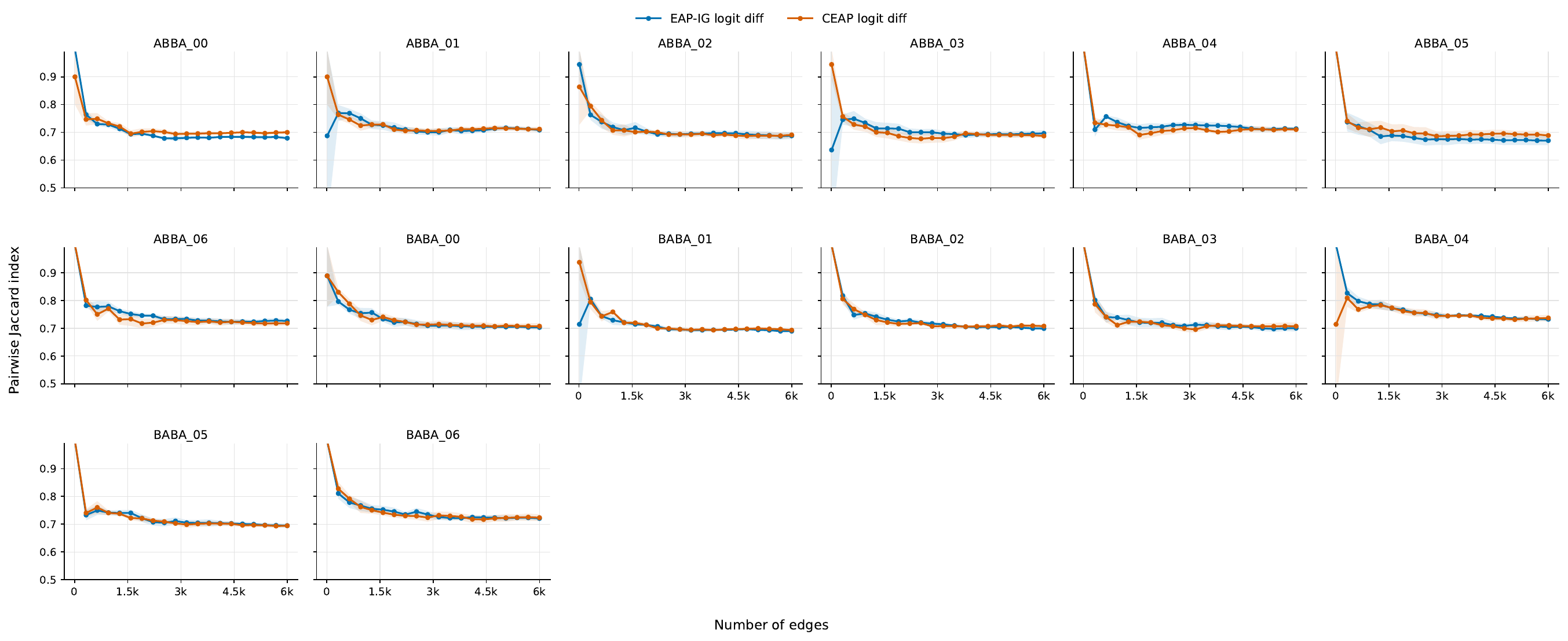}
    \caption{Pairwise Jaccard index vs.\ number of edges for Pythia-160M on IOI.}
\end{figure}

\begin{figure}[H]
    \centering
    \includegraphics[width=\textwidth]{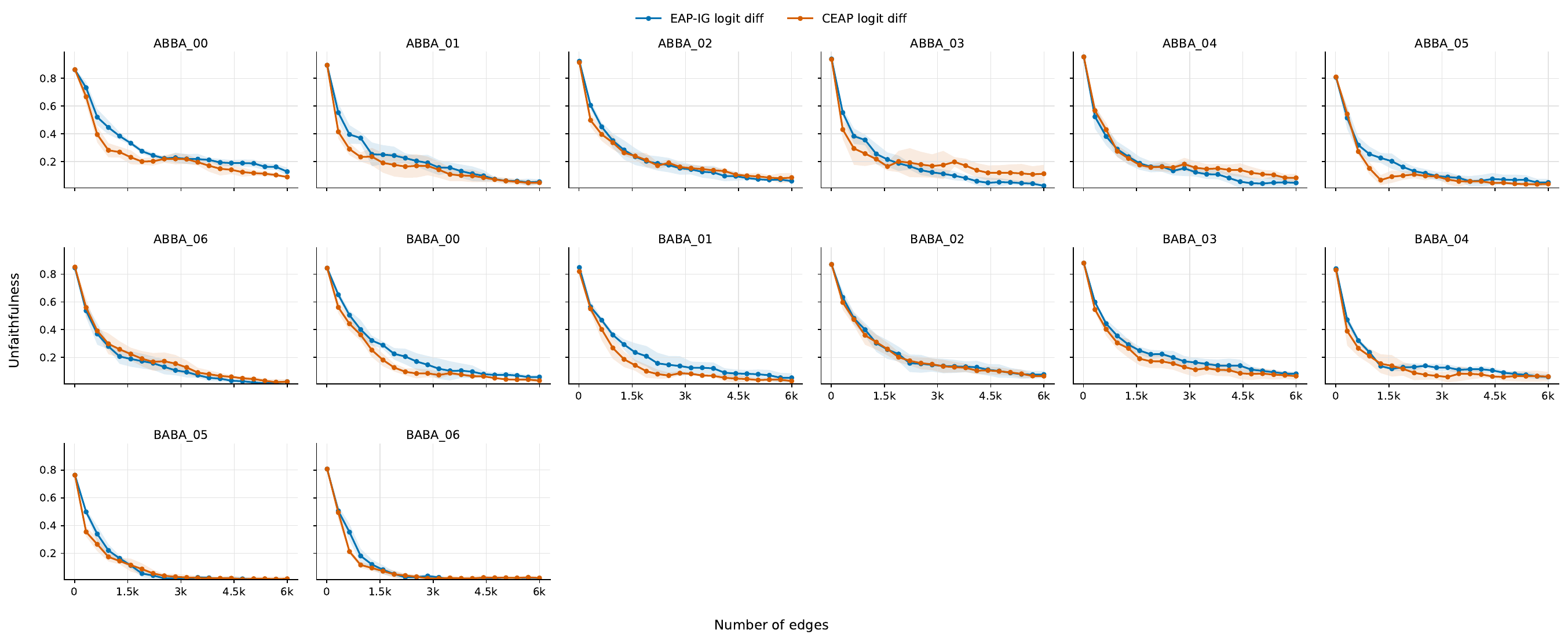}
    \caption{Unfaithfulness vs.\ number of edges for Pythia-160M on IOI.}
\end{figure}

\subsubsection{Greater-than}
\begin{figure}[H]
    \centering
    \includegraphics[width=\textwidth]{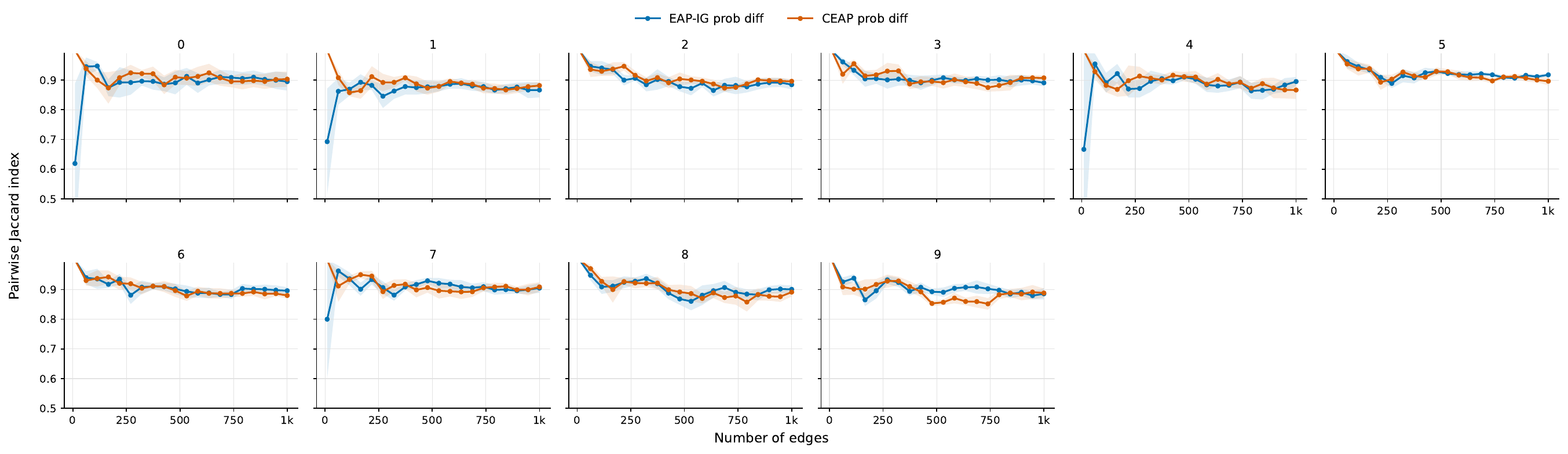}
    \caption{Pairwise Jaccard index vs.\ number of edges for Pythia-160M on Greater-than.}
\end{figure}

\begin{figure}[H]
    \centering
    \includegraphics[width=\textwidth]{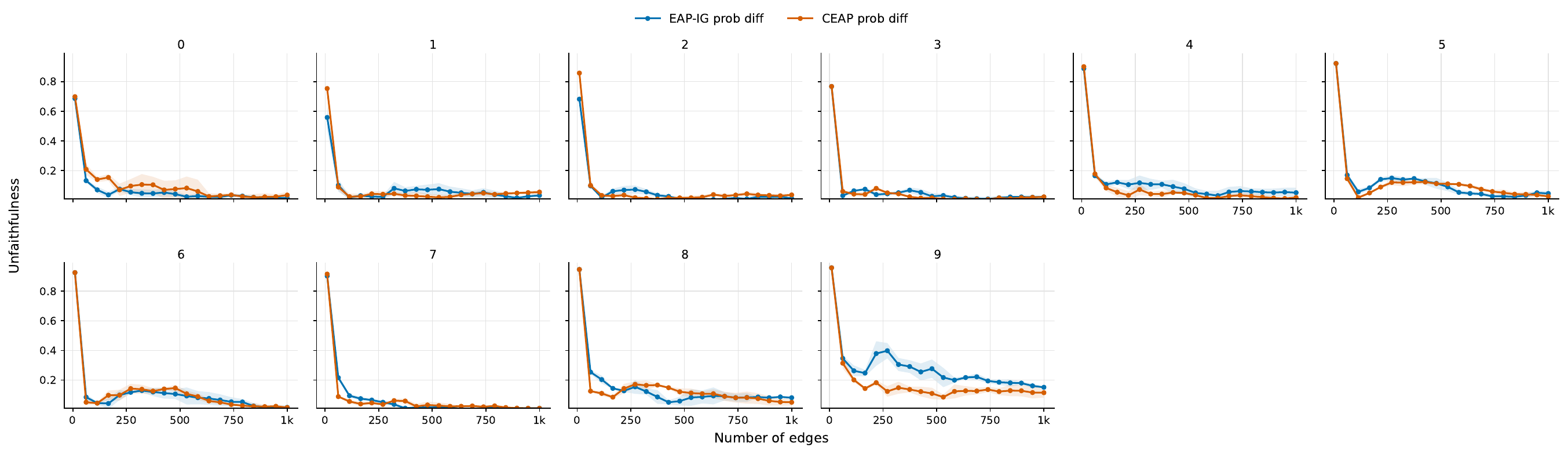}
    \caption{Unfaithfulness vs.\ number of edges for Pythia-160M on Greater-than.}
\end{figure}

\section{Template Absolute-Score Rank UMAP and Sample-Level Circuit Overlap Visualizations}
\label{app: umap and pji matrix}

This section pairs absolute-score rank vector UMAP embeddings of sample-level edge scores with sample-by-sample pairwise Jaccard index matrices for the greedily selected circuits.
We include GPT-2 small and Pythia-160M on SVA, IOI, and greater-than.

For the PJI matrix visualizations of GPT-2 small and Pythia-160M, we choose graph sizes large enough so that the average unfaithfulness across all samples falls below $0.2$.
The only exception is GPT-2 small on IOI, where we fix the graph size to $6000$ edges (about $20\%$ of all possible edges), which is the maximum graph size we experimented with; even at this size, the unfaithfulness remains above $0.2$.

\subsection{GPT-2 small}

\subsubsection{SVA}
\begin{figure}[H]
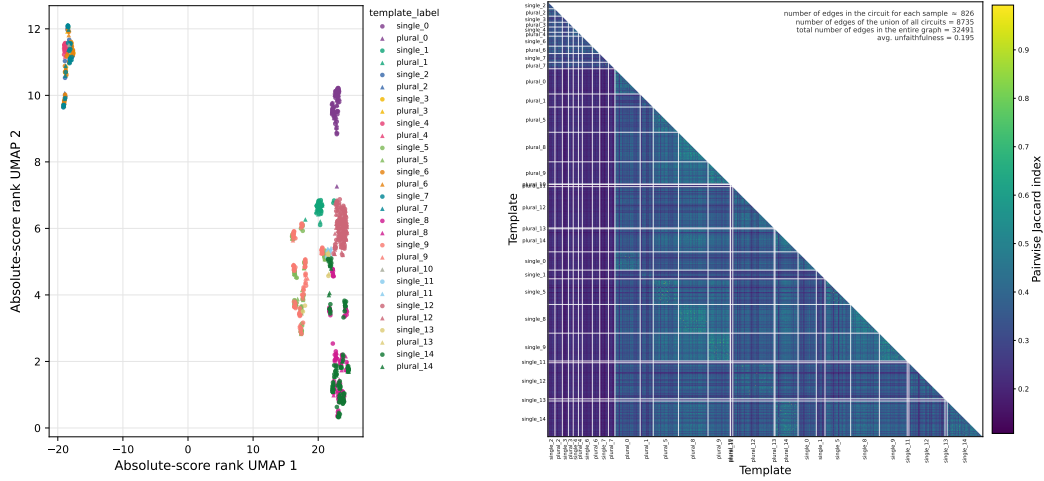

    \centering
    \begin{subfigure}[t]{0.44\textwidth}
        \centering
        \includegraphics[width=\linewidth,height=0.28\textheight,keepaspectratio]{figs/gpt2_sva_template_rank_umap_prob_diff.pdf}
        \caption{Absolute-score rank UMAP.}
    \end{subfigure}
    \hfill
    \begin{subfigure}[t]{0.52\textwidth}
        \centering
        \includegraphics[width=\linewidth,height=0.28\textheight,keepaspectratio]{figs/05_pairwise_sample_edge_iou_gpt2_sva.pdf}
        \caption{Pairwise Jaccard index.}
    \end{subfigure}
    \caption{Template-induced circuit difference for GPT-2 small on SVA.}
\end{figure}

\subsubsection{IOI}
\begin{figure}[H]
    \centering
    \begin{subfigure}[t]{0.44\textwidth}
        \centering
        \includegraphics[width=\linewidth,height=0.28\textheight,keepaspectratio]{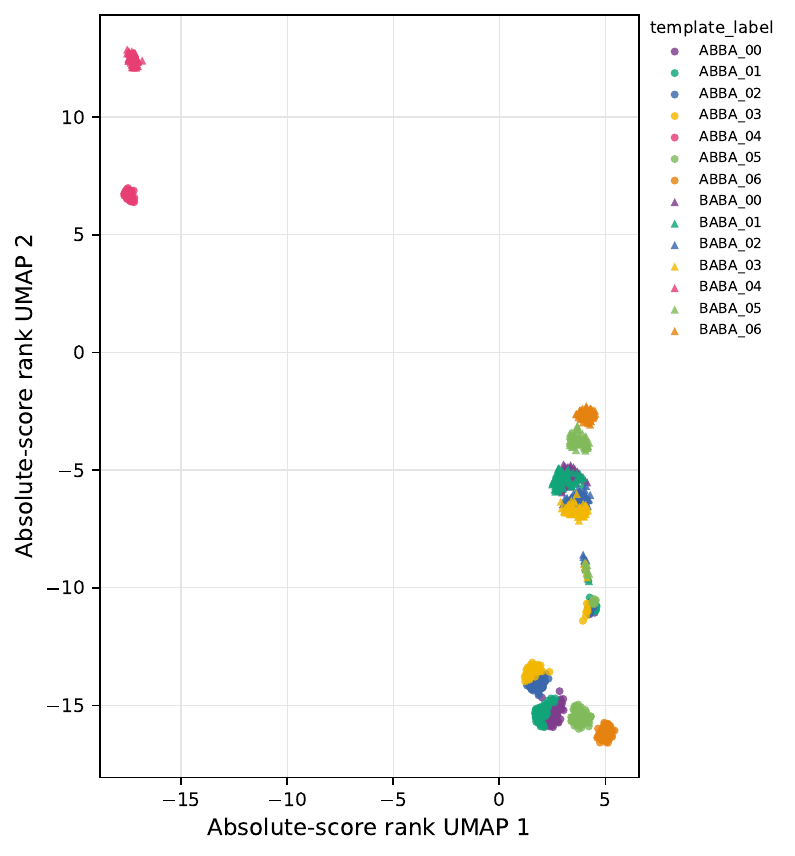}
        \caption{Absolute-score rank UMAP.}
    \end{subfigure}
    \hfill
    \begin{subfigure}[t]{0.52\textwidth}
        \centering
        \includegraphics[width=\linewidth,height=0.28\textheight,keepaspectratio]{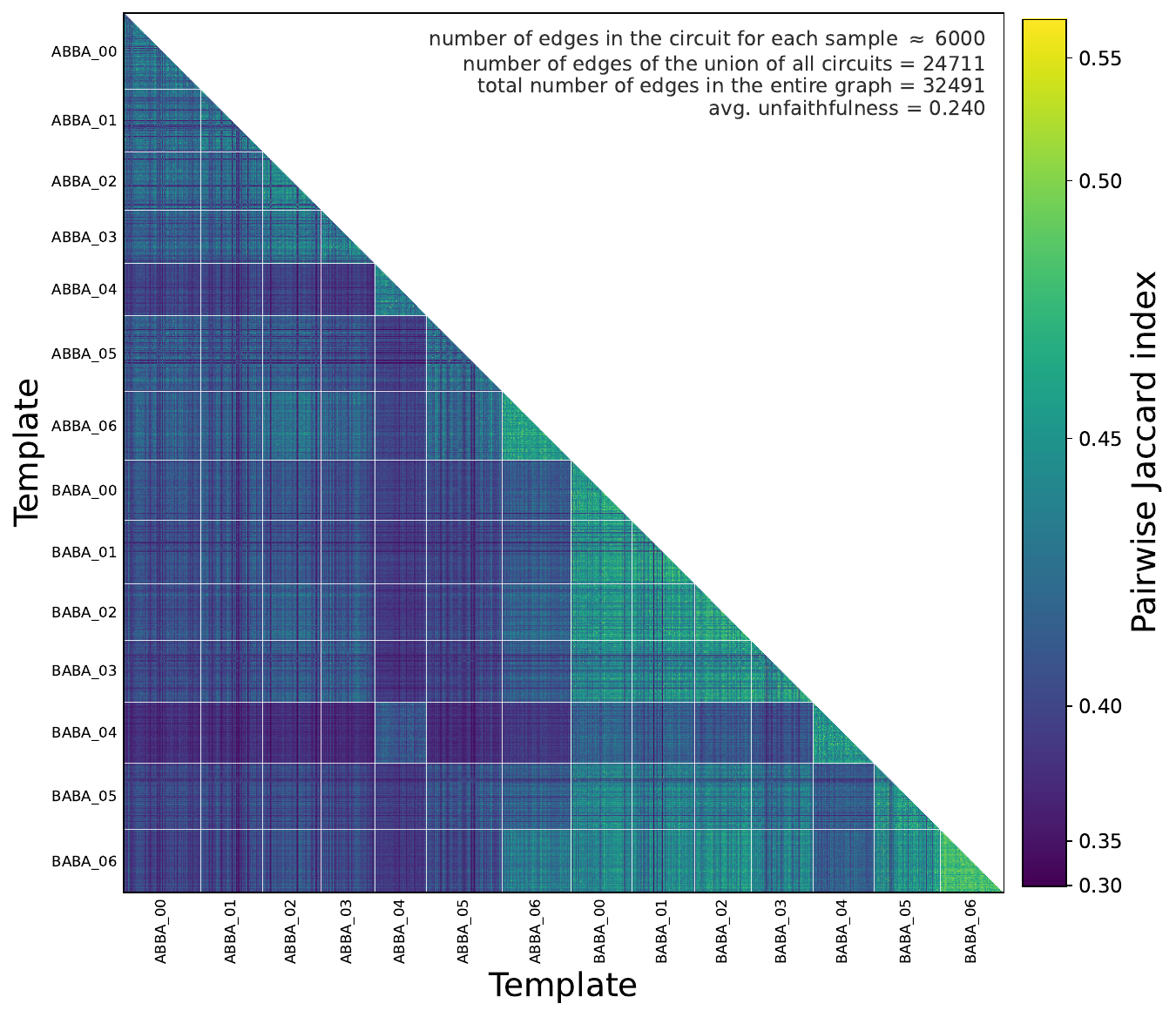}
        \caption{Pairwise Jaccard index.}
    \end{subfigure}
    \caption{Template-induced circuit difference for GPT-2 small on IOI.}
\end{figure}

\subsubsection{Greater-than}
\begin{figure}[H]
    \centering
    \begin{subfigure}[t]{0.44\textwidth}
        \centering
        \includegraphics[width=\linewidth,height=0.28\textheight,keepaspectratio]{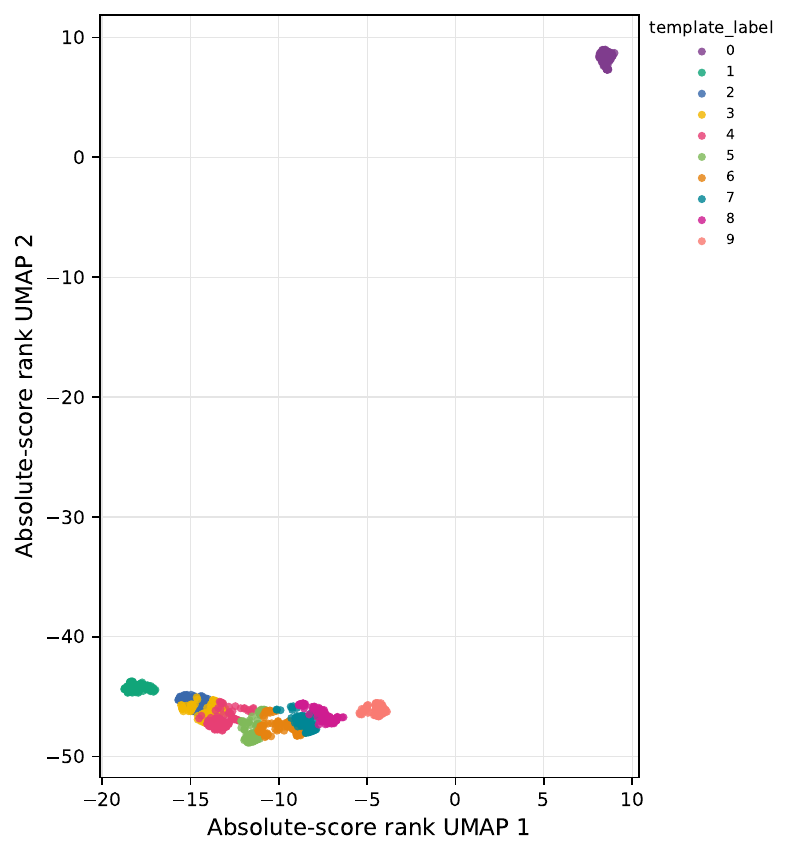}
        \caption{Absolute-score rank UMAP.}
    \end{subfigure}
    \hfill
    \begin{subfigure}[t]{0.52\textwidth}
        \centering
        \includegraphics[width=\linewidth,height=0.28\textheight,keepaspectratio]{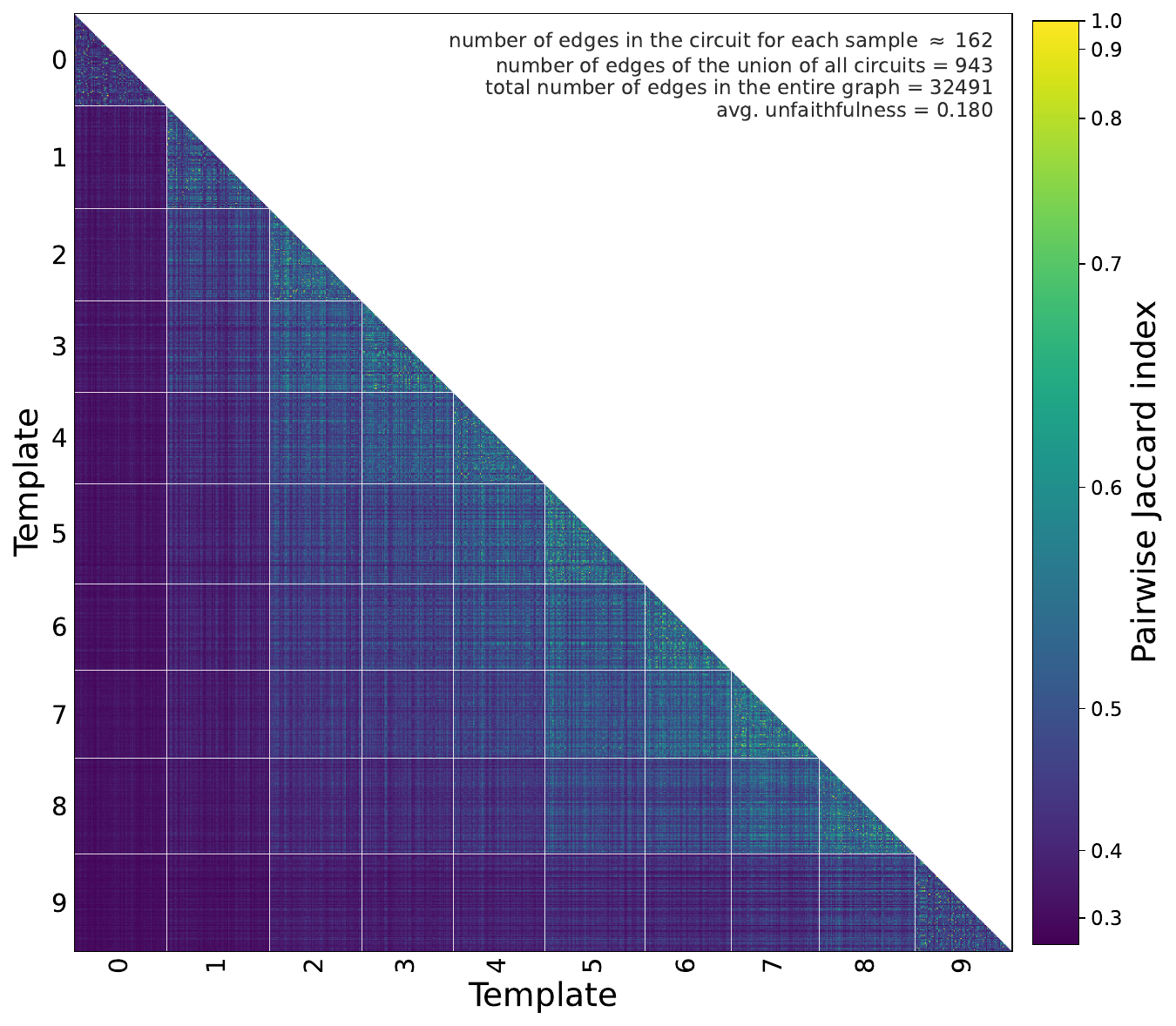}
        \caption{Pairwise Jaccard index.}
    \end{subfigure}
    \caption{Template-induced circuit difference for GPT-2 small on greater-than.}
\end{figure}

\subsection{Pythia-160M}

\subsubsection{SVA}
\begin{figure}[H]
    \centering
    \begin{subfigure}[t]{0.44\textwidth}
        \centering
        \includegraphics[width=\linewidth,height=0.28\textheight,keepaspectratio]{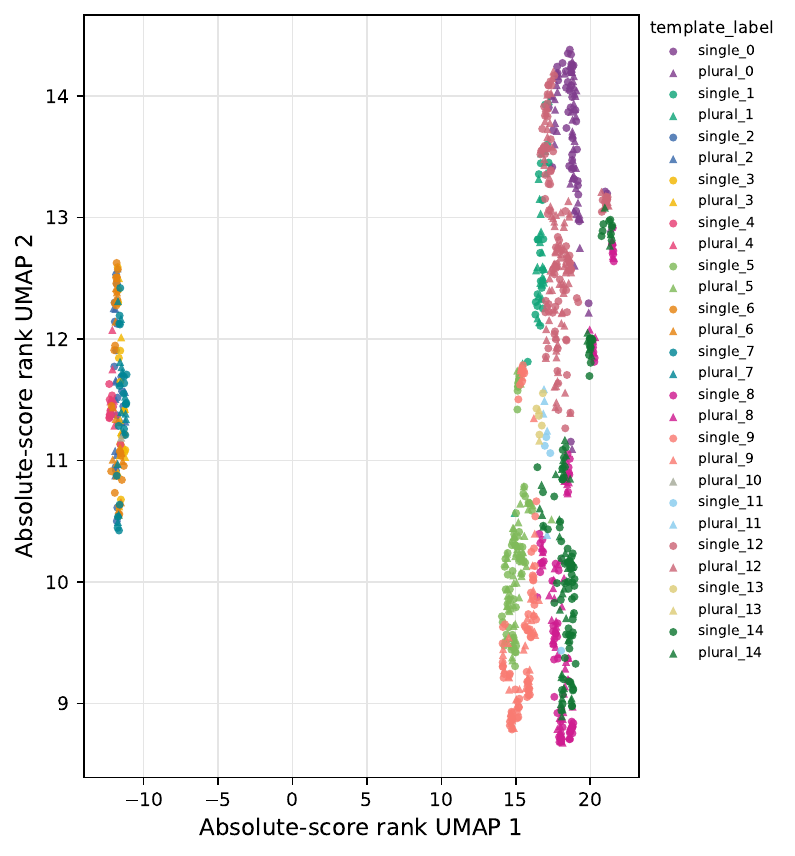}
        \caption{Absolute-score rank UMAP.}
    \end{subfigure}
    \hfill
    \begin{subfigure}[t]{0.52\textwidth}
        \centering
        \includegraphics[width=\linewidth,height=0.28\textheight,keepaspectratio]{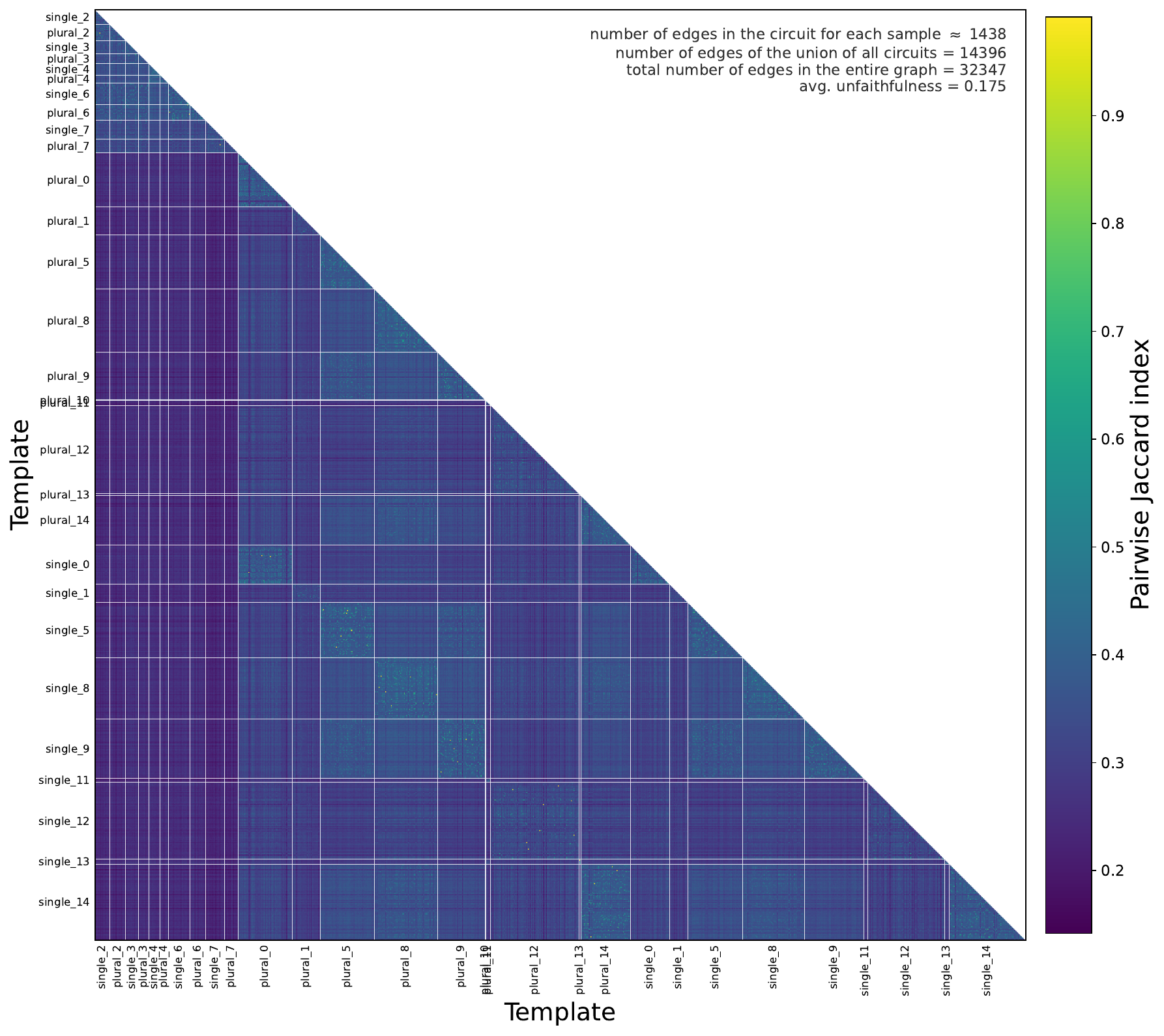}
        \caption{Pairwise Jaccard index.}
    \end{subfigure}
    \caption{Template-induced circuit difference for Pythia-160M on SVA.}
\end{figure}

\subsubsection{IOI}
\begin{figure}[H]
    \centering
    \begin{subfigure}[t]{0.44\textwidth}
        \centering
        \includegraphics[width=\linewidth,height=0.28\textheight,keepaspectratio]{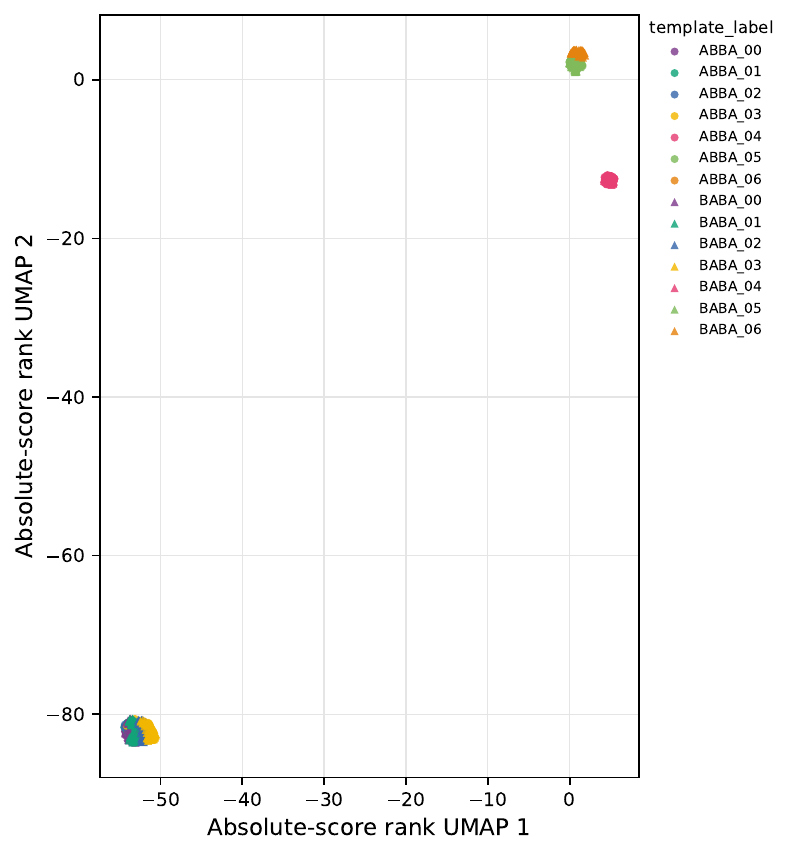}
        \caption{Absolute-score rank UMAP.}
    \end{subfigure}
    \hfill
    \begin{subfigure}[t]{0.52\textwidth}
        \centering
        \includegraphics[width=\linewidth,height=0.28\textheight,keepaspectratio]{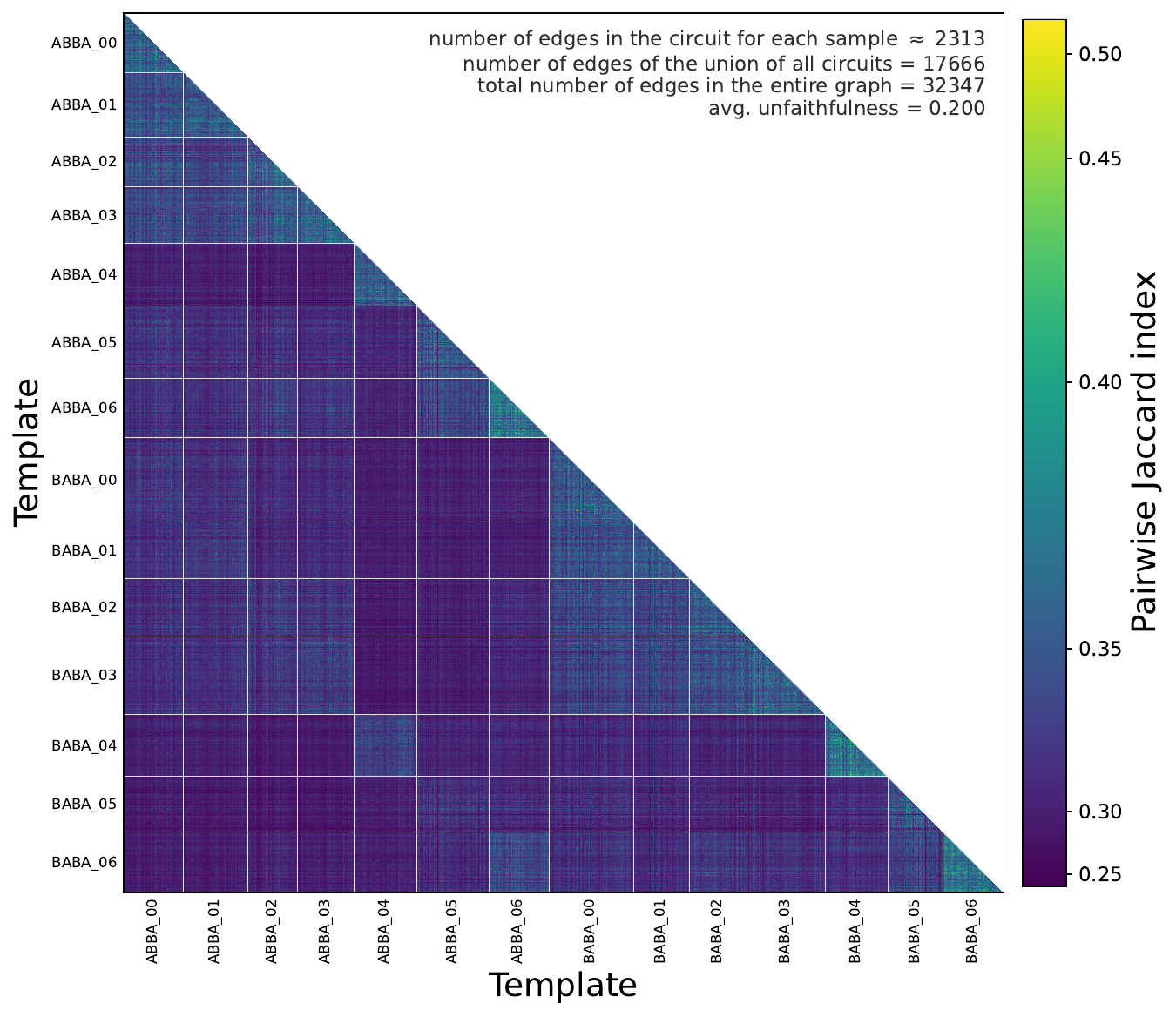}
        \caption{Pairwise Jaccard index.}
    \end{subfigure}
    \caption{Template-induced circuit difference for Pythia-160M on IOI.}
\end{figure}

\subsubsection{Greater-than}
\begin{figure}[H]
    \centering
    \begin{subfigure}[t]{0.44\textwidth}
        \centering
        \includegraphics[width=\linewidth,height=0.28\textheight,keepaspectratio]{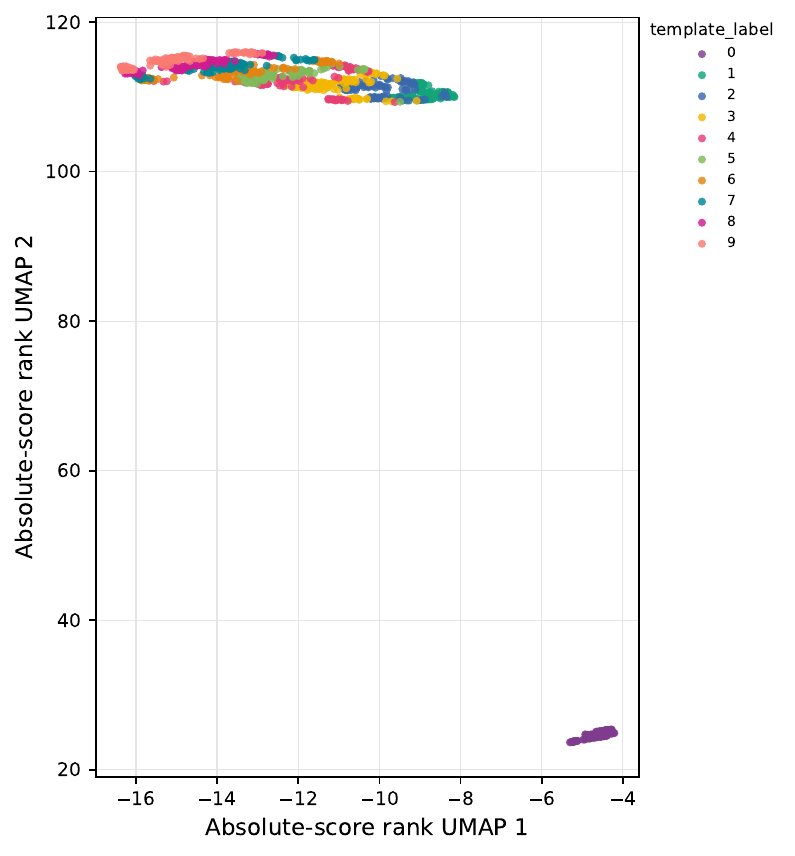}
        \caption{Absolute-score rank UMAP.}
    \end{subfigure}
    \hfill
    \begin{subfigure}[t]{0.52\textwidth}
        \centering
        \includegraphics[width=\linewidth,height=0.28\textheight,keepaspectratio]{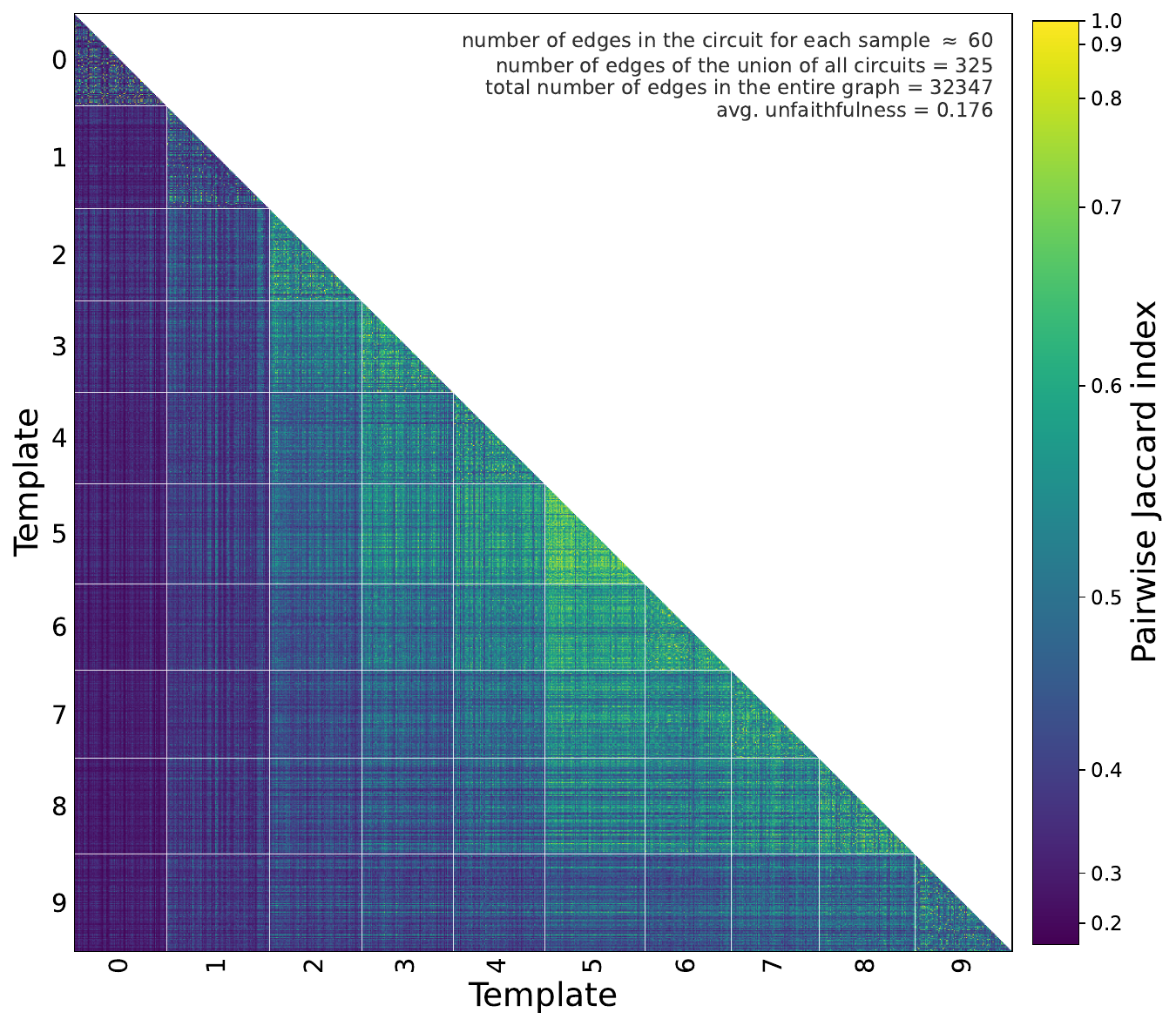}
        \caption{Pairwise Jaccard index.}
    \end{subfigure}
    \caption{Template-induced circuit difference for Pythia-160M on greater-than.}
\end{figure}

\section{Patching Circuits Found for a Sample Might Act Opposite to Intention on Another Sample}
\label{app: fail to steer}
\begin{table}[H]
\centering
\small
\setlength{\tabcolsep}{6pt}
\renewcommand{\arraystretch}{1.18}
\begin{tabular}{@{}>{\raggedright\arraybackslash}p{0.25\linewidth}
                >{\raggedright\arraybackslash}p{0.34\linewidth}
                >{\raggedright\arraybackslash}p{0.28\linewidth}@{}}
\toprule
\rowcolor{black!6}
\textbf{Quantity} & \textbf{Sample 93} & \textbf{Sample 858} \\
\midrule
Clean prompt &
\textit{The senators like to watch television shows and} &
\textit{The teachers the parent} \\
Corrupted prompt &
\textit{The senator likes to watch television shows and} &
\textit{The teachers the parents} \\
\addlinespace[2pt]
Clean \texttt{prob\_diff} &
$0.352503$ &
$0.119733$ \\
Own-circuit \texttt{prob\_diff} &
$0.336449$ &
$0.072026$ \\
Cross-circuit \texttt{prob\_diff} &
$0.256708$ &
$-0.223778$ \\
\bottomrule
\end{tabular}
\caption{
Cross-sample circuit patching between samples 93 and 858 of IOI. The own-circuit rows evaluate each sample using the circuit discovered on that sample, whereas the cross-circuit rows evaluate each sample using the circuit discovered on the other sample.
}
\label{tab:cross-circuit-93-858}
\end{table}

We use the saved
per-sample greedy circuits for GPT-2 on the SVA task, with the
\texttt{prob\_diff} metric and target circuit size 826. For a given sample here, the
standard evaluation patches activations in the circuit so as to recover the
movement from the pro-singular corrupted prompt toward the pro-plural prompt.
The cross-sample condition keeps the input pair fixed but replaces the circuit
with the circuit discovered on the other sample, thereby testing whether the
learned pathway transfers across examples.

Although the two circuits are comparable in size after pruning,\footnote{A procedure implemented by \citep{faithfulness} to remove headless and tailless branches.} with 712 edges
for sample 93 and 754 edges for sample 858, their edge-level overlap is small
(IoU $=0.125960$). This low overlap is accompanied by a qualitative failure of
transfer: when the circuit from sample 93 is used to patch sample 858, the
\texttt{prob\_diff} moves from $0.119733$ under the clean full-model evaluation
to $-0.223778$ under cross-circuit evaluation. Thus the transferred circuit does
not merely recover less of the intended effect; it drives the logit movement in
the opposite direction from the intended pro-plural intervention.

\section[Template Absolute-Score Rank UMAP and Sample-Level Circuit Overlap Visualizations]{Template Absolute-Score Rank UMAP and Sample-Level Circuit Overlap Visualizations for Models by \citep{gao2025weight}}
\label{app:sparse umap pji plot}
We repeat the visualization for three models open-sourced by \citep{gao2025weight}:
\begin{itemize}[leftmargin=*,itemsep=0.1em,parsep=0pt,topsep=0.2em]
    \item \texttt{dense1\_4x},
    \item \texttt{csp\_sweep1\_4x\_3.7Mnonzero\_afrac1.000},
    \item \texttt{csp\_sweep1\_4x\_3.7Mnonzero\_afrac0.500}.
\end{itemize}
Among these, \texttt{dense1\_4x} and \texttt{csp\_sweep1\_4x\_3.7Mnonzero\_afrac1.000} form the closest dense-sparse pair by test loss; the latter has $3.7$M nonzero parameters out of $110$M total parameters.
The difference between the last two models (as indicated by name: \texttt{afrac1.000} versus \texttt{afrac0.500})
is that the latter imposed activation sparsity on top of weight sparsity.
Since the unfaithfulness of all models on both tasks is very low (see \cref{table: single-double-quote overlap,table: else-elif overlap}) but can still differ by two orders of magnitude, we use circuits containing roughly $5.2\%$ of all graph edges for the PJI plots, rather than choosing an edge count based on unfaithfulness.
\cref{table: single-double-quote overlap,table: else-elif overlap} shows that cross-template overlap remains low, and the visualizations below show that sparse models still deploy different circuits for different templates.

\subsection{single-double-quote}

\subsubsection{\texttt{dense1\_4x}}
\begin{figure}[H]
    \centering
    \begin{subfigure}[t]{0.44\textwidth}
        \centering
        \includegraphics[width=\linewidth,height=0.28\textheight,keepaspectratio]{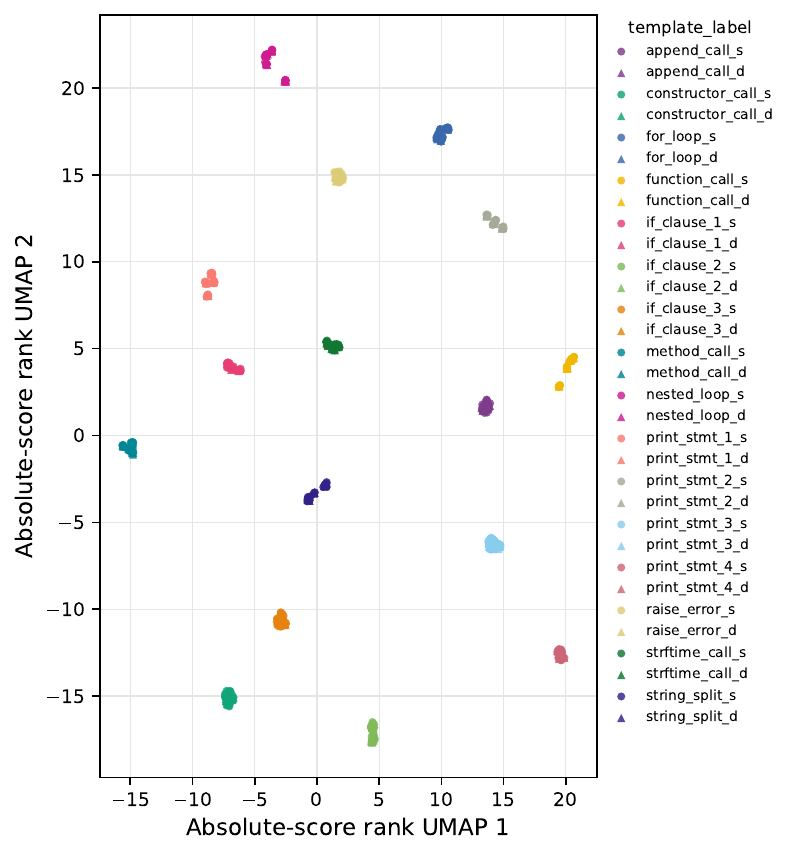}
        \caption{Absolute-score rank UMAP.}
    \end{subfigure}
    \hfill
    \begin{subfigure}[t]{0.52\textwidth}
        \centering
        \includegraphics[width=\linewidth,height=0.28\textheight,keepaspectratio]{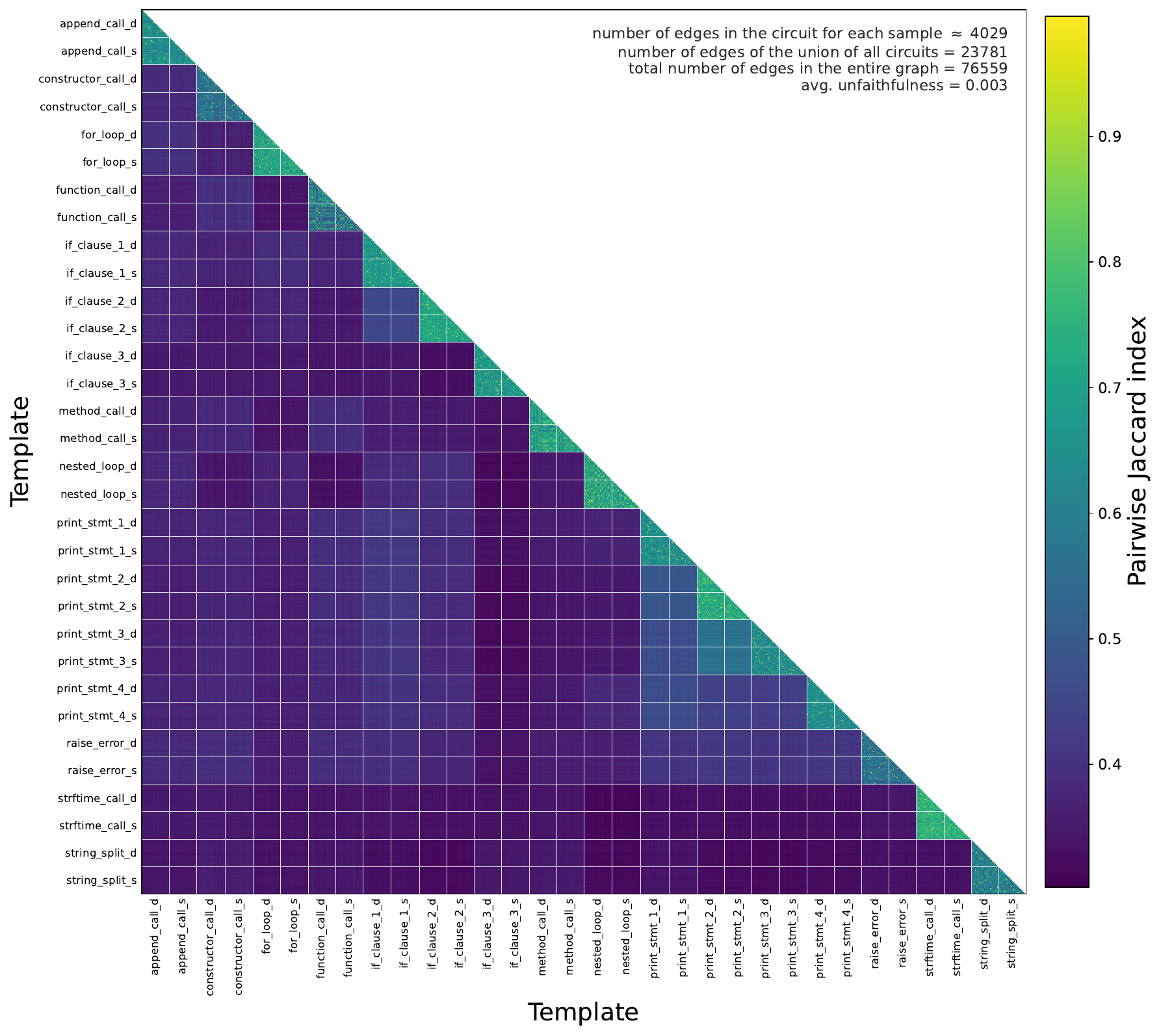}
        \caption{Pairwise Jaccard index.}
    \end{subfigure}
    \caption{Template-induced circuit difference for \texttt{dense1\_4x} on single-double-quote.}
\end{figure}

\subsubsection{\texttt{csp\_sweep1\_4x\_3.7Mnonzero\_afrac1.000}}
\begin{figure}[H]
    \centering
    \begin{subfigure}[t]{0.44\textwidth}
        \centering
        \includegraphics[width=\linewidth,height=0.28\textheight,keepaspectratio]{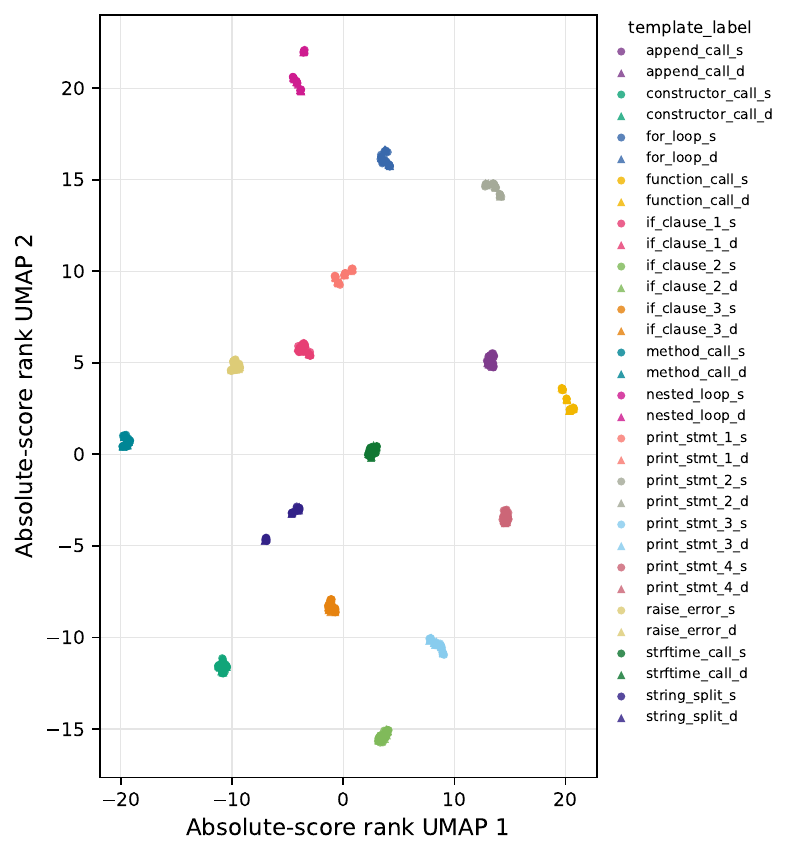}
        \caption{Absolute-score rank UMAP.}
    \end{subfigure}
    \hfill
    \begin{subfigure}[t]{0.52\textwidth}
        \centering
        \includegraphics[width=\linewidth,height=0.28\textheight,keepaspectratio]{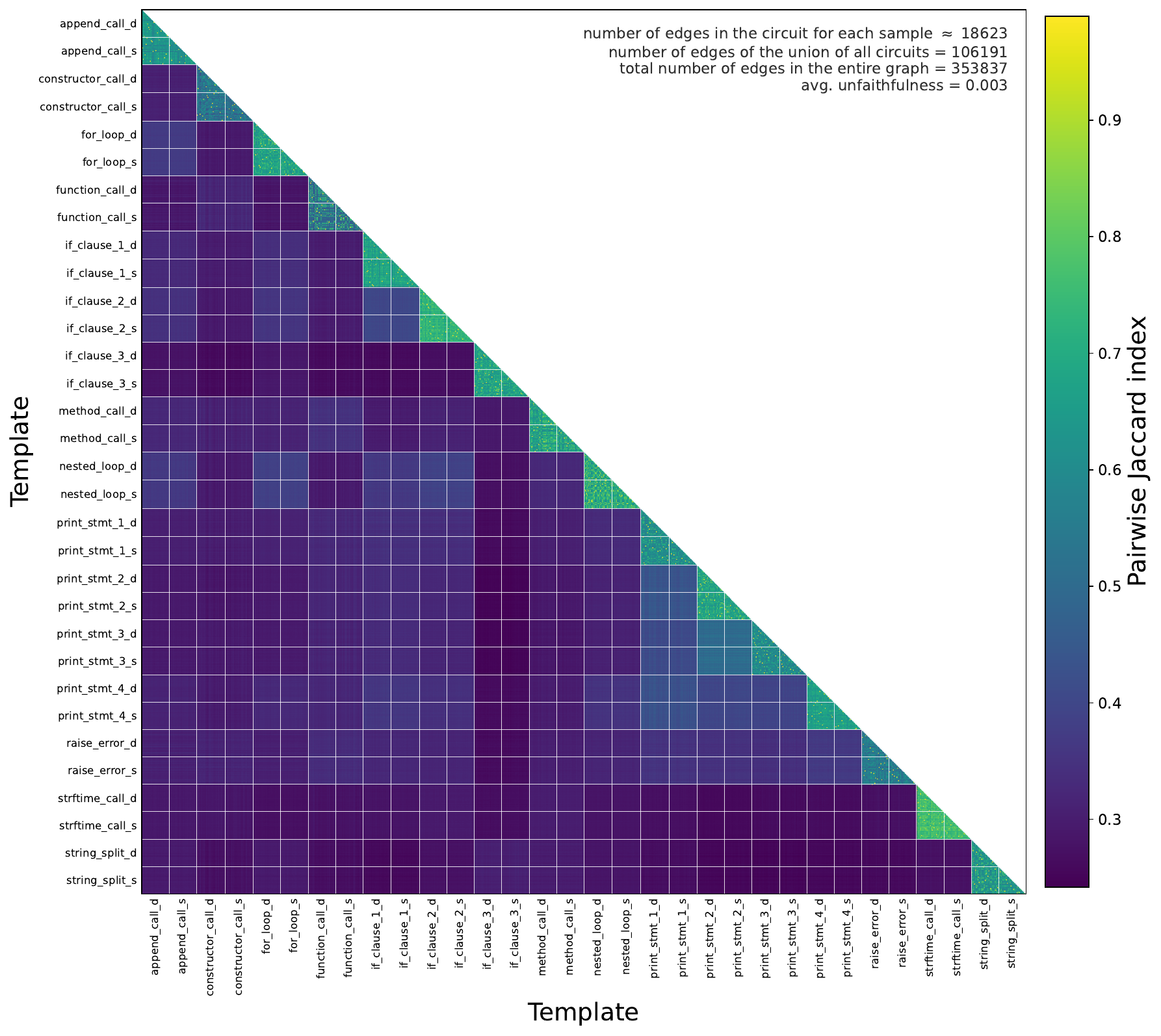}
        \caption{Pairwise Jaccard index.}
    \end{subfigure}
    \caption{Template-induced circuit difference for \texttt{csp\_sweep1\_4x\_3.7Mnonzero\_afrac1.000} on single-double-quote.}
\end{figure}

\subsubsection{\texttt{csp\_sweep1\_4x\_3.7Mnonzero\_afrac0.500}}
\begin{figure}[H]
    \centering
    \begin{subfigure}[t]{0.44\textwidth}
        \centering
        \includegraphics[width=\linewidth,height=0.28\textheight,keepaspectratio]{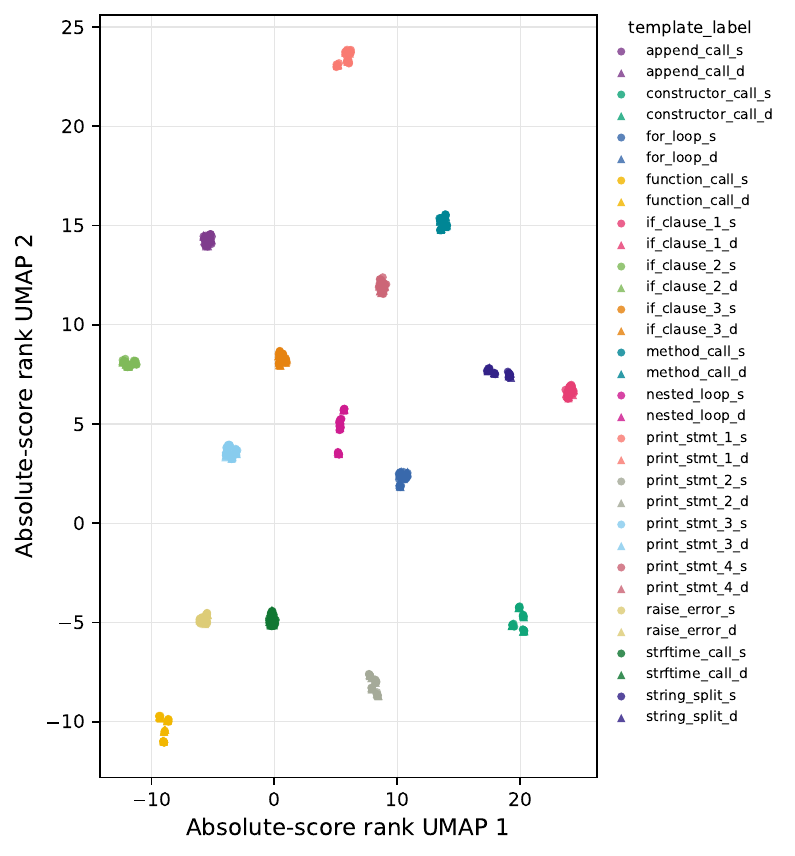}
        \caption{Absolute-score rank UMAP.}
    \end{subfigure}
    \hfill
    \begin{subfigure}[t]{0.52\textwidth}
        \centering
        \includegraphics[width=\linewidth,height=0.28\textheight,keepaspectratio]{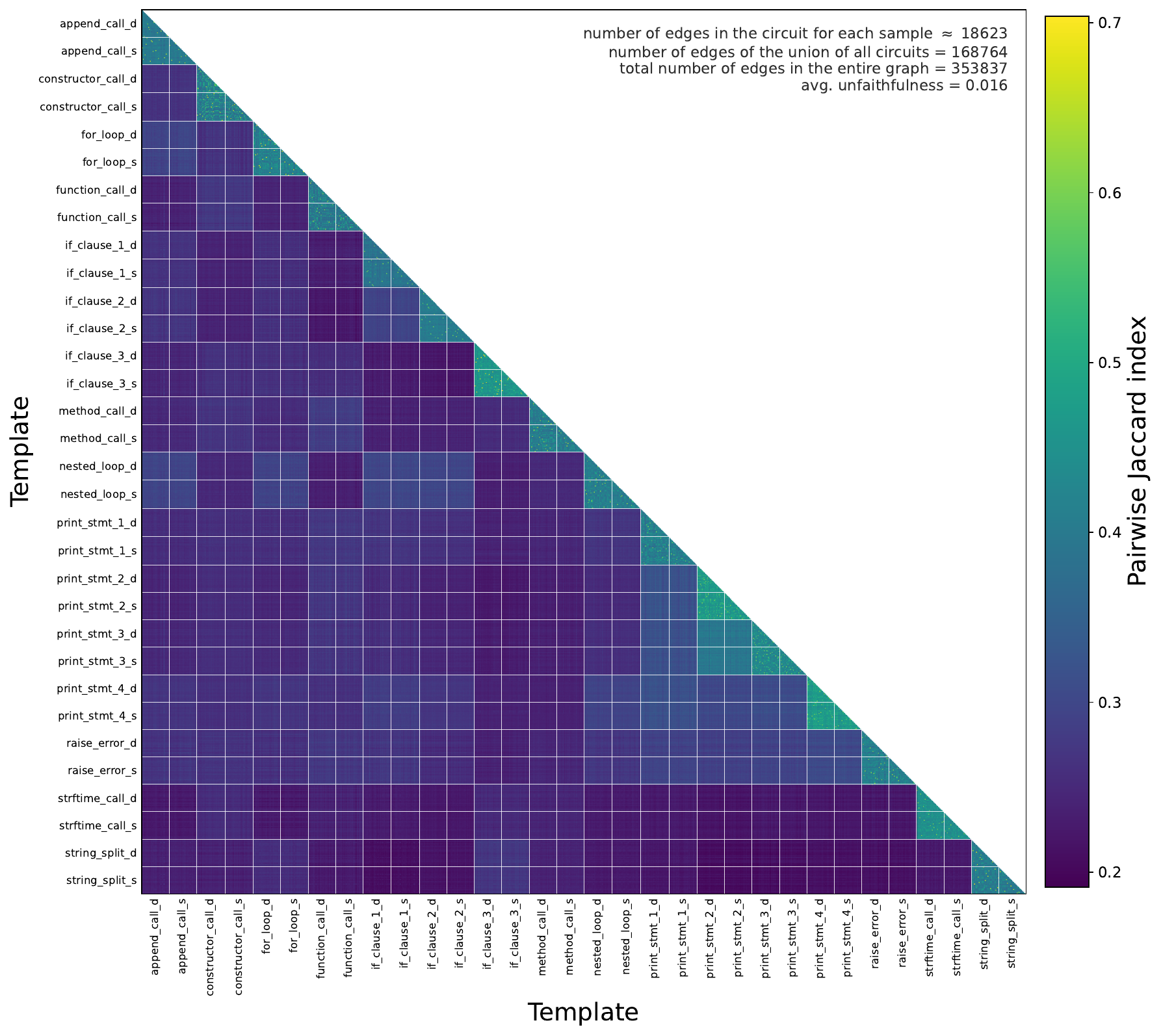}
        \caption{Pairwise Jaccard index.}
    \end{subfigure}
    \caption{Template-induced circuit difference for \texttt{csp\_sweep1\_4x\_3.7Mnonzero\_afrac0.500} on single-double-quote.}
\end{figure}

\subsection{else-elif}

\subsubsection{\texttt{dense1\_4x}}
\begin{figure}[H]
    \centering
    \begin{subfigure}[t]{0.44\textwidth}
        \centering
        \includegraphics[width=\linewidth,height=0.28\textheight,keepaspectratio]{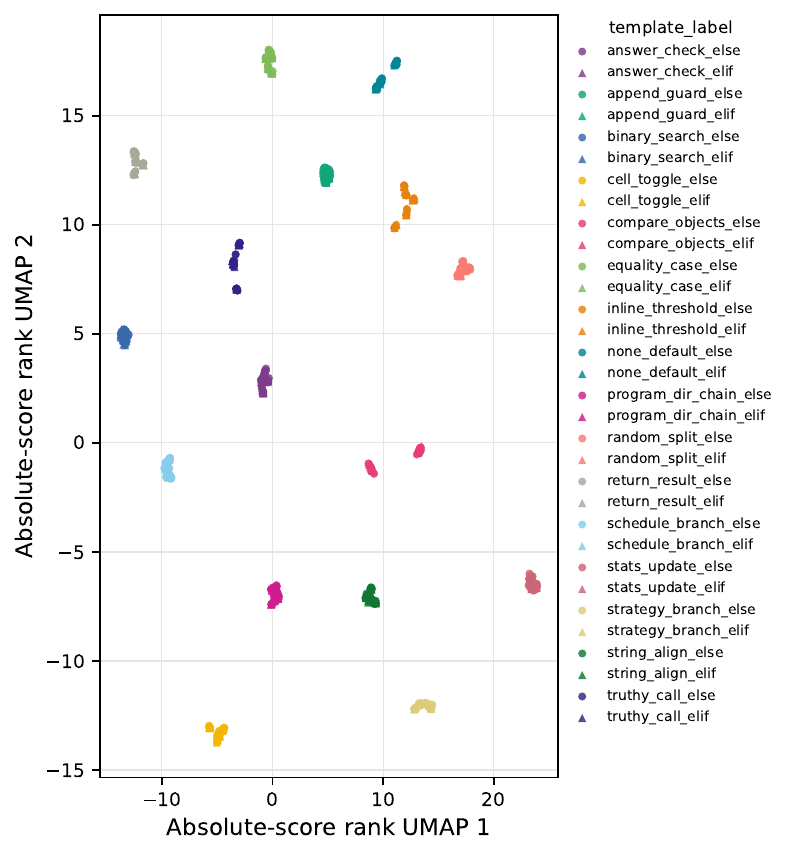}
        \caption{Absolute-score rank UMAP.}
    \end{subfigure}
    \hfill
    \begin{subfigure}[t]{0.52\textwidth}
        \centering
        \includegraphics[width=\linewidth,height=0.28\textheight,keepaspectratio]{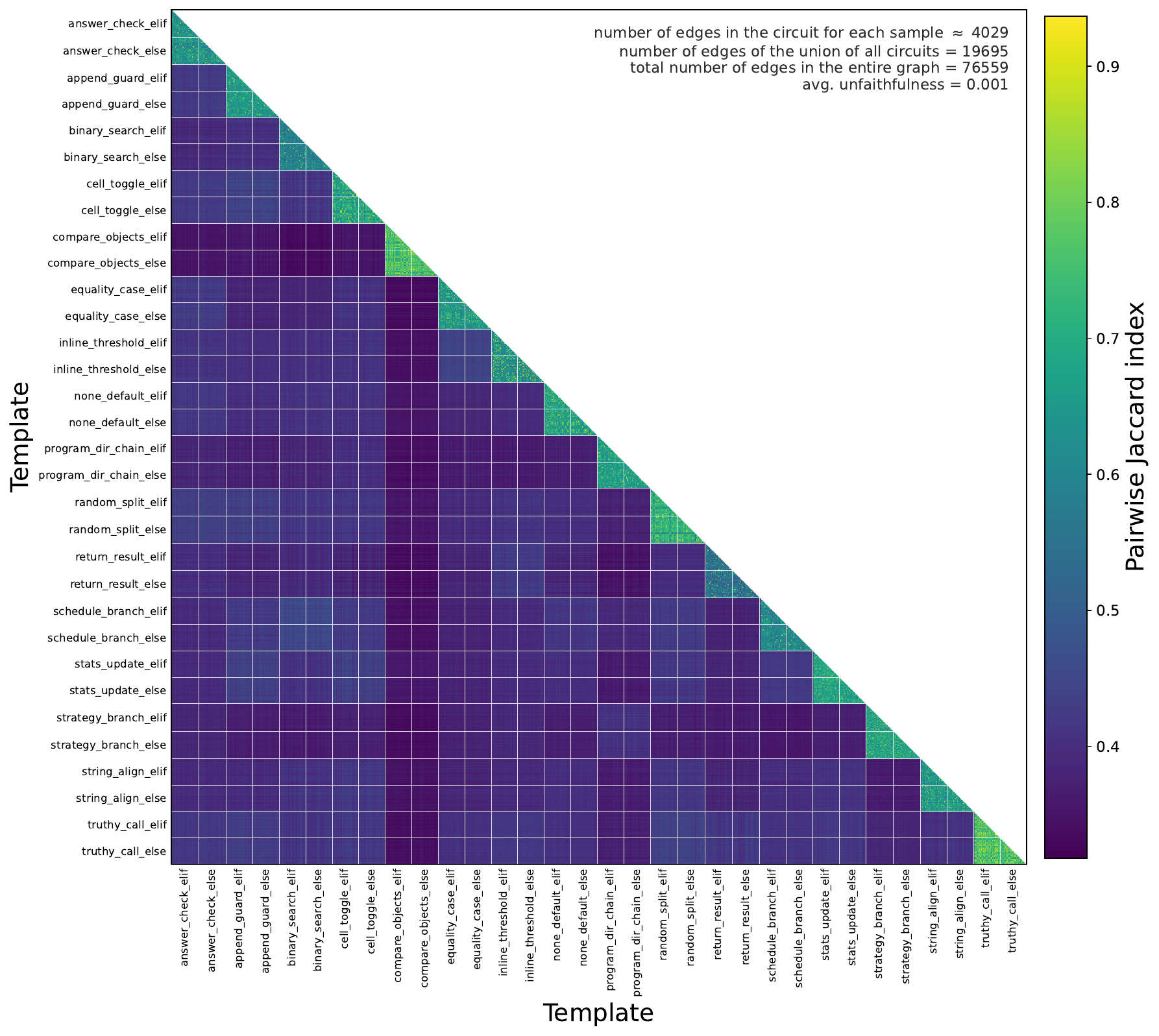}
        \caption{Pairwise Jaccard index.}
    \end{subfigure}
    \caption{Template-induced circuit difference for \texttt{dense1\_4x} on else-elif.}
\end{figure}

\subsubsection{\texttt{csp\_sweep1\_4x\_3.7Mnonzero\_afrac1.000}}
\begin{figure}[H]
    \centering
    \begin{subfigure}[t]{0.44\textwidth}
        \centering
        \includegraphics[width=\linewidth,height=0.28\textheight,keepaspectratio]{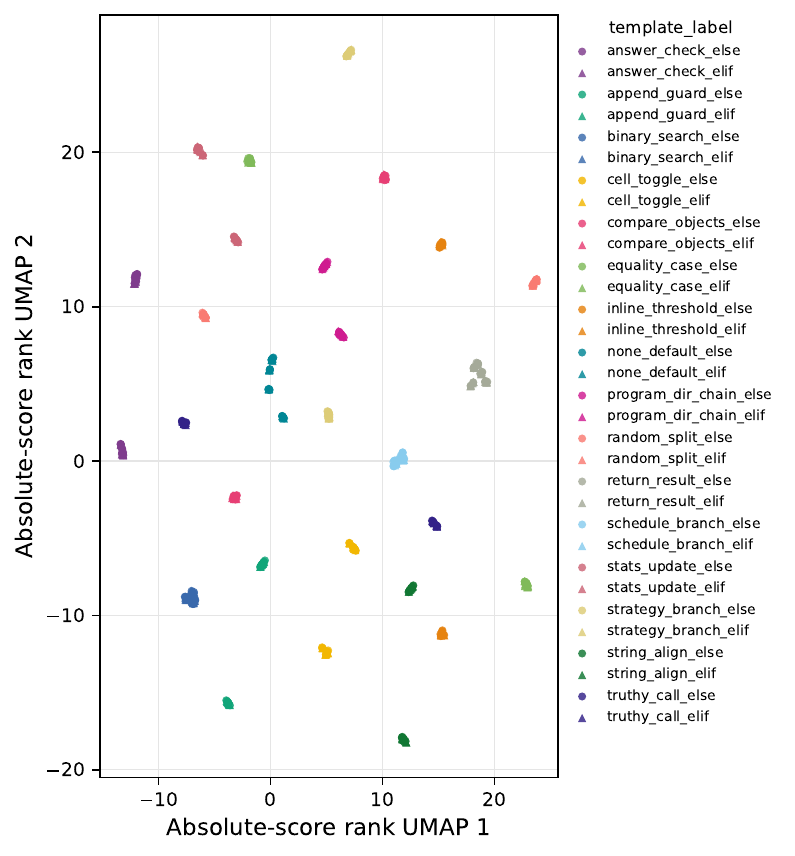}
        \caption{Absolute-score rank UMAP.}
    \end{subfigure}
    \hfill
    \begin{subfigure}[t]{0.52\textwidth}
        \centering
        \includegraphics[width=\linewidth,height=0.28\textheight,keepaspectratio]{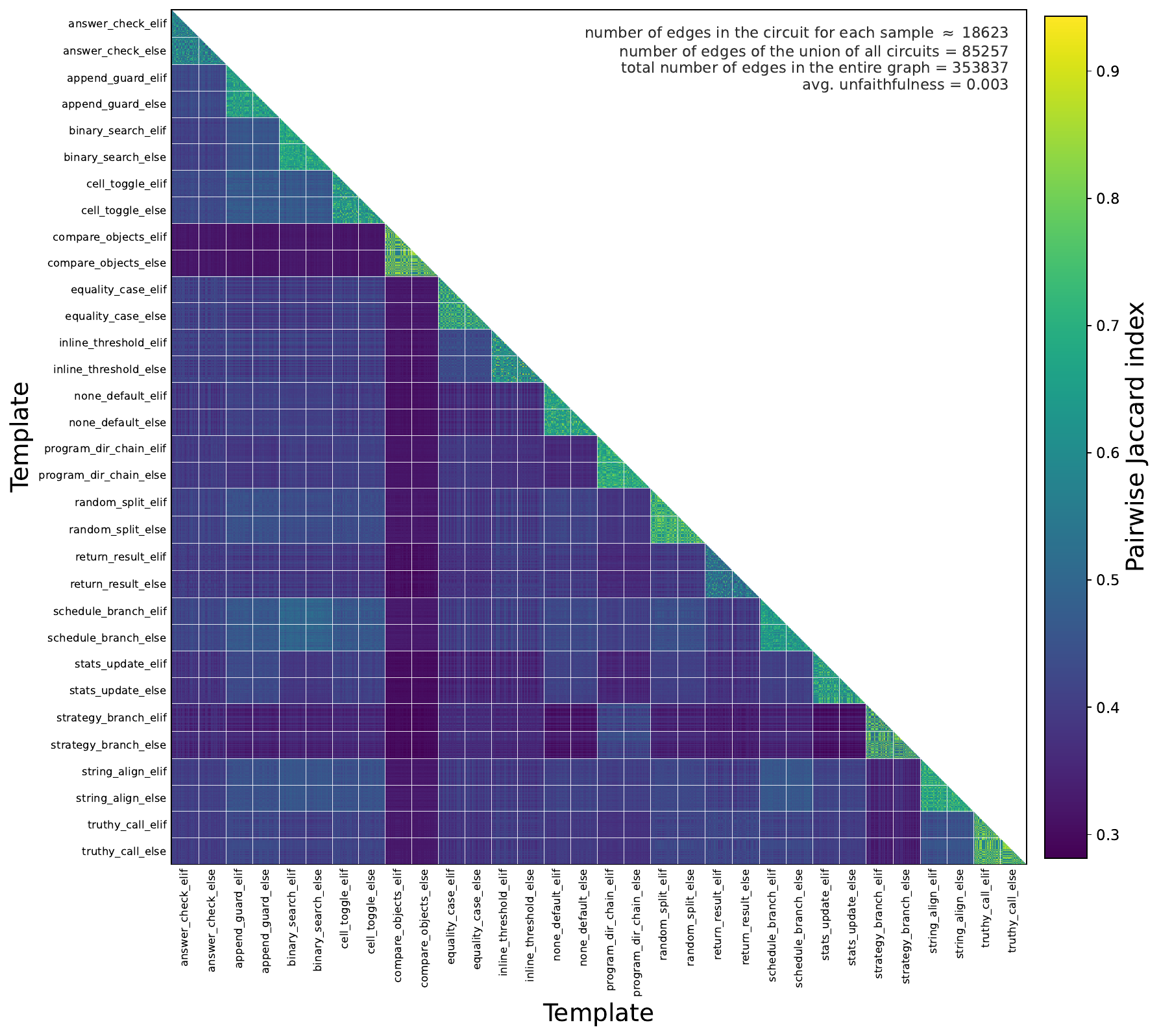}
        \caption{Pairwise Jaccard index.}
    \end{subfigure}
    \caption{Template-induced circuit difference for \texttt{csp\_sweep1\_4x\_3.7Mnonzero\_afrac1.000} on else-elif.}
\end{figure}

\subsubsection{\texttt{csp\_sweep1\_4x\_3.7Mnonzero\_afrac0.500}}
\begin{figure}[H]
    \centering
    \begin{subfigure}[t]{0.44\textwidth}
        \centering
        \includegraphics[width=\linewidth,height=0.28\textheight,keepaspectratio]{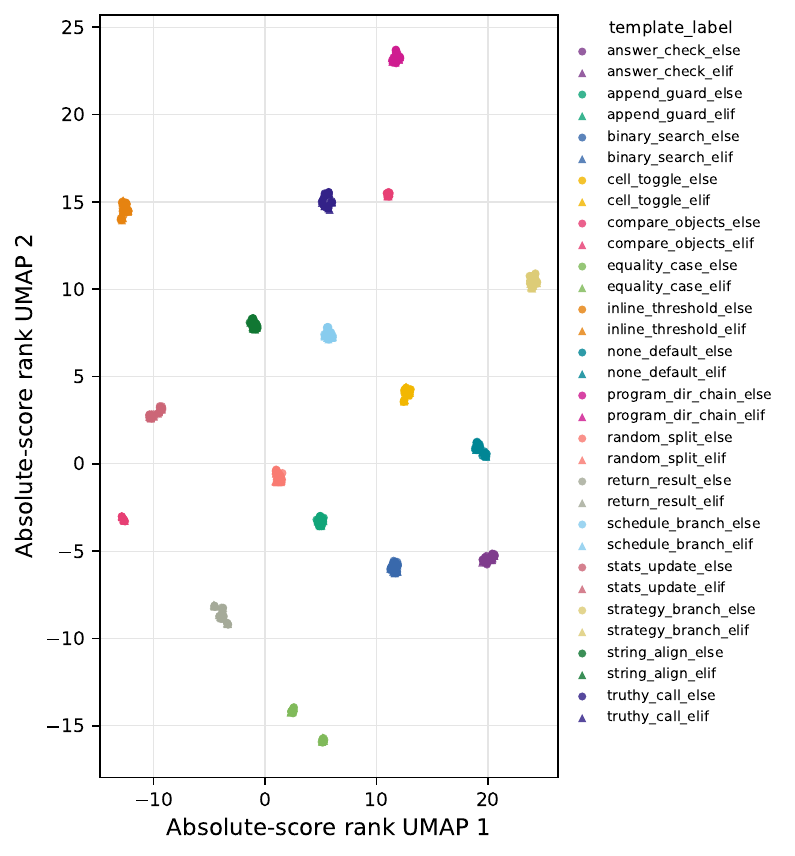}
        \caption{Absolute-score rank UMAP.}
    \end{subfigure}
    \hfill
    \begin{subfigure}[t]{0.52\textwidth}
        \centering
        \includegraphics[width=\linewidth,height=0.28\textheight,keepaspectratio]{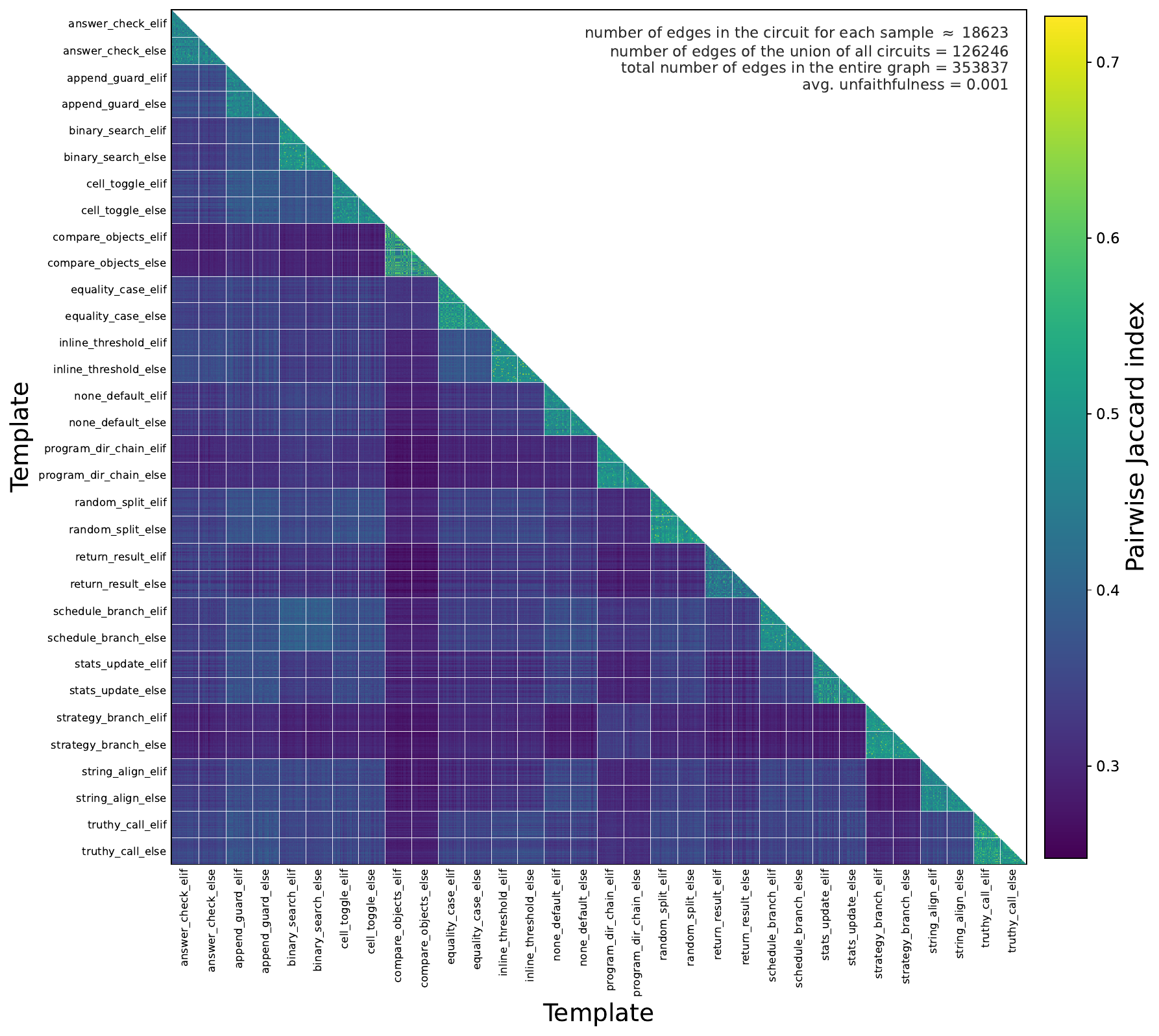}
        \caption{Pairwise Jaccard index.}
    \end{subfigure}
    \caption{Template-induced circuit difference for \texttt{csp\_sweep1\_4x\_3.7Mnonzero\_afrac0.500} on else-elif.}
\end{figure}

\section{Cross-Template Circuit Overlap Summary}
\label{app: template overlap summary tables}
Tables~\ref{tab:single-double-quote-template-offdiag-summary} and~\ref{tab:else-elif-template-offdiag-summary} summarize the cross-template sample-level circuit overlap for the dense and circuit-sparsity models.
For each model and task, we use circuits containing roughly $5.2\%$ of all graph edges.
The PJI average excludes not only the elementwise diagonal, but also pairs of samples from the same normalized template; for example, \texttt{A\_s} and \texttt{A\_d} are treated as the same single-double-quote template, and \texttt{A\_else} and \texttt{A\_elif} are treated as the same else-elif template.
Note that for models with \texttt{afrac} less than $1$, \citet{gao2025weight} use $\mathrm{topk(\cdot)}$ to enforce activation sparsity, which introduces discontinuities into the network.
In this setting, conductance can only capture the contribution of each edge along the continuous portions of the computation.
Any effect an edge has by causing the output to jump across a discontinuity is not reflected in the conductance score.
In principle, this could introduce inaccuracies.
However, as demonstrated in the tables below, all circuits achieve low unfaithfulness, providing empirical evidence that conductance remains effective.

\begin{table}[H]
    \centering
    \small
    \resizebox{\textwidth}{!}{%
    \begin{tabular}{llcc}
        \toprule
        Type & Model & Cross-template avg. PJI & Unfaithfulness \\
        \midrule
        \multirow{10}{*}{sparse} & \texttt{csp\_sweep1\_4x\_0.9Mnonzero\_afrac1.000} & 0.3424 & 0.0022 \\
         & \texttt{csp\_sweep1\_4x\_1.9Mnonzero\_afrac1.000} & 0.3195 & 0.0025 \\
         & \texttt{csp\_sweep1\_4x\_3.7Mnonzero\_afrac1.000} & 0.3116 & 0.0025 \\
         & \texttt{csp\_sweep1\_4x\_7.4Mnonzero\_afrac1.000} & 0.3074 & 0.0024 \\
         & \texttt{csp\_sweep1\_4x\_14.8Mnonzero\_afrac1.000} & 0.3059 & 0.0034 \\
         & \texttt{csp\_sweep1\_4x\_3.7Mnonzero\_afrac0.062} & 0.2129 & 0.0270 \\
         & \texttt{csp\_sweep1\_4x\_3.7Mnonzero\_afrac0.125} & 0.2436 & 0.0151 \\
         & \texttt{csp\_sweep1\_4x\_3.7Mnonzero\_afrac0.250} & 0.2547 & 0.0139 \\
         & \texttt{csp\_sweep1\_4x\_3.7Mnonzero\_afrac0.500} & 0.2565 & 0.0157 \\
         & \texttt{csp\_bridges2} & 0.2642 & 0.0184 \\
        \midrule
        \multirow{2}{*}{dense} & \texttt{dense1\_1x} & 0.3752 & 0.0090 \\
         & \texttt{dense1\_4x} & 0.3686 & 0.0026 \\
        \bottomrule
    \end{tabular}
    }
    \caption{Cross-template pairwise Jaccard index and unfaithfulness for single-double-quote. Values use circuits containing roughly $5.2\%$ of all graph edges; sample pairs from the same normalized template are excluded from the PJI average.}
    \label{table: single-double-quote overlap}
    \label{tab:single-double-quote-template-offdiag-summary}
\end{table}

\begin{table}[H]
    \centering
    \small
    \resizebox{\textwidth}{!}{%
    \begin{tabular}{llcc}
        \toprule
        Type & Model & Cross-template avg. PJI & Unfaithfulness \\
        \midrule
        \multirow{10}{*}{sparse} & \texttt{csp\_sweep1\_4x\_0.9Mnonzero\_afrac1.000} & 0.4496 & 0.0004 \\
         & \texttt{csp\_sweep1\_4x\_1.9Mnonzero\_afrac1.000} & 0.4106 & 0.0159 \\
         & \texttt{csp\_sweep1\_4x\_3.7Mnonzero\_afrac1.000} & 0.3915 & 0.0034 \\
         & \texttt{csp\_sweep1\_4x\_7.4Mnonzero\_afrac1.000} & 0.3818 & 0.0014 \\
         & \texttt{csp\_sweep1\_4x\_14.8Mnonzero\_afrac1.000} & 0.3634 & 0.0008 \\
         & \texttt{csp\_sweep1\_4x\_3.7Mnonzero\_afrac0.062} & 0.3496 & 0.0026 \\
         & \texttt{csp\_sweep1\_4x\_3.7Mnonzero\_afrac0.125} & 0.3206 & 0.0101 \\
         & \texttt{csp\_sweep1\_4x\_3.7Mnonzero\_afrac0.250} & 0.3330 & 0.0006 \\
         & \texttt{csp\_sweep1\_4x\_3.7Mnonzero\_afrac0.500} & 0.3292 & 0.0009 \\
         & \texttt{csp\_bridges2} & 0.3471 & 0.0106 \\
        \midrule
        \multirow{2}{*}{dense} & \texttt{dense1\_1x} & 0.4161 & 0.0320 \\
         & \texttt{dense1\_4x} & 0.3929 & 0.0013 \\
        \bottomrule
    \end{tabular}
    }
    \caption{Cross-template pairwise Jaccard index and unfaithfulness for else-elif. Values use circuits containing roughly $5.2\%$ of all graph edges; sample pairs from the same normalized template are excluded from the PJI average.}
    \label{table: else-elif overlap}
    \label{tab:else-elif-template-offdiag-summary}
\end{table}

\section{Population-level (Un)faithfulness}
\label{app: population-level unfaithfulness}
Let $N$ denote the set of samples and $(x_0^n,x_1^n)$ denote the corrupted-clean pair for sample $n\in N$.
Simplifying the notation from \cref{sec: background}, we write $M^n \coloneqq M(x;x_0^n,x_1^n)$,  $Q_\Gamma^n\coloneqq Q_\Gamma(x_0^n,x_1^n)$ for $\Gamma\in\{G,G'\}$ and write the corresponding sample-level faithfulness and unfaithfulness as $\Phi^n\coloneqq \Phi(G',G;x_0^n,x_1^n)=Q_{G'}^n/Q_G^n$ and $U^n\coloneqq U(G',G;x_0^n,x_1^n)=|1-\Phi^n|$.

Following the two notions of population faithfulness discussed by \citep{miller2024robust}, we note that there can be two types of population unfaithfulness:

\paragraph{Outer expectation:} 
$U_O^N\coloneqq \frac{1}{\vert N\vert}\sum_{n\in N} U^n.$

\paragraph{Inner expectation:}
The population faithfulness used by \citep{wang2023interpretability,faithfulness} is defined as $\Phi_I^N\coloneqq \frac{\sum_{n\in N}Q_{G'}^n}{\sum_{n\in N}Q_{G}^n}.$
Then the inner expectation form of unfaithfulness naturally follows
$$
U_I^N\coloneqq \vert 1- \Phi_I^N \vert = \left\vert  \frac{\sum_{n\in N}(Q_{G}^n-Q_{G'}^n)}{\sum_{n\in N}Q_{G}^n} \right\vert.
$$

At first glance, outer expectation appears to be the more reasonable option.
Its interpretation as an “average level of unfaithfulness” is straightforward.
Nevertheless, we advocate the inner expectation form as the more principled measure, for the following two reasons.

First, the current circuit discovery algorithms score the edges by averaging (or, equivalently, summing) the scores across all the samples \citep{faithfulness,franco2026findinghighlyinterpretablepromptspecific,meloux2025mechanistic}.
For gradient-based methods, by the linearity of the gradient, this approach is equivalent to substituting the metric $M$ in \cref{eq:conductance_formula} with $\sum_{n\in N}M^n$.
This implies that the network behavior accounted for is the summation of behavior over all the samples $\sum_{n\in N}Q^n$ when the model is confronted with all corrupted-clean pairs in $N$, analogous to the sample-wise scoring of \cref{eq:conductance_formula} accounting for the behavior of $Q_G$.
From this perspective, $\Phi_I^N$ and $U_I^N$ are more apt metrics to measure the quality of circuit discovery, since they directly account for the summation in their form.
Admittedly, one can instead substitute $M$ with the weighted sum $\sum_{n\in N}\frac{M^n}{Q^n_G}$, where $Q_G^n$ is treated as a constant, in \cref{eq:conductance_formula}, so that the behavior accounted for by the scoring is more aligned with the outer expectation form.
Nevertheless, we also have the following reason to further advocate the inner expectation form.

The outer expectation treats the unfaithfulness of all samples equally.
However, as we see in \cref{sec: selective contribution scaling}, the high unfaithfulness of samples with small $\vert Q_G^n\vert$ is an artifact of the metric itself and does not indicate major circuit defects.
Hence, we should be more forgiving of the unfaithfulness of samples with small $\vert Q_G^n\vert $, which the inner expectation form captures.
The inaccuracy of each $n$ caused by $G'$ is factored into this form through $(Q_G^n - Q_{G'}^n)$:
If two samples contribute the same $(Q_G^n - Q_{G'}^n)$ to the inner expectation, then the sample with smaller $\vert Q_G^n\vert$ would have much worse $U^n$, which reflects a higher level of tolerance for it.
In contrast, the outer expectation form might yield results that are drowned out by the extreme unfaithfulness values of samples with small $\vert Q_G^n\vert$.

\section{The Defect of Unfaithfulness for a Single Sample}
\label{app: single sample faithfulness}
Unfaithfulness is a metric intended to capture how differently the chosen circuit behaves from the whole model.
Nonetheless, the difference could lie in multiple aspects.
\citep{marrbook} proposed three levels of analysis: computational, algorithmic, and implementational.
The implementational level concerns the hardware-specific mechanisms to realize AI models and is largely irrelevant to interpretability.
The algorithmic level characterizes the program that the model uses to handle tasks \citep{wang2023interpretability,franco2026findinghighlyinterpretablepromptspecific}.
While it is most appealing to characterize unfaithfulness at this level, there is currently no unified framework for doing so.
The closest proxy for this is completeness \citep{wang2023interpretability}, but it is computationally intractable.
Thus, we are left with characterizing unfaithfulness at the computational level—that is, how much the input-output function implemented by the circuit over the relevant region of the input space (which defines a given task) deviates from the corresponding function implemented by the full model.
Achieving low computational unfaithfulness is a necessary condition for attaining low algorithmic unfaithfulness \citep{marrbook}.  
In this light, the form of unfaithfulness we investigate (computational unfaithfulness) should be viewed as only a crude surrogate for the genuine notion of unfaithfulness (algorithmic unfaithfulness).
This provides us with a rationale to refine our notion of unfaithfulness.

Within computational unfaithfulness, there is also the difference between single-sample computational unfaithfulness and functional computational unfaithfulness.
In terms of a single sample, the agreement between the full model’s output for it and that of the selected circuit provides very limited information regarding how similar the underlying algorithms implemented by the full model and the circuit actually are.
The circuit could instantiate a vastly different function from the full model, only coinciding in value at the point of the sample itself (see \cref{fig:function-fitting-diagram}).
As we include more edges into the circuit, the circuit function evaluated at the sample could deviate from the full model, but conceptually, the underlying algorithm is more developed with more edges and thus the function as a whole should be closer to the full model function.
Hence, it is reasonable to adopt the pessimistic unfaithfulness as a more canonical notion of single sample unfaithfulness to enforce monotonicity, which reflects the underlying improvement of algorithmic faithfulness as more edges are added.

On the other hand, population unfaithfulness involves many points in the input space, so it is more informative as a measure of unfaithfulness at the function level and already better satisfies the desideratum of monotonicity, which can be observed in \cref{app: complete ceap vs eap-ig jaccard and unfaithfulness}.
This explains the efficacy of (un)faithfulness in numerous previous works \citep{faithfulness,finetunecircuit,meloux2025mechanistic}.

\begin{figure}[H]
    \centering
    \includegraphics[width=0.78\textwidth]{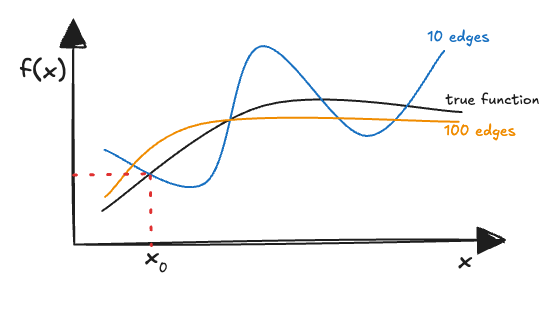}
    \caption{A schematic of why single sample faithfulness is highly non-monotonic. 
    Suppose we apply circuit discovery to recover the true function (which is the prerequisite for achieving algorithmic faithfulness), while the sample of interest is at $x_0$. 
    When we include $10$ edges, the circuit function might be quite far away from the true function, but they coincide at $x_0$. 
    With more edges included, the circuit function is closer to the true function as a whole, but it may deviate further from the true function at $x_0$.}
    \label{fig:function-fitting-diagram}
\end{figure}

\section{Extended Correlation Diagnostics Across Templates, Models, and Circuit Sizes}
\label{app:extended-correlation-diagnostics}

We extend the six rank-correlation diagnostics shown in \cref{fig: unfaithfulness exaggeration,fig: long-tail mechanism plot} from one GPT-2 small SVA template and one circuit size to all templates of the SVA/IOI/greater-than tasks, GPT-2 small and Pythia-160M, and the full greedy edge-count sweep. Solid curve segments indicate statistically significant Spearman rank correlations at $p<0.05$; dotted segments indicate non-significant correlations. We present the diagnostics in the same order as the corresponding panels in the main text.

When the curves are mostly above zero, it indicates a positive correlation; otherwise, a negative correlation.
There are many curve segments that are not statistically significant.
This is because certain templates reach low unfaithfulness easily.
Once all the samples in one template are highly faithful, the trends that involve unfaithfulness will disappear.
For example, if we include all edges in the graph, then all unfaithfulness will be zero, and the heavy-tailedness of $s^n$ will have no effect on it.
However, this does not nullify our explanation, which aims mainly at explaining cases where extremely poor unfaithfulness scores exist.

Note that $U'/U$ graphs are, in general, noisier and less statistically significant than $\bar U'/\bar U$. 
However, the general trend does not change.

\subsection{Individual Scoring}

\subsubsection{\texorpdfstring{Figure 7(b): $\bar{U}/U$ vs.\ $|Q_G|$}{Figure 7(b): Ubar/U vs. |QG|}}

\begin{figure}[H]
    \centering
    \includegraphics[width=\textwidth]{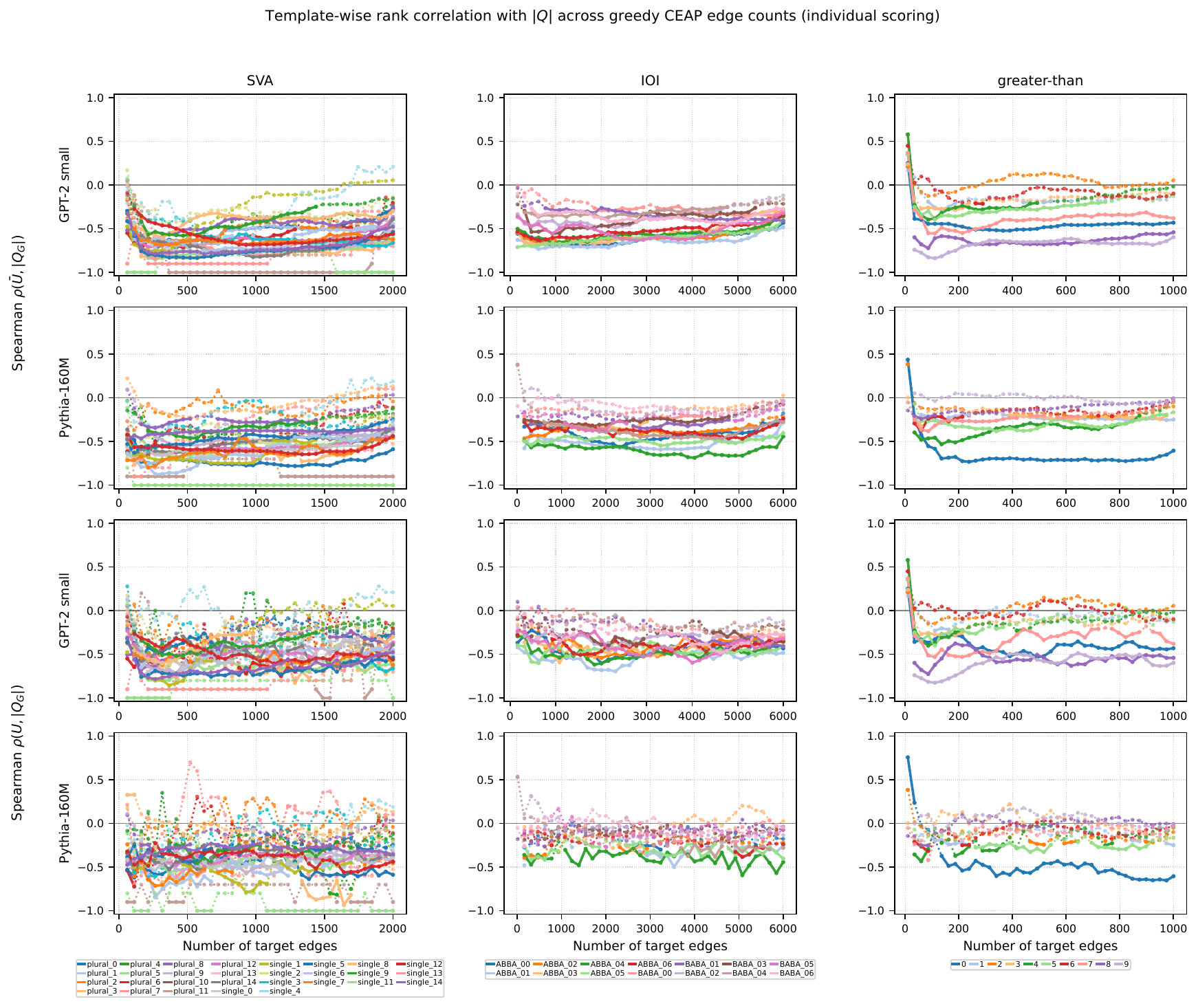}
    \caption{Extension of the $\bar{U}/U$ vs.\ $|Q_G|$ diagnostic. All the curves that are statistically significant are below zero, except in a few cases where the circuit is too small for the unfaithfulness to be informative yet.
    Moreover, if we constrain the circuit to be of a small size, the circuit discovery algorithm needs to accommodate more for the connectivity of the subgraph, rather than picking the edges with higher score absolute values, which deviates from 
    the mental image of \cref{fig:long-tail-mental-image}.}
    \label{fig:extended-u-vs-q}
\end{figure}

\subsubsection{\texorpdfstring{Figure 7(c): $\bar{U}'/U'$ vs.\ $|Q_G|$}{Figure 7(c): Ubar-prime/U-prime vs. |QG|}}

\begin{figure}[H]
    \centering
    \includegraphics[width=\textwidth]{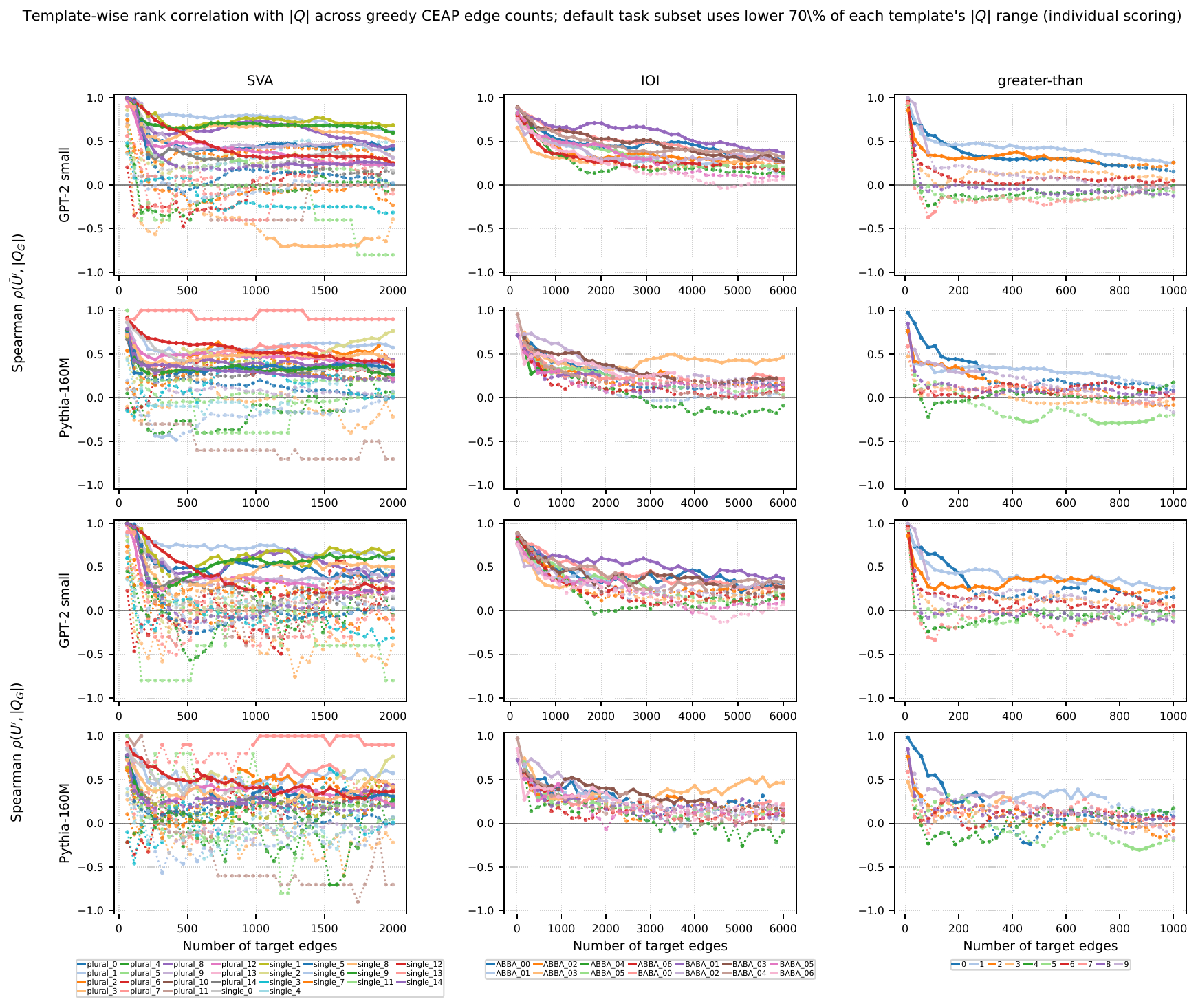}
    \caption{Extension of the unnormalized $\bar{U}'/U'$ vs.\ $|Q_G|$ diagnostic.
    It generally shows positive correlation, with a few exceptions, such as \texttt{plural\_3}. There, we have only $11$ samples, and the correlation readings may be unreliable. 
    Greater-than may show negative correlations here, because we see that $\bar U'/U'$ normally first goes up with $\vert Q_G\vert$, then goes down. To account for this, for each template we remove all points beyond the first $70\%$ of that template's $\vert Q_G\vert$ range. Note that this does not contradict our main point that points with small $\vert Q_G\vert$ in general do not have large $\bar U'/U'$. }
    \label{fig:extended-uprime-vs-q}
\end{figure}

\subsubsection{\texorpdfstring{Figure 9(a): $\mu$ vs.\ $|Q_G|$}{Figure 9(a): mu vs. |QG|}}

\begin{figure}[H]
    \centering
    \includegraphics[width=\textwidth]{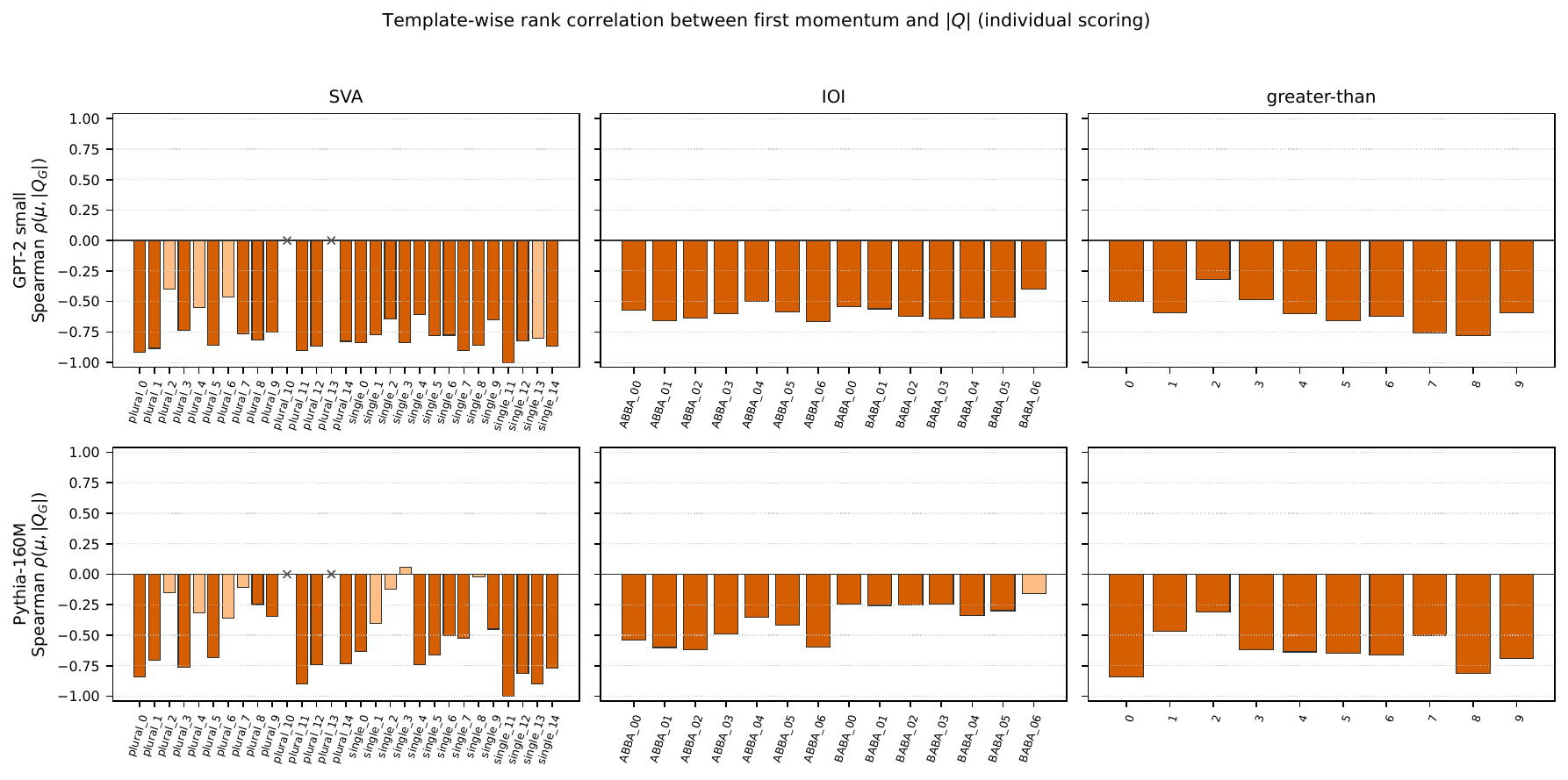}
    \caption{Extension of the $\mu$ vs.\ $|Q_G|$ diagnostic. There is a clearly negative correlation in general. A darker color suggests the Spearman $\rho$ is statistically significant, and a lighter color suggests otherwise. The x's here indicate that there are too few samples for those templates in the original dataset and we did not get any data for them when sampling the dataset for our probing here.
    }
    \label{fig:extended-tail-vs-q}
\end{figure}

\subsubsection{\texorpdfstring{Figure 9(b): $R$ vs.\ $\mu$}{Figure 9(b): R vs. mu}}

\begin{figure}[H]
    \centering
    \includegraphics[width=\textwidth]{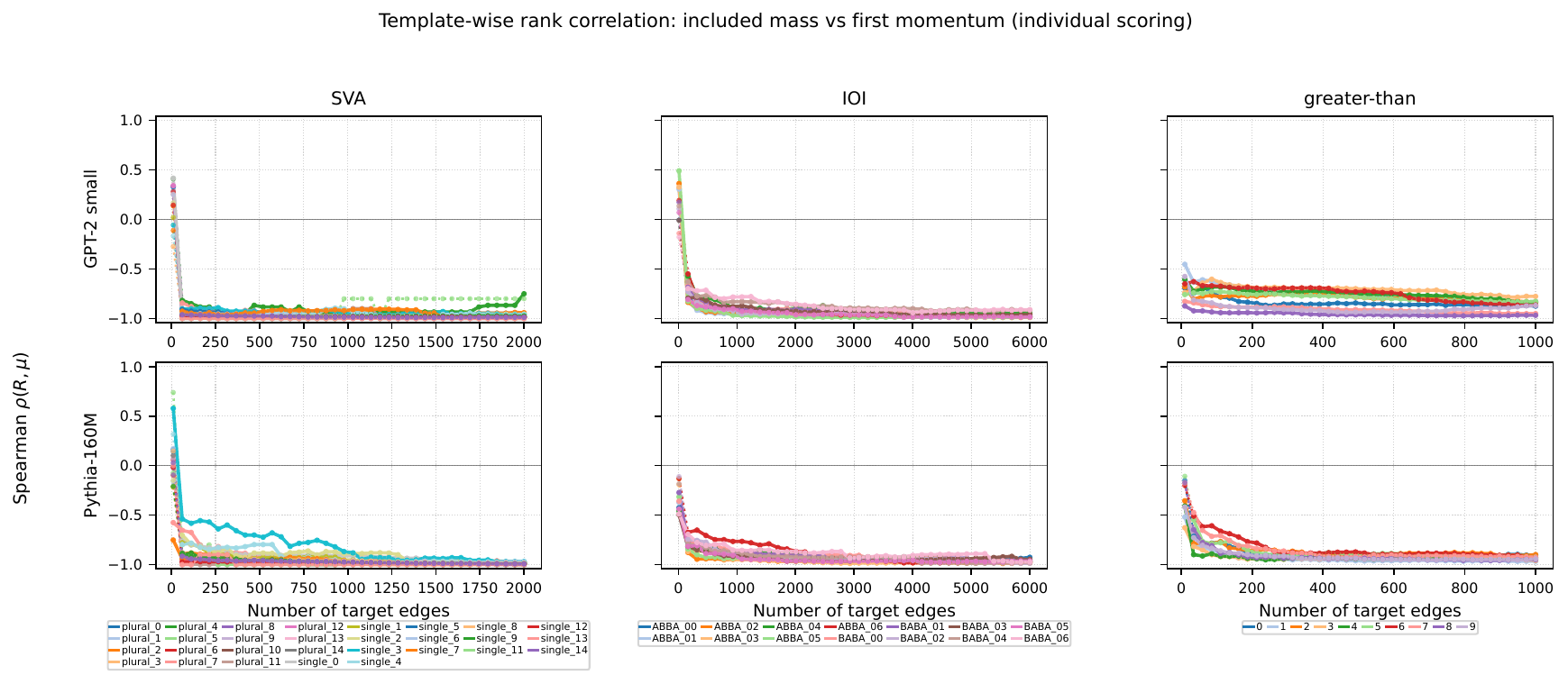}
    \caption{Extension of the $R$ vs.\ $\mu$ diagnostic. There is a clearly negative correlation in general.}
    \label{fig:extended-mass-vs-tail}
\end{figure}

\subsubsection{\texorpdfstring{Figure 9(c): $\bar{U}/U$ vs.\ $R$}{Figure 9(c): Ubar/U vs. R}}

\begin{figure}[H]
    \centering
    \includegraphics[width=\textwidth]{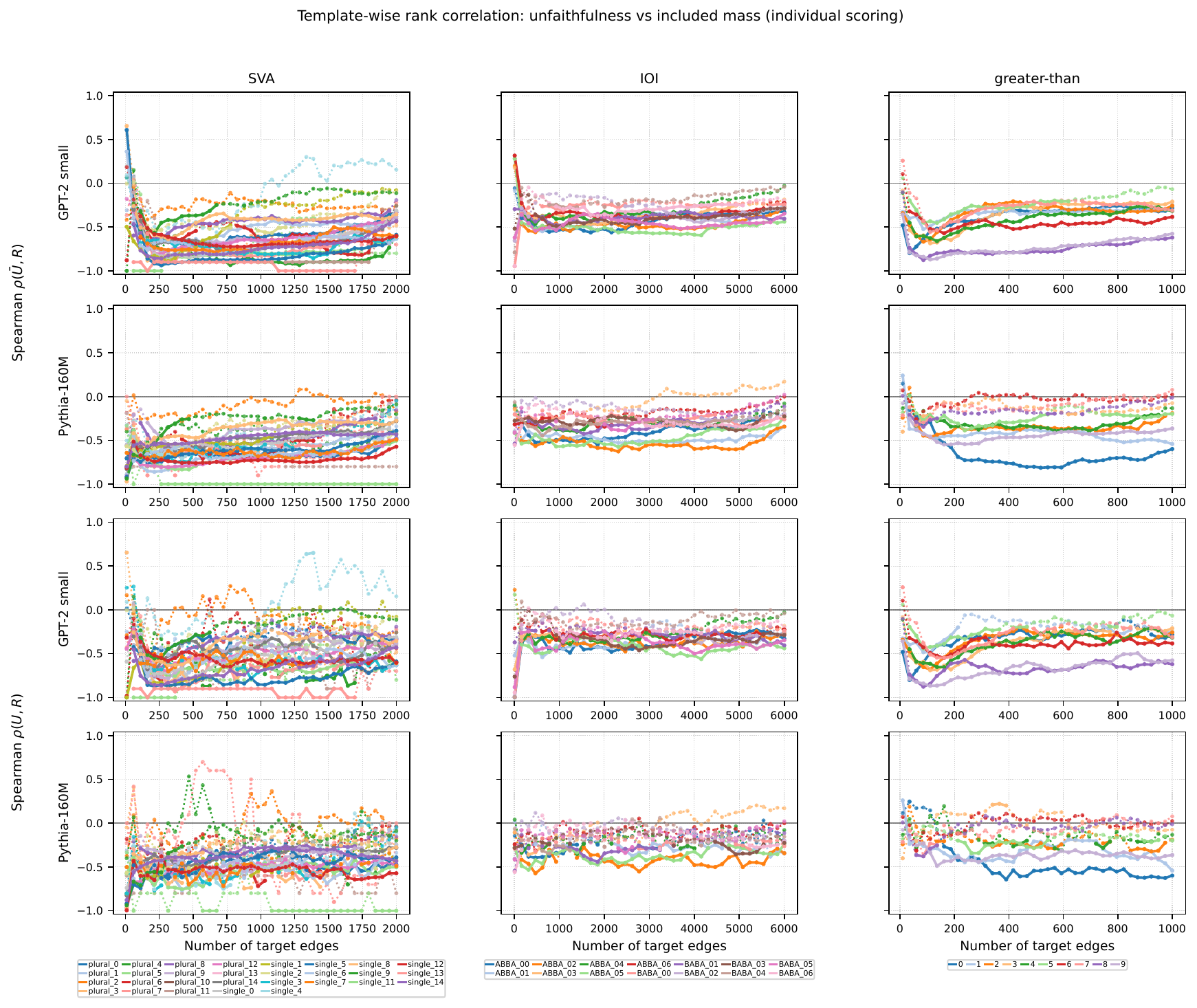}
    \caption{Extension of the $\bar{U}/U$ vs.\ $R$ diagnostic. There is a clearly negative correlation in general.}
    \label{fig:extended-u-vs-mass}
\end{figure}

\subsubsection{\texorpdfstring{Figure 9(d): $\bar{U}/U$ vs.\ $\mu$}{Figure 9(d): Ubar/U vs. mu}}

\begin{figure}[H]
    \centering
    \includegraphics[width=\textwidth]{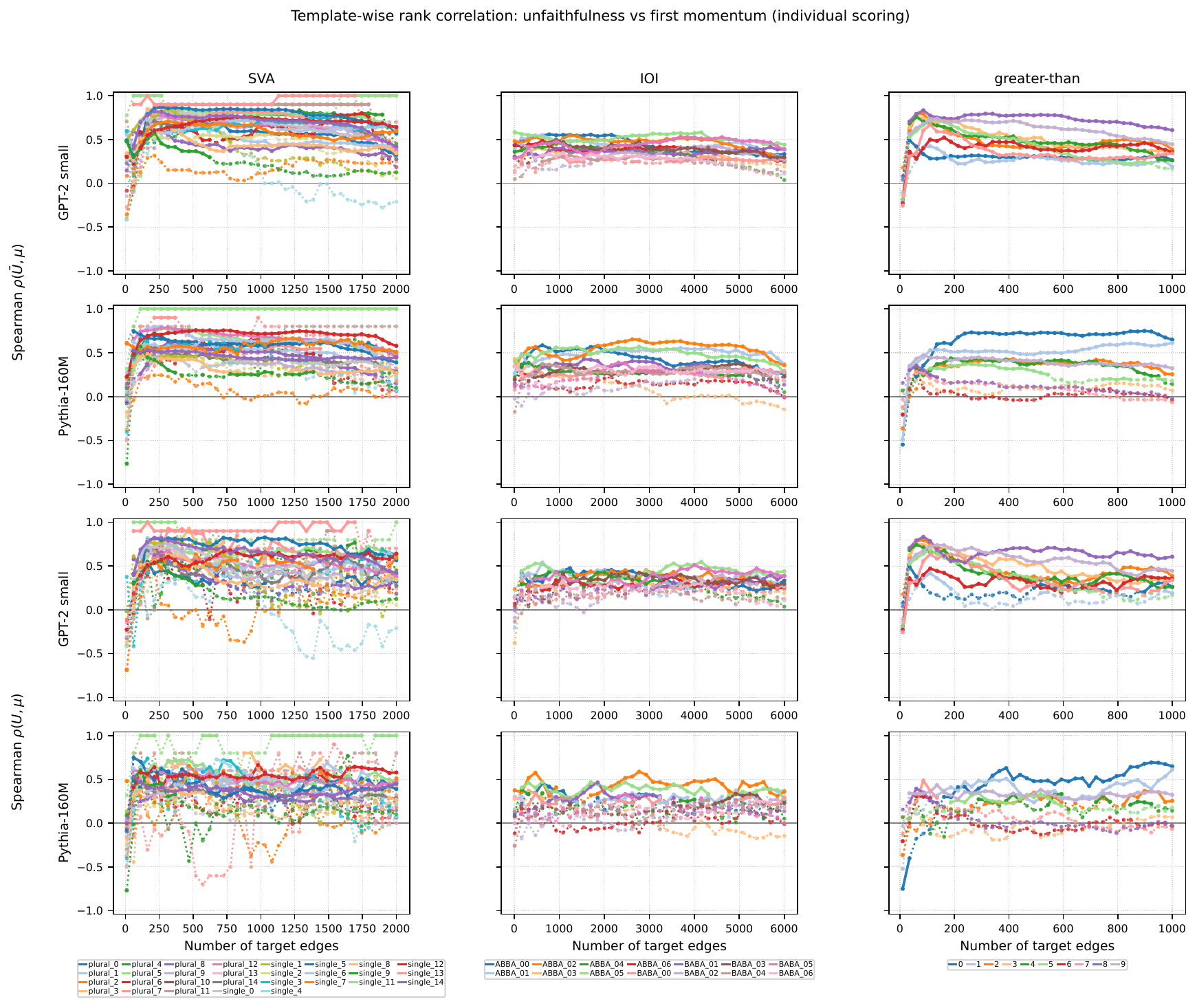}
    \caption{Extension of the $\bar{U}/U$ vs.\ $\mu$ diagnostic. There is a clearly positive correlation in general.}
    \label{fig:extended-u-vs-tail}
\end{figure}

\subsection{Template-averaged Scoring}
All figures follow the trend in the last section, with a few exceptions for greater-than.
We believe this is due to the diffuse nature of the greater-than circuits. As shown in \cref{app: umap and pji matrix}, the circuits underlying templates 1 through 9 form a smooth continuum rather than cleanly separated clusters, meaning that circuits within these templates are not particularly similar. Consequently, the results from the previous section provide only limited predictive power for this setup, because we require the greedy circuit discovery to approximately select the edges with the top absolute value scores for the mental image of \cref{fig:long-tail-mental-image} to work.
If the circuit is chosen via averaging across circuits that are not similar, then this assumption does not hold.

\subsubsection{\texorpdfstring{Figure 5(a): $U$ distribution}{Figure 5(a): U distribution}}

\begin{figure}[H]
    \centering
    \includegraphics[width=0.32\textwidth]{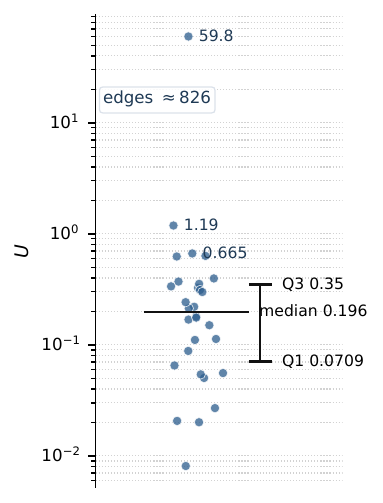}
    \caption{Template-averaged scoring version of the $U$ distribution diagnostic for GPT-2 small on SVA template \texttt{plural\_1}.}
    \label{fig:u-distribution-all-averaging}
\end{figure}

\subsubsection{\texorpdfstring{Figure 7(b): $\bar{U}/U$ vs.\ $|Q_G|$}{Figure 7(b): Ubar/U vs. |QG|}}

\begin{figure}[H]
    \centering
    \includegraphics[width=\textwidth]{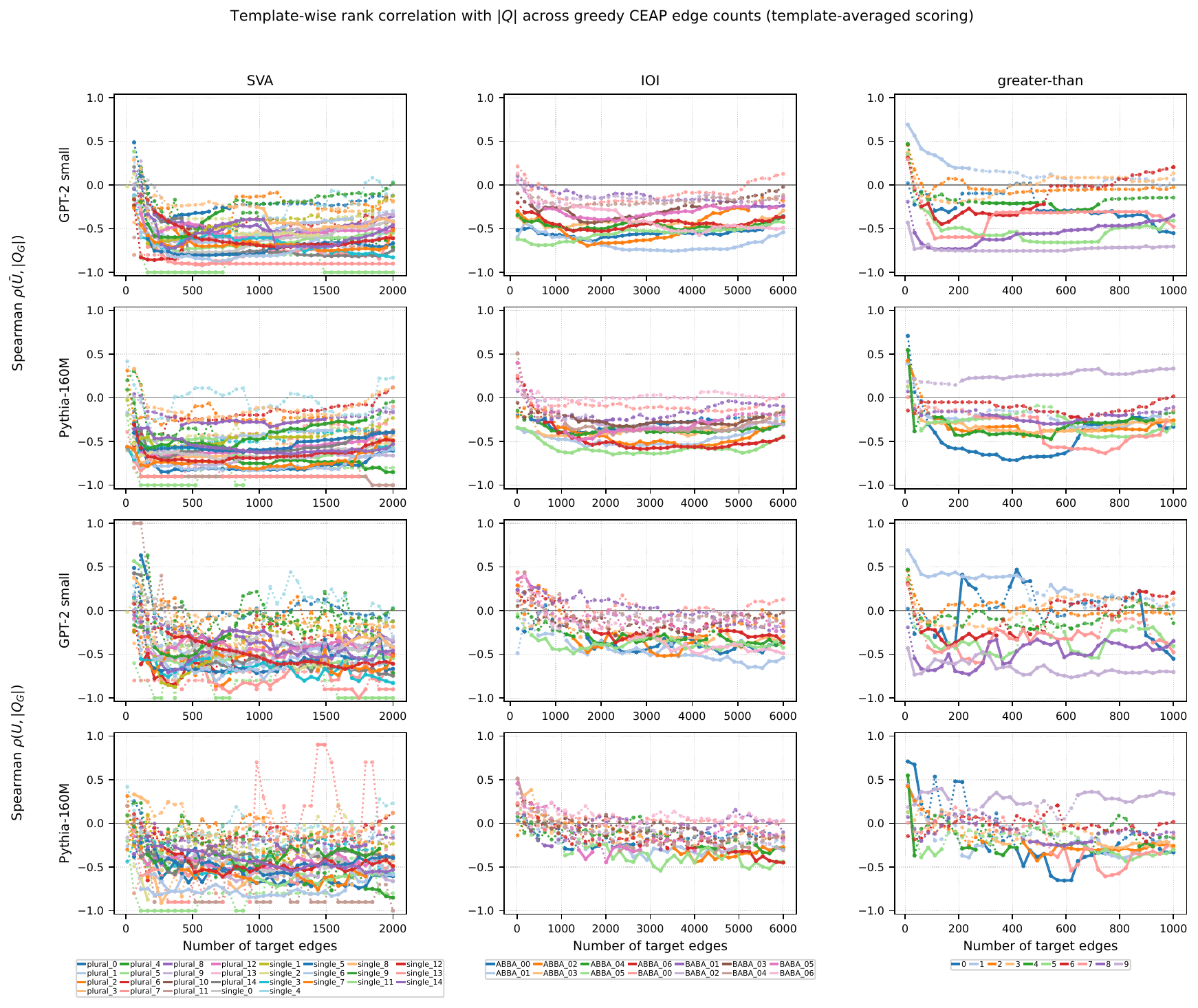}
    \caption{Template-averaged scoring version of the $\bar{U}/U$ vs.\ $|Q_G|$ diagnostic.}
    \label{fig:extended-u-vs-q-all-averaging}
\end{figure}

\subsubsection{\texorpdfstring{Figure 7(c): $\bar{U}'/U'$ vs.\ $|Q_G|$}{Figure 7(c): Ubar-prime/U-prime vs. |QG|}}

\begin{figure}[H]
    \centering
    \includegraphics[width=\textwidth]{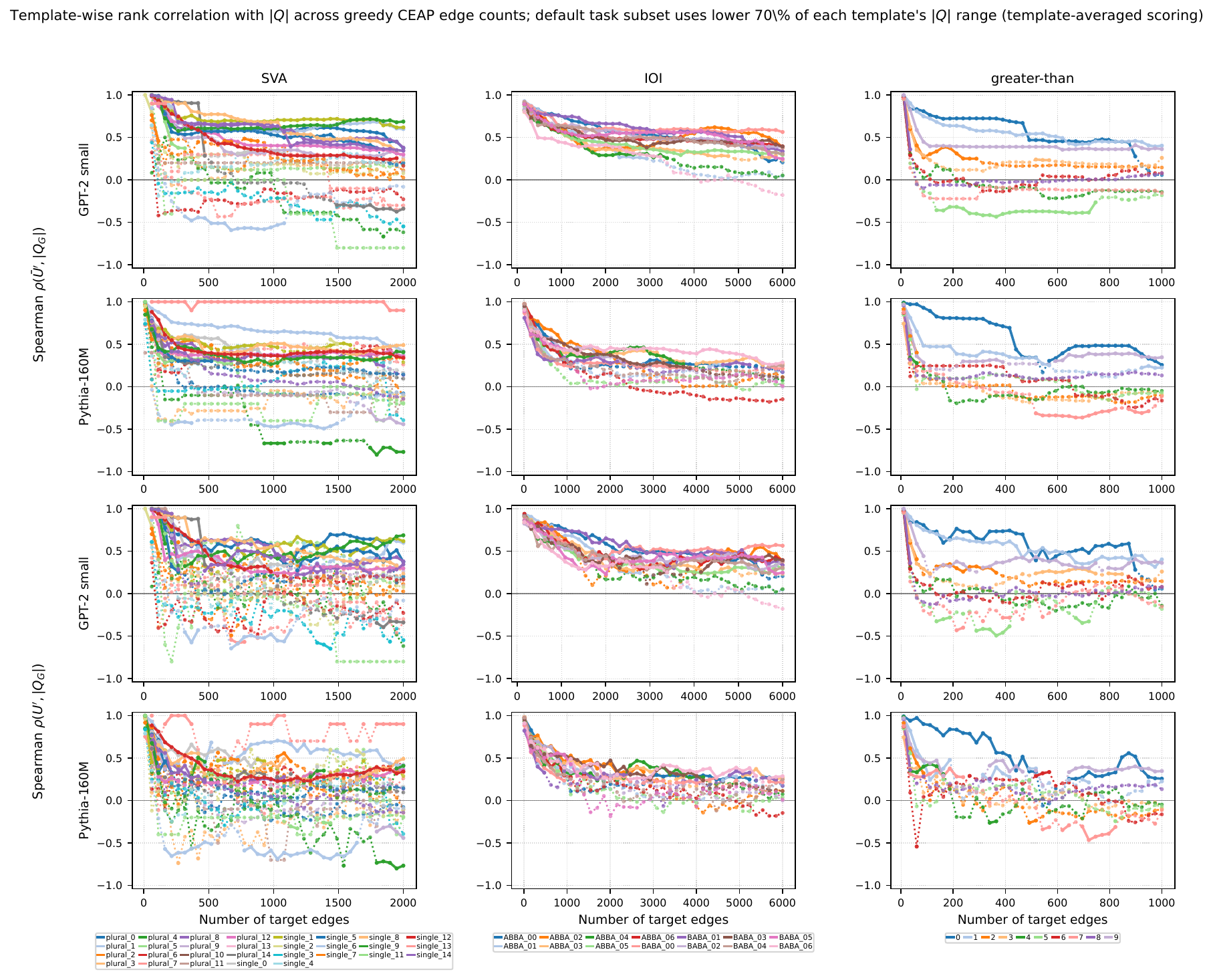}
    \caption{Template-averaged scoring version of the unnormalized $\bar{U}'/U'$ vs.\ $|Q_G|$ diagnostic.}
    \label{fig:extended-uprime-vs-q-all-averaging}
\end{figure}

\subsubsection{\texorpdfstring{Figure 9(a): $\mu$ vs.\ $|Q_G|$}{Figure 9(a): mu vs. |QG|}}

This plot is exactly the same as \cref{fig:extended-tail-vs-q}, as it reflects a property of each sample and does not depend on whether we average over all samples.

\subsubsection{\texorpdfstring{Figure 9(b): $R$ vs.\ $\mu$}{Figure 9(b): R vs. mu}}

\begin{figure}[H]
    \centering
    \includegraphics[width=\textwidth]{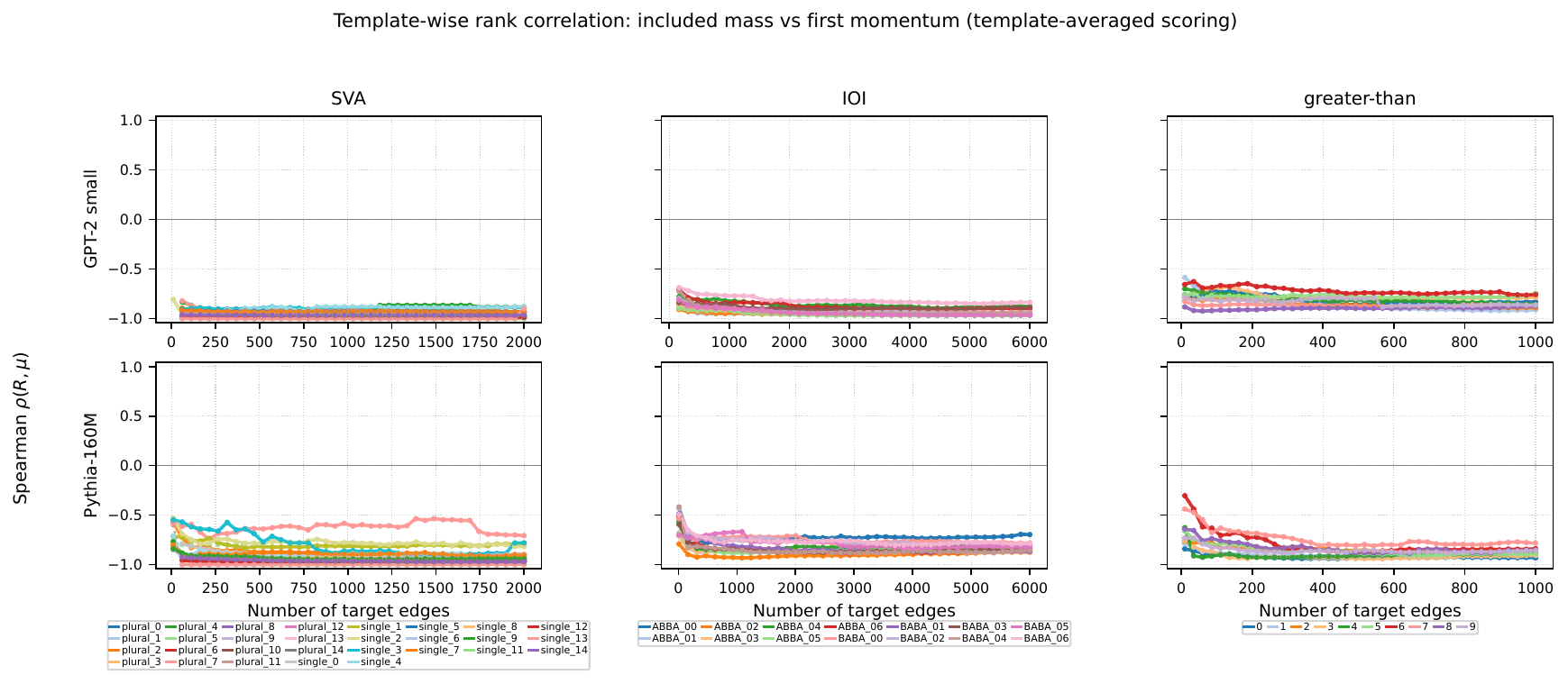}
    \caption{Template-averaged scoring version of the $R$ vs.\ $\mu$ diagnostic.}
    \label{fig:extended-mass-vs-tail-all-averaging}
\end{figure}

\subsubsection{\texorpdfstring{Figure 9(c): $\bar{U}/U$ vs.\ $R$}{Figure 9(c): Ubar/U vs. R}}

\begin{figure}[H]
    \centering
    \includegraphics[width=\textwidth]{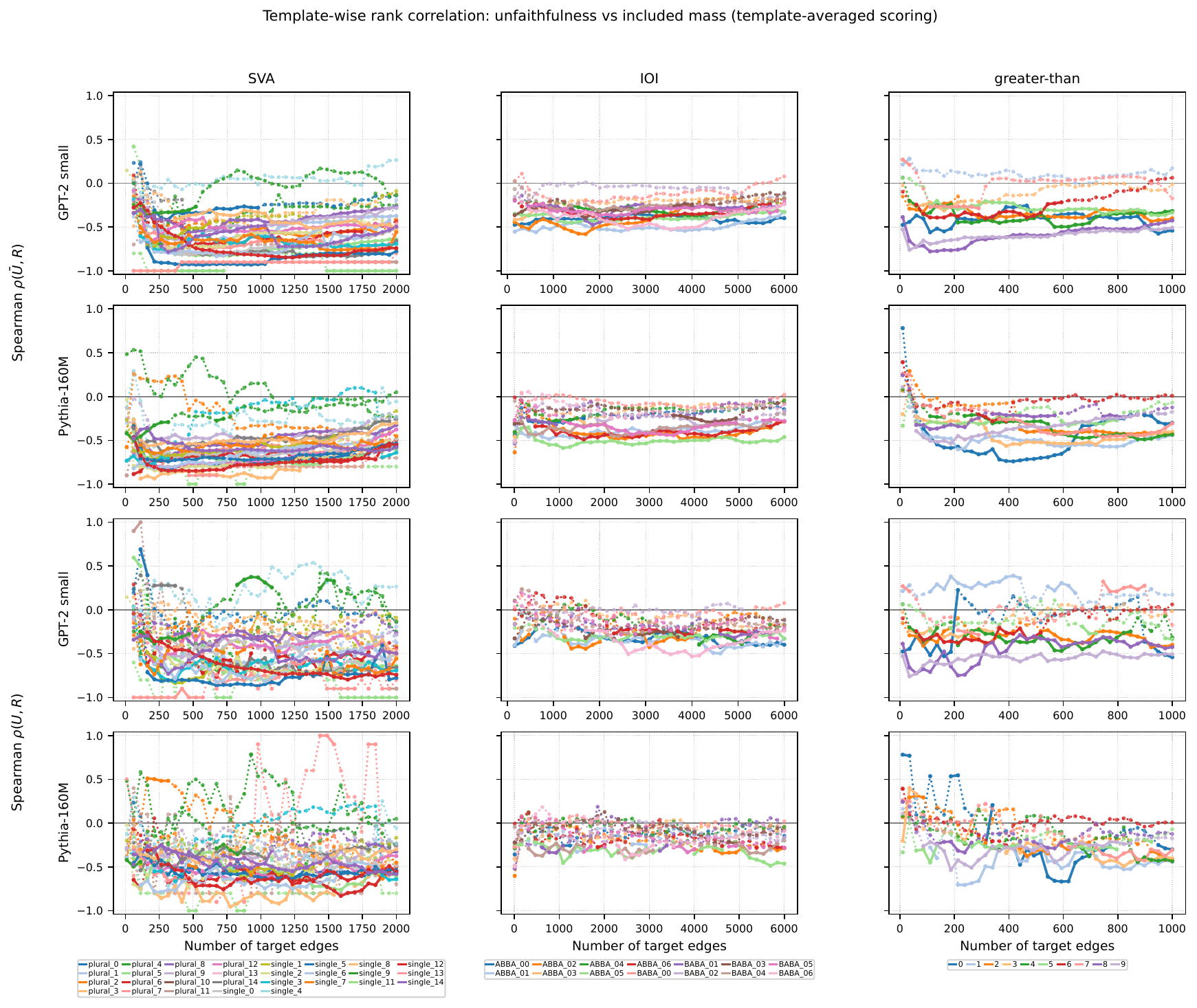}
    \caption{Template-averaged scoring version of the $\bar{U}/U$ vs.\ $R$ diagnostic.}
    \label{fig:extended-u-vs-mass-all-averaging}
\end{figure}

\subsubsection{\texorpdfstring{Figure 9(d): $\bar{U}/U$ vs.\ $\mu$}{Figure 9(d): Ubar/U vs. mu}}

\begin{figure}[H]
    \centering
    \includegraphics[width=\textwidth]{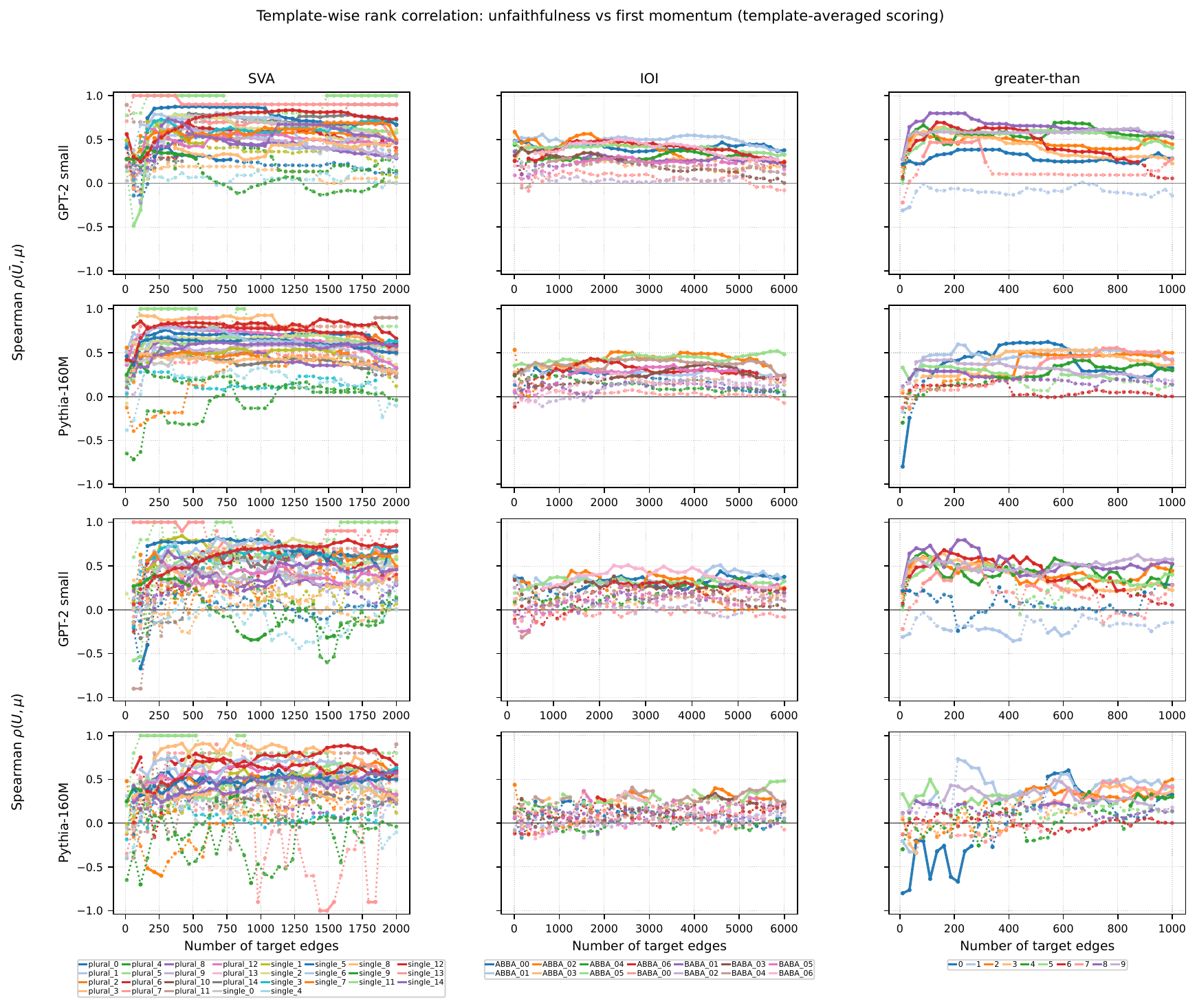}
    \caption{Template-averaged scoring version of the $\bar{U}/U$ vs.\ $\mu$ diagnostic.}
    \label{fig:extended-u-vs-tail-all-averaging}
\end{figure}

\section[Substantiating the Mental Image with Two Real Samples]{Substantiate the Mental Image of \cref{fig:long-tail-mental-image} with Two Real Samples}
\label{app: mental image corroboration}
\begin{figure}
    \centering
    \includegraphics[width=\textwidth]{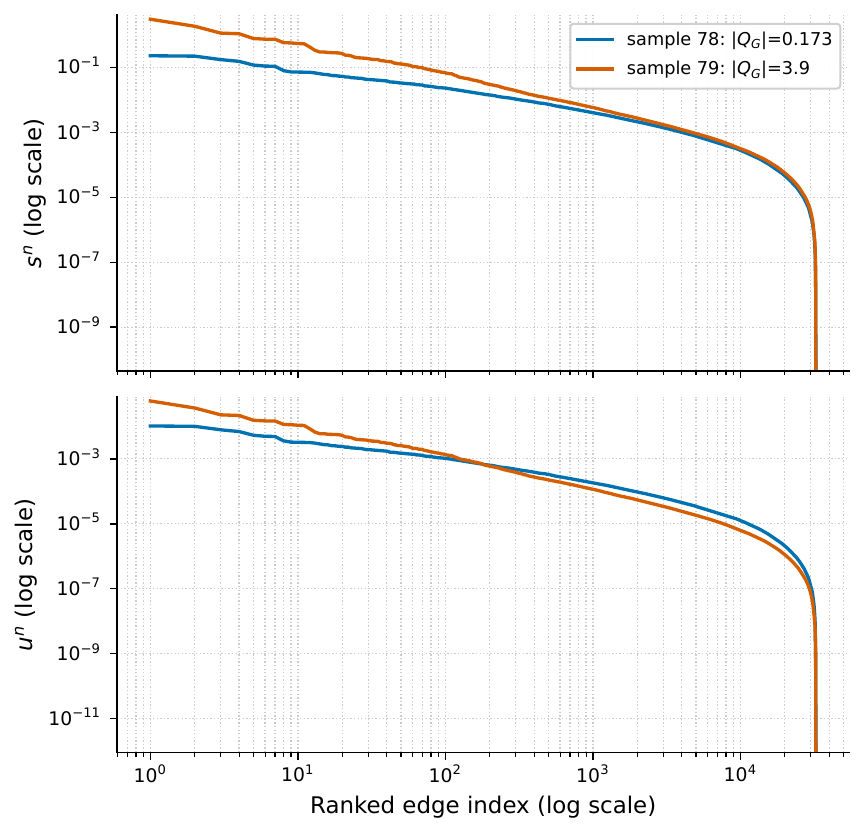}
    \caption{Redraw \cref{fig:long-tail-mental-image} with two real samples of the same template. Sample 78 is the one with the poorest $U$ reading in \cref{fig: u distribution}, and sample 79 is a sample of the same template but much higher $\vert Q_G\vert$. The curves are shown on logarithmic axes to make both the high-score region and long tail visible.}
    \label{fig:sva-score-tail-comparison}
\end{figure}

\newpage
\end{document}